\documentclass[10pt]{article}
\usepackage{lipsum}
\usepackage{calc}
\usepackage[
  letterpaper,%
  textheight={47\baselineskip+\topskip},%
  textwidth={\paperwidth-126pt},%
  footskip=75pt,%
  marginparwidth=0pt,%
  top={\topskip+0.75in}]{geometry}
\usepackage{fancyhdr}

\usepackage{times}  
\usepackage{helvet}  
\usepackage{courier}  
\usepackage[hyphens]{url}  
\usepackage{graphicx} 
\usepackage[numbers]{natbib}  
\usepackage{caption} 
\usepackage{algorithm}
\usepackage{algorithmic}

\usepackage{newfloat}
\usepackage{listings}
\DeclareCaptionStyle{ruled}{labelfont=normalfont,labelsep=colon,strut=off} 
\floatstyle{ruled}
\newfloat{listing}{tb}{lst}{}
\floatname{listing}{Listing}
\usepackage{appendix}
\usepackage{csquotes}
\usepackage{comment}
\usepackage{xcolor}
\usepackage{amsthm}
\usepackage{amsmath}
\usepackage{amsfonts}
\usepackage{amssymb}
\usepackage{multirow}
\usepackage{booktabs}
\usepackage{soul}
\usepackage{subcaption}

\usepackage{mathtools}

\newtheorem{theorem}{Theorem}

\newtheorem{definition}[theorem]{Definition}

\newcommand{\Prob}{\boldsymbol{P}}
\DeclareMathOperator{\E}{\mathbb{E}}

\newcommand*\rot{\rotatebox{90}}

\title{Multi-Source Survival Domain Adaptation}
\author {Ammar Shaker\thanks{The corresponding author.} , Carolin Lawrence\\[3mm]
    NEC Laboratories Europe GmbH, Heidelberg, Germany\\
    \{Ammar.Shaker,Carolin.Lawrence\}@neclab.eu
}

\begin{document}

\maketitle
\begin{abstract}
 
Survival analysis is the branch of statistics that studies the relation between the characteristics of living entities and their respective survival times, taking into account the partial information held by censored cases. A good analysis can, for example, determine whether one medical treatment for a group of patients is better than another. With the rise of machine learning, survival analysis can be modeled as learning a function that maps studied patients to their survival times. To succeed with that, there are three crucial issues to be tackled. 
 First, some patient data is censored: we do not know the true survival times for all patients. Second, data is scarce, which led past research to treat different illness types as domains in a multi-task setup. Third, there is the need for adaptation to new or extremely rare illness types, where little or no labels are available. In contrast to previous multi-task setups, we want to investigate how to efficiently adapt to a new survival target domain from multiple survival source domains. 
 For this, we introduce a new survival metric and the corresponding discrepancy measure between survival distributions.
 These allow us to define domain adaptation for survival analysis while incorporating censored data, which would otherwise have to be dropped. Our experiments on two cancer data sets reveal a superb performance on target domains, a better treatment recommendation, and a weight matrix with a plausible explanation.

\end{abstract}

\section{Introduction}


The abundance of health records has massively increased in the last few decades, mainly due to the advancement of data collection methods and the increasing financial support for medical trials and research. To determine the effects of a specific environment or the success of a treatment, survival analysis can be used to study the relation between the characteristics of living entities and their respective survival times. This induced relation is often described by the survival function or the hazard function, which models the conditional propensity for the event of death to happen.

A crucial challenging characteristic of learning with health records is censoring, which is the case when only partial information about the patient's survival is known. This could happen either due to losing track of the patient or the termination of the study before observing the intended event on all patients. 
Simply discarding this data would lead to losing all the partial information carried by the censored cases, which would be particularly harmful when censoring is prevalent. This, for example, occurs in the messenger RNA data for breast adenocarcinoma, where censoring exceeds 87\%\footnote{\url{https://www.cancer.gov/about-nci/organization/ccg/research/structural-genomics/tcga}}. 

While censoring makes the direct application of machine learning methods unfeasible, active research tries to tackle this challenge. One essential line of work to tackle this challenge adopts the proportional hazards assumption (PH)~\cite{Cox:1972:RMLT}, instead of attempting to fully model the survival function. More recently, 
traditional survival analysis methods~\cite{Cox:1972:RMLT} have been complemented and then superseded by machine learning approaches; for a survey, see~\cite{wang2019machine}. 
For example, with the increasing success of deep learning methods, DeepSurv~\cite{katzman2018deepsurv} has reported a significant increase in performance by employing a neural network with a loss function adapted to hold the assumed proportionality of hazards.


An additional challenge arises when there is insufficient data for a particular problem of interest. This scenario is quite relevant in the medical field, where some diseases are more common than others, such as the varying incidence rates of cancer types confirmed by The Cancer Genome Atlas (TCGA) data. The issue is also present 
for new illnesses that arise, such as a significantly changed variant of a previous disease. In such a setup, a fitting machine learning technique would be multi-source domain adaptation~\cite{mansour2008domain}, which tries to exploit the knowledge-transfer from multiple source domains into the target domain. To the best of our knowledge, there has not been yet any work that tackles domain or multi-source domain adaptation for survival domains.


In this work, we introduce a first attempt to transfer knowledge from multiple source survival domains to a target survival domain. Our main contributions are summarized as follows:
\begin{itemize}
  \item{We construct the symmetric discordance index ($SDI$) to measure the distance between risk functions. We show the utility of $SDI$ in the survival domain adaptation in which multiple survival source tasks are observed (Section \ref{subsec:loss}). Second, we introduce the survival domain discrepancy distance $D_{SDI-disc}(P_s,P_t)$ to measure the proximity between distributions ($P_s$, $P_t$) with respect to hypothesis space $\mathcal{H}$ (Section \ref{subsec:bound}).}
\item{We derive an error generalization bound for survival target domains (Section \ref{subsec:bound}) and we employ this bound in an adversarial min-max optimization problem objective (Section \ref{subsec:optimization}).}
\item{We show empirically on two TCGA data sets the utility of our method in both the unsupervised and the partially supervised settings. We also show that our approach facilitates treatment recommendation that is in 66\% of the cases better than the administered treatment. Additionally, we learn a weight matrix that discovers relations between the different cancer types (Section \ref{sec:eval}).}
\end{itemize}



\section{Background: Survival Analysis}

\subsection{Preliminaries}
\label{sec:Preliminaries}
Survival analysis methods aim at learning the relation between features of individuals and their corresponding survival times (time-to-event). We use the term \textit{instance} instead of \textit{individual} since studied subjects could be humans, animals, or even mechanical parts.
Typically, survival data take the form 
$D =\{(\boldsymbol{x}_i,t_i,\delta_i)| i\in \{1,\dots,n\}\}$, where $n$ is the number of instances, $\boldsymbol{x}_i \in \mathbb{R}^d$ is a vector of covariates, $t_i$ is either the observation time of the event or the censoring time and $\delta_i$ is an event indicator that reveals the status of censoring, i.e., $\delta_i=0$ for censored cases and $\delta_i=1$ otherwise. Censoring occurs when the target event is not observed before the termination of the study; thus, we acquire only the partial information about surviving at least till $t_i$. We consider only right-censoring in which the actual survival time of a censored instance is after the time of the last observation, i.e., censoring time.

\subsection{Survival Functions}
The time-to-event $t$ is a random variable that can be characterized by three functions: \textit{(i)} the probability density function, \textit{(ii)} the survival function, and \textit{(iii)} the hazard function. 
Knowing any of these functions leads to deriving the other two.
Given the random variable $T$, time-to-event, the density function models the probability for the event to occur in infinitesimal interval $[t,t+\Delta t]$, i.e., $f(t) = \, \lim_{\Delta t \to 0} \frac{\Prob \left\{ t < T \leq t + \Delta t \right\}}{\Delta t}$.
The survival function, $S(\cdot)$, models the probability of surviving till time $t$: $S(t) \, = \, \Prob \left\{T > t \right\} = 1 - F(t) = \int_t^\infty f(x)\, dx$, where $F(\cdot)$ is the cumulative distribution function. The conditional probability for the event to occur in the interval $[t,t+\Delta t]$, provided it has not occurred before $t$, is called the hazard function, $\lambda(\cdot)$; $\lambda(t)  = \, \lim_{\Delta t \to 0} \frac{\Prob \left\{ t < T \leq t + \Delta t | T > t \right\}}{\Delta t} = \, \frac{f(t)}{S(t)}$. Both $f(\cdot)$ and $\lambda(\cdot)$ can be derived from $S(\cdot)$ as 
$f(t) = \, \frac{d}{dt}[1-S(t)] = -S'(t)$ 
and $\lambda(t) = \, \frac{f(t)}{S(t)}=\frac{-S'(t)}{S(t)}$.

Since the interest of the three survival functions is instance-wise, in the remaining of the paper, we extend the notation by adding the vector of the instance's covariates as a parameter, i.e., $f(t;\boldsymbol{x})$, $F(t;\boldsymbol{x})$, $S(t;\boldsymbol{x})$ and $\lambda(t;\boldsymbol{x})$.

The proportional hazards (PH) assumption, which is first introduced in the Cox proportional hazards model~\cite{Cox:1972:RMLT}, assumes constant proportionality of hazards between instances over time, i.e., the hazard ratio $HR = \lambda(t;\boldsymbol{x}_1)/\lambda(t;\boldsymbol{x}_2)$ between the instances $\boldsymbol{x}_1$ and $\boldsymbol{x}_2$ is constant. Hence, for an instance $\boldsymbol{x}$, the hazard is the product of the baseline hazard $\lambda_0(t)$ and a time-independent function $r(\boldsymbol{x})$, i.e., $\lambda(t;\boldsymbol{x}) \, = \, \lambda_0(t)\cdot r(\boldsymbol{x})$. The Cox PH model assumes that $r(\boldsymbol{x})$ is a log-linear function of $\boldsymbol{x}$:
\begin{align}
\label{eq:Cox1}
\lambda(t;\boldsymbol{x})  = \, \lambda_0(t) \cdot \exp\left( \sum_{i=1}^n \beta_i \cdot x_i \right) \enspace ,
\end{align}
where $\lambda_0(t)$ is the hazard when all covariates are set to zero~\cite{Lee:1992:SMSDA}.
The coefficients $\beta_i$ are found by maximizing the log of the so-called partial likelihood (PL); this likelihood depends on ordering events instead of their joint probabilities. PL computes the event's conditional probability only for non-censored instances, given their risk set, which contains the surviving instances so far. 

\subsection{Performance Measures}
To estimate the performance of a fitted survival model, evaluation measures compute the agreement between the rank of the predicted survivals and the actual survival times. The concordance, also known as the $\text{C-index}$,~\cite{harrell1982evaluating,harrell1984regression} measures how well a risk model ranks instances according to their estimated hazards, survivals, or predicted death times. To this end, it considers each pair of instances and checks if the model's prediction ranks the two instances in accordance with their true order of events. Each non-censored instance is compared against all instances that outlive it (having a larger event or censoring time). Each correct ranking of pairs is counted as 1, and the final score is normalized over the total number of valid pairs. For example, if the Cox PH model, Eq.~(\ref{eq:Cox1}), is used, the $\text{C-index}$ takes the form
\begin{align}
\label{eq:C-index}
 \text{C-index}(r;D) &= \frac{1}{Z}  \sum_{\substack{(\boldsymbol{x}_i,t_i,\delta_i) \in D\\ \land \delta_i=1}} \, \sum_{
\substack{(\boldsymbol{x}_j,t_j,\delta_j) \in D\\ \land t_j>t_i}} I[r(\boldsymbol{x}_i) > r(\boldsymbol{x}_j)] \enspace , 
\end{align}\noindent where 
$Z = \sum_{\substack{(\boldsymbol{x}_i,t_i,\delta_i) \in D\\ \land \delta_i=1}} \, \sum_{
\substack{(\boldsymbol{x}_j,t_j,\delta_j) \in D\\ \land t_j>t_i}} 1$, $\{(\boldsymbol{x}_j,t_j,\delta_j) \in D| t_j>t_i\}$ is the risk set of the instance $(\boldsymbol{x}_i,t_i,\delta_i)$, 
$r(\cdot)$ is the time-independent risk function, and $I[\cdot]$ is the indicator function. Eq.~(\ref{eq:C-index}) becomes the area under the curve (AUC) when the event times are replaced with the binary problem (event, no event) with no censoring cases; see~\cite{haider2020effective}. 
Alternatively, the loss based on discordance can be computed as in 
$\text{D-index}(r;D) = 1- \text{C-index}(r;D)$.

\section{Multi-Source Survival Domain Adaptation}
\label{sec:Multi-Source Survival Domain Adaptation}
For survival instances $(\boldsymbol{x}_i,t_i,\delta_i)$, let $\mathcal{X}=  \mathbb{R}^d$ and $\mathcal{Y} = \mathbb{R}^+$ be the spaces of the input's covariates and the event time, as described earlier.
Let $\{D_i\}_{i=1}^K$ be $K$ survival source domains characterized by the distributions $P_i$, and let $\{(\boldsymbol{x}_i^j,y_i^j,\delta_i^j)\}_{j=1}^{N_i}$
be the acquired instances for each domain $D_i$. Let $D_t$ be the target survival domain for which samples are described only by their covariates $\boldsymbol{x}_t^j$ and the event indicators $\delta_t^j$, whereas the survival times remain missing, i.e., the instances $\{(\boldsymbol{x}_t^j,?,\delta_t^j)\}_{j=1}^{N_t}$ are observed from $P_{t}$.

Typically, multi-source domain adaptation aims at adapting a model fitted on the source domains to the target domain while minimizing the expected loss on $D_t$. This work considers multi-source survival domain adaptation (MSSDA) where the true survival times $t_t^j$ are unknown.

\subsection{Discordance Loss for Survival Data}\label{subsec:loss}
The $\text{D-index}$ could serve as a loss for risk functions on survival domains; however, it does not enjoy symmetry when the roles of the ground truth and the risk function are exchanged due to censoring cases. We aim to enforce symmetry because it is an essential property for bounding the generalization loss on the target domain. To impose symmetry, we define the symmetric discordance index ($SDI$) for two risk functions $r_1$ and $r_2$, and prove that it is a metric satisfying the triangular inequality.

$SDI$ is composed of two parts \textit{(i)} the disagreements in ranking each pair of non-censored instances, and (ii) the disagreements in ranking pairs of censored and non-censored instances. The weights $\alpha_1$ and $\alpha_2$ transform $SDI$ into a convex combination of these two parts.
Moreover, $SDI$'s symmetry follows from counting the discordance with censored cases twice, once for each of the risk functions while considering the other as the ground truth.
{\fontsize{9.5pt}{10.8pt} \selectfont 
\begin{align}
& SDI(r_1,r_2;D) = \frac{\alpha_1}{\alpha_1+\alpha_2} \frac{1}{Z}  \sum_{\substack{(\boldsymbol{x}_i,t_i,\delta_i),\\(\boldsymbol{x}_j,t_j,\delta_j) \in  D_{ev}\\ i<j}} I\Bigg[
\bigg(\Big(r_1(\boldsymbol{x}_i) < r_1(\boldsymbol{x}_j)\Big) \land 
\Big(r_2(\boldsymbol{x}_i) > r_2(\boldsymbol{x}_j)\Big)\bigg)\nonumber\\
&\lor\bigg(\Big(r_1(\boldsymbol{x}_i) > r_1(\boldsymbol{x}_j)\Big) \land \Big(r_2(\boldsymbol{x}_i) < r_2(\boldsymbol{x}_j)\Big)\bigg)\bigg] 
+ \frac{\alpha_2}{\alpha_1+\alpha_2} \frac{1}{|D_{ce}|} \sum_{\substack{(\boldsymbol{x}_i,t_i,\delta_i) \in D_{ce}}}
\frac{|C_{r_1,\boldsymbol{x}_i} \triangle \ C_{r_2,\boldsymbol{x}_i} |}{|C_{r_1,\boldsymbol{x}_i} \cup C_{r_2,\boldsymbol{x}_i} |} \label{eq:SDI}\\
&\text{s.t.}\; C_{r,\boldsymbol{x}} = 
\Big\{(\boldsymbol{x}_j,t_j,\delta_j) \in D_{ev}| r(\boldsymbol{x}_j) > r(\boldsymbol{x})\Big\} , \
 \alpha_1 = {|D_{ev}|\choose 2}, \ \alpha_2 = |D_{ev}|.|D_{ce}|/2, \
Z = \sum_{\substack{(\boldsymbol{x}_i,t_i,\delta_i),\\(\boldsymbol{x}_j,t_j,\delta_j) \in D_{ev}\\ i<j }} 1  \nonumber\\
&D_{ev} = \{(\boldsymbol{x}_j,t_j,\delta_j) \in D| \delta_j = 1\}, D_{ce} = \{(\boldsymbol{x}_j,t_j,\delta_j) \in D| \delta_j \neq 1\} \ , \nonumber 
\end{align}}\noindent
where $D_{ev} \subseteq D$ and $D_{ce} \subseteq D$ are the sets of non-censored and censored instances, respectively. $C_{r,\boldsymbol{x}}$ is the set of instances (from $D$) that are assumed to outlive $\boldsymbol{x}$, according to the risk function $r$. $\triangle$ is the set symmetric difference (disjunctive union).

Notice that the $SDI$ is equivalent to Kendall's tau distance between two rankings when \textit{(i)} counting 0.5 as a score for ties on the survival times, and \textit{(ii)} no censoring.

Next, we present the formal definition of Kendall tau as a rank distance; thereafter, Theorem~\ref{theorem:SDI-index} proves that SDI is a metric by presenting it as a weighted sum of the Kendall tau and the Jaccard metric.


\begin{definition}{Kendall tau~\cite{cicirello2019kendall}:}\label{def:Kendall tau}
Let $S=\{s_1,\dots,s_n\}$ be the set of $n$ ordered instances, and let $\tau_1$ and $\tau_2$ be two different permutations of instances in $S$, such that for each $s_i\in S$, $\tau(s_i)$ gives the ranking of $s_i$ in the permutation $\tau$. Kendall tau distance~\cite{kendall1955rank} measures the number of pair-wise interventions needed to make two permutations become the same. Kendall tau between the permutations $\tau_1$ and $\tau_2$ is defined as: 
\begin{align}
\kappa(\tau_1,&\tau_2) = \frac{2 K_{d}}{ n*(n-1)} \label{eq:kindalltau1}\\
 K_{d} =& |\{(s_i,s_j) \in S \times S| \, i<j \, \land \nonumber\\
&( ( ( \tau_1(s_i) < \tau_1(s_j)) \land ( \tau_2(s_i) < \tau_2(s_j))) \, \lor \nonumber\\
&( ( \tau_1(s_i) > \tau_1(s_j)) \land ( \tau_2(s_i) > \tau_2(s_j)))) \}| \enspace , 
\end{align}
where $K_{d}$ is the number of discordance pairs.

\end{definition}

\begin{theorem}
	\label{theorem:SDI-index}	
	Given the survival data $D =\{(\boldsymbol{x}_i,t_i,\delta_i)| i\in \{1,\dots,n\}\}$ and the risk estimators $r_1,r_2:\mathbb{R}^d \rightarrow \mathbb{R}^+$, the symmetric discordance index $SDI$ (Eq.~\ref{eq:SDI}) is a metric.
\end{theorem}

The proof of Theorem~\ref{theorem:SDI-index} follows from demonstrating that the $SDI$ is a weighted average of two metrics, Kendall tau, and the Jaccard index. The first term, Kendall tau, is measured over the set of non-censored instances. The second term is the sum of the Jaccard index, for each censored instance, on the two risk sets induced by the ranking function $r_1$ and $r_2$, see the proof in the supplementary material. Establishing $SDI$ as a metric implies that it enjoys the triangular inequality. This in turn is a necessary criterion that will allow us to derive a generalization bound for the target domain.



\subsection{Generalization Bound for Target Domain}\label{subsec:bound}

To derive the bound of the loss on the target domain by that of the source domains, we follow~\cite{cortes2014domain}. We first define a discordance-based distance $D_{SDI-disc}$ to quantify the discrepancy of two distributions $P_s$ and $P_t$, over sets from $\mathcal{X}$, based on the loss $SDI:\mathcal{H}\times\mathcal{H}\times\mathcal{X}^N\rightarrow[0,1]$, where $N$ is the size of the sets over which the distance between two rankings is measured, and $\mathcal{H}$ is the hypothesis space.

\begin{definition}{$D_{SDI-disc}$:}\label{def:discordance_hypothesis_divergence}
	The discordance-based distance ($D_{SDI-disc}$) is the largest distance between two domains (concerning the hypothesis space $\mathcal{H}$) in a metric space equipped with the metric $SDI$ as a distance function. Let $D_s$ and $D_t$ be two survival domains with their corresponding distributions $P_s$ and $P_t$. In survival domains, some samples undergo censoring independent of their features, where the censoring time is bound by the survival time. Each hypothesis in $\mathcal{H}$ is a scoring function that acts as a ranking or a risk function. For the distributions $P_s$ and $P_t$, and $N \in \mathbb{N}$, $D_{SDI-disc}$ takes the form:
	\begin{align}
	&D_{SDI-disc}(P_s,P_t) = 
	\max_{h,h^\prime \in \mathcal{H}} \E_{\substack{
	M_s=\{x_1,\dots,x_{N}\sim P_s\}\\
	M_t=\{x_1,\dots,x_{N}\sim P_t\}
	}} \left| SDI(h,h^\prime;M_s) - SDI(h,h^\prime;M_t)\right| \enspace ,
	\label{eq:d-disc}
	\end{align} \noindent
	where $M_s$ and $M_t$ are the sets of size $N$ from the source and target domains, respectively.
\end{definition}

The discordance distance, $D_{SDI-disc}$ reaches its maximum when two ranking functions $h,h^\prime \in \mathcal{H}$ rank the instances of the survival source domain similarly (high concordance) and differently rank the samples of the target domain (high discordance), or vice-versa. Theorem~\ref{theorem:SDI_bound}	utilizes $D_{SDI-disc}$ as a distance between distributions to bound the discordance loss on the target survival domain.
\begin{theorem}
\label{theorem:SDI_bound}	
Let $S$ be a set of $K$ source survival domains $S=\{D_{s_1},\dots,D_{s_K}\}$ with distributions $P_{s_i}$, and denote the ground truth mapping (risk) function in $D_{s_i}$ as $f_{s_i}$. Similarly, let $D_t$ be a target survival domain with the corresponding distribution $P_t$ and the true risk function $f_t$. 
Assume the following sets: $M_{s_i}=\{x_1,\dots,x_{N}| x_j \sim P_{s_i}\}$ and $M_t=\{x_1,\dots,x_{N}| x_j \sim P_t\}$ to be sampled, of size $N$, from the source domains $D_{S_i}$ in $S$ and the target domain $D_t$, respectively. Also, assume a weighting scheme 
$w_i$ for the source domain $D_{s_i}$ s.t. $\sum_{i=1}^{K}{w_i=1}$. 
For any hypothesis $h\in \mathcal{H}$, the $SDI$ on the target domain $D_t$ is bound in the following way:
\begin{align}
&SDI(r_h,f_t; M_t) \leq \eta_{D}(f_S,f_t) + 
\sum_{i=1}^k w_i \cdot \bigg(SDI(r_h,f_{s_i}; M_{s_i}) + D_{SDI-disc}(P_{s_i},P_t) \bigg)  \enspace ,
\end{align}\noindent
where $r_h$ is the risk (or ranking) function induced by $h$ and 
\[\eta_{D}(f_{S},f_t)= \min_{h^{*} \in \mathcal{H}}
SDI(r_{h^*},f_t; M_t) + \sum_{i=1}^k w_i \cdot SDI(r_{h^*},f_{s_i}; M_{s_i}) 
\]\noindent
is the minimum joint empirical $SDI$ losses on the sources $S$ and the target $D_t$, achieved by an optimal hypothesis $h^{*}$.
\end{theorem}

The supplementary material provides the complete proof of Theorem~\ref{theorem:SDI_bound}. This proof starts by deriving the error bound for a single source domain $D_{s_i}$. Thanks to the metric properties of $SDI$, we prove at first that 
\[
SDI(r_h,f_t; M_t) \leq SDI(r_h,f_{s_i}; M_{s_i})+ SDI(r_{h^*},f_t; M_t)+ SDI(r_{h^*},f_{s_i}; M_{s_i}) + D_{SDI-disc}(P_{s_i},P_t) \enspace .
\]
The proof concludes by reweighing and aggregating this inequality for each source domain. The main outcome of Theorem~\ref{theorem:SDI_bound} is bounding the symmetric discordance on the target domain by the quantities \textit{i)} the weighted average of the $SDI$ on the survival source domains,
\textit{ii)} the weighted mismatch between the target $D_t$ and each of $D_{S_i}$ in terms of the discordance-based distance ($D_{SDI-disc}$), and 
\textit{iii)} the minimum joint empirical $SDI$ losses on the source and target domains. Based on this result, next, we design an optimization objective for survival domain adaptation.

\begin{figure}[t!]
 \centering
\includegraphics[width=.65\linewidth,trim = 150 140 170 0,clip]{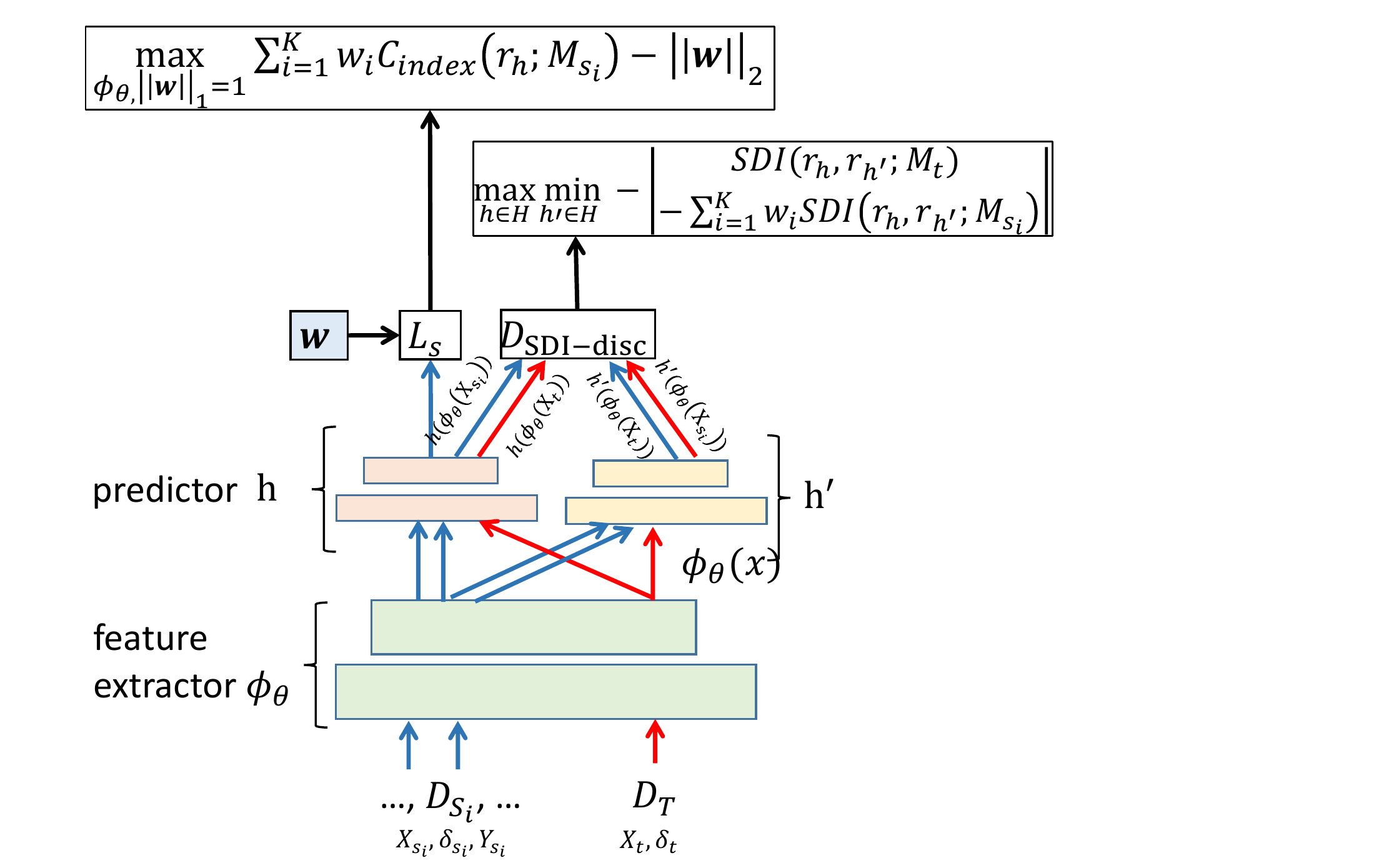}
\caption{An illustration of how symmetric discordance index ($SDI$) is employed in our multi-source survival domain adaptation method, MSSDA. 
The objective includes three terms: 1) the first term enforces the ranking function $r_h$, to be a good ranker, in terms of $\text{C-index}$, on all source domains; and 2)
the second term is an explicit realization of the weighted discordance-based distance ($w_i D_{SDI-disc}(P_t,P_{S_i})$); and 3) the third term is a regularization on the learned weights vector, $\boldsymbol{w}$, that specifies the weight for each source domain concerning the target domain.}
\label{fig:Illustrations_SA}
\end{figure}


\begin{table}[t]
      \centering
    \begin{tabular}{|p{0.004\textwidth}p{0.35\textwidth}p{0.06\textwidth}|p{0.02\textwidth}c|}
        \toprule
ID&Cancer name&Acr.&\multicolumn{2}{|c|}{Instances}\\
&&&\#&$\delta=1$\\
        \midrule
1&Breast Adenocarcinoma&BRCA&707&90\\
2&Glioblastoma Multiforme&GBM&275&176\\
3&Head and Neck Squamous Cell Carci.&HNSC&298&119\\
4&Kidney Renal Clear Cell Carcinoma&KIRC&415&136\\
5&Acute Myeloid Leukaemia&LAML&172&105\\
6&Lung Adenocarcinoma&LUAD&148&49\\
7&Lung Squamous Cell Carcinoma&LUSC&163&68\\
8&Ovarian Serous Carcinoma&OV&315&181\\
        \hline
    \end{tabular}
        \caption{Properties of the mRNA data.}\label{tab:mRNA}
\end{table}
\begin{table}
      \centering
    \begin{tabular}{|p{0.0\textwidth}cp{0.02\textwidth}c|p{0.02\textwidth}c|
    p{0.02\textwidth}c|p{0.015\textwidth}|}
        \toprule
ID& Acronym&\multicolumn{2}{c}{Instances}&\multicolumn{2}{|c|}{Pharma.}&\multicolumn{2}{c|}{Rad.}&TR\\
&&\#&$\delta=1$&\#&$\delta=1$&\#&$\delta=1$&\\
        \midrule
1&ACC&80&29&34&13&2&0&\\
2&BLCA&407&178&111&40&25&15&X\\
3&BRCA&754&105&238&15&31&2&X\\
4&CESC&307&72&4&2&35&9&\\
5&CHOL&36&18&13&7&1&1&\\
6&ESCA&184&77&12&5&22&5&X\\
7&HNSC&484&203&6&2&134&46&\\
8&KIRC&254&76&24&16&7&3&\\
9&KIRP&290&44&18&13&4&4&\\
10&LGG&510&124&50&10&70&16&X\\
11&LIHC&371&128&39&18&12&4&X\\
12&LUAD&441&157&103&36&34&24&X\\
13&LUSC&338&137&72&20&19&13&X\\
14&MESO&86&73&29&26&2&1&\\
15&PAAD&178&93&75&41&-&-&\\
16&SARC&259&98&46&22&48&15&X\\
17&SKCM&97&26&22&9&1&0&\\
18&STAD&382&147&106&42&1&0&\\
19&UCEC&410&72&55&18&82&10&X\\
20&UCS&56&34&15&12&5&5&\\
21&UVM&80&23&11&6&4&2&\\
        \hline
    \end{tabular}
        \caption{Properties of the miRNA data. The treatment columns (Pharmaceutical and Radiation) are collected by matching the data with The Cancer Genome Atlas (TCGA). The TR column indicates whether or not the cancer type is used for evaluating the treatment recommendation.}\label{tab:miRNA}
\end{table}

\subsection{Optimization Problem of Multi-Source Survival Domain Adaptation}\label{subsec:optimization}
We exploit the bound derived in Theorem~\ref{theorem:SDI_bound} to enforce distribution matching through an adversarial min-max optimization objective, following domain-adversarial neural networks (DANN)~\cite{ganin2016domain}.
To this end, we search in the hypothesis space $\mathcal{H}$, where each $h \in \mathcal{H}$ defines a risk function $r_h$, the time-independent function in the hazard Eq.~(\ref{eq:Cox1}). Thus, keeping the proportional hazards assumption.
Formally, the hypotheses in $\mathcal{H}$ take the form $h:\mathcal{V}\rightarrow\mathbb{R}$, where $\mathcal{V}$ is the feature space. We also search for the feature extractor $\phi_\theta:\mathcal{X}\rightarrow\mathcal{V}$, and the weighting $\boldsymbol{w}$ of the source domains, such that:
\begin{align}
    & \max\limits_{\substack{\phi_\theta,h\in \mathcal{H}\\ ||\boldsymbol{w}||_1=1}} \min\limits_{h^{\prime} \in \mathcal{H}} \left( 
    \sum_{i=1}^{K} w_i \text{C-index}(r_h;M_{s_i})-
	\lambda_1 \left|SDI(r_h,r_{h^\prime};M_t)  -\sum_{i=1}^{K}
	w_i SDI(r_h,r_{h^\prime};M_{s_i})\right|
	 -  \lambda_2 ||\boldsymbol{w}||_2  \right) \enspace ,
	\label{eq:SDI_objective}
\end{align}
\noindent
where $r_h$ and $r_{h^\prime}$ are the ranking functions induced by the hypotheses $h$ and $h^\prime$ respectively. The first term of Eq.~(\ref{eq:SDI_objective}) enforces the ranking function $r_h$, to be a good ranker, in terms of $\text{C-index}$, on all source domains; this term is realized by minimizing the negative log-partial likelihood. 
The second term is an explicit realization of the weighted discordance-based distance ($w_i D_{SDI-disc}(P_t,P_{S_i})$). The third term is a regularization on the learned weights vector, $\boldsymbol{w}$, that specifies the weight for each source domain concerning the target domain.
This adversarial min-max game aims at finding, for the survival source and target domains, a feature extractor $\phi_\theta$ and a ranker $r_h$ such that for any other ranker $r_{h^\prime}$, the weighted distance is minimized, i.e., achieving feature invariance of the target domain and each of the source domains (in a weighted manner). We term our method the multi-source survival domain adaptation as (MSSDA), and acknowledge that comparable min-max objectives were used in~\cite{pei2018multi,saito2019semi,richard2020unsupervised,shaker2022learningVCD} outside of survival analysis.
Notice that this algorithm does not optimize for $\eta_{D}$ since this term is constant for a single source domain and for the mixture of sources in $\eta_{D}(f_{S},f_t)$ given the weighting $\boldsymbol{w}$.

Figure~\ref{fig:Illustrations_SA} depicts a graphical illustration of the proposed optimization problem; it shows the details of our method, MSSDA.
$X_{s_i}$, $\delta_{s_i}$ and $Y_{s_i}$ are the input samples, the censoring indicators, and the survival times from the source domain $D_{s_i}$; $X_t$ and $\delta_t$ are the input samples and the censoring indicators of the target domain without survival times. The hypothesis $h$ is trained to produce a good ranker $r_h$ in terms of the weighted $\text{C-index}$ on the sources. The hypothesis $h^\prime$ tries to increase the discordance-based distance ($D_{SDI-disc}$) between the target distribution and the weighted combination of source domains (i.e., $D_{SDI-disc}$).

\begin{figure*}
     \centering
     
     \begin{subfigure}[b]{.49\textwidth}
         \centering
         \includegraphics[width=.98\textwidth]{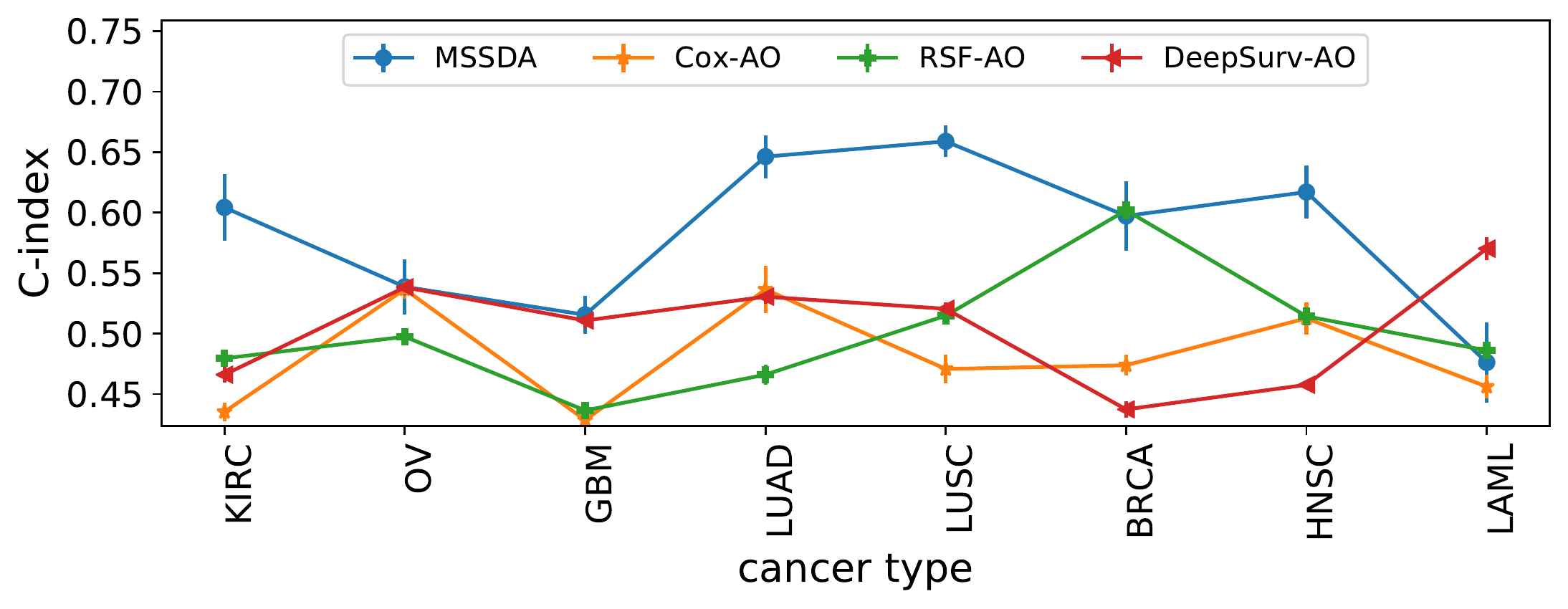}
         \caption{$\text{C-index}$ on mRNA with no supervision.}
         \label{fig:mRNANo_Supervision_RSF}
     \end{subfigure}
     \begin{subfigure}[b]{.49\textwidth}
         \centering
         \includegraphics[width=.98\textwidth]{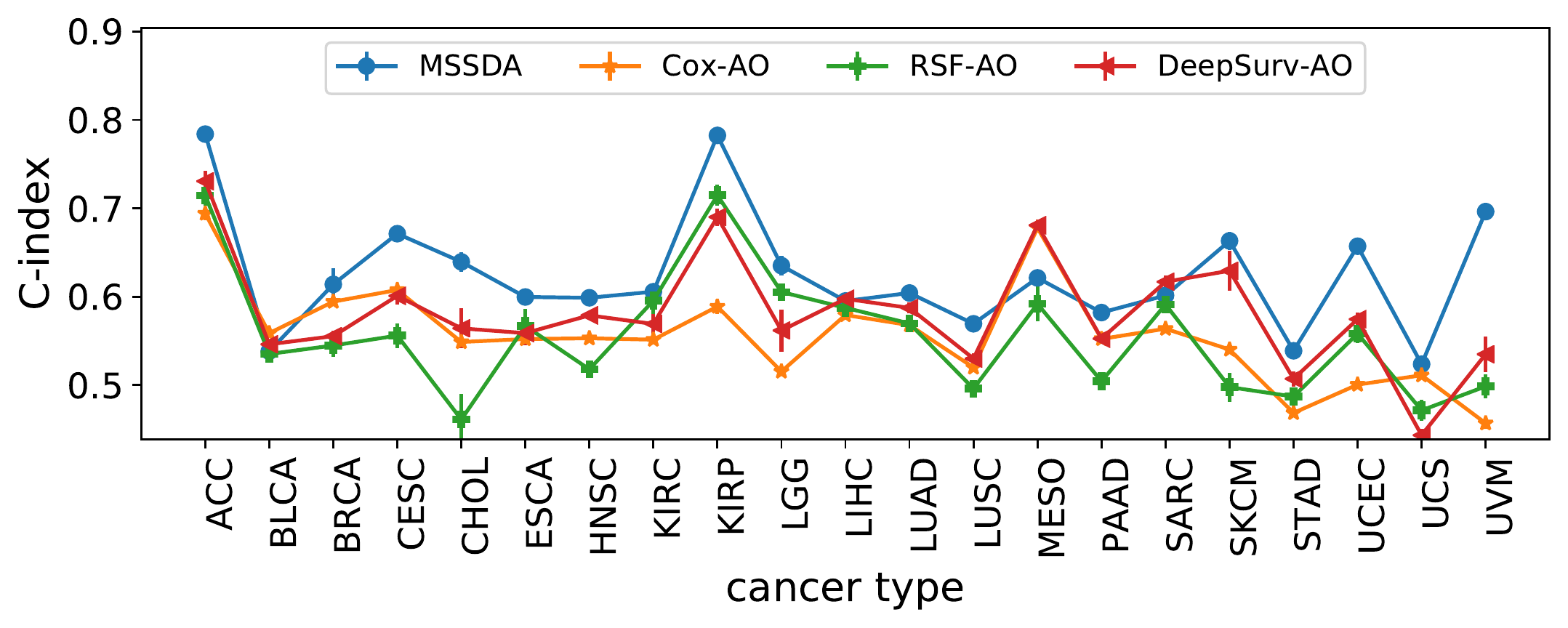}
         \caption{$\text{C-index}$ on miRNA with no supervision.}
         \label{fig:miRNANo_Supervision_RSF}
     \end{subfigure}
     
     \begin{subfigure}[b]{0.49\textwidth}
         \centering
         \includegraphics[width=.98\linewidth]{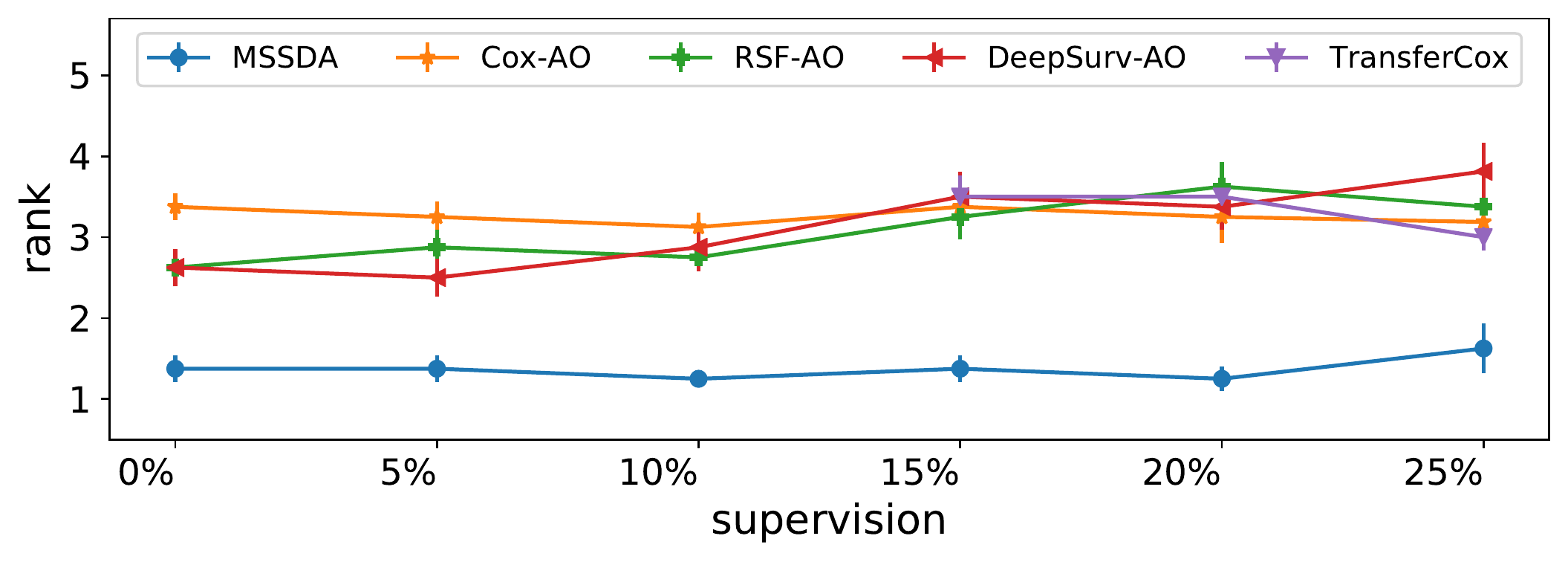}
         \caption{Rank of the different methods based on the $\text{C-index}$ on mRNA.}
         \label{fig:mRNA_RSF_Rank_c_index_rem}
     \end{subfigure}
     \begin{subfigure}[b]{0.49\textwidth}
         \centering
         \includegraphics[width=.98\linewidth]{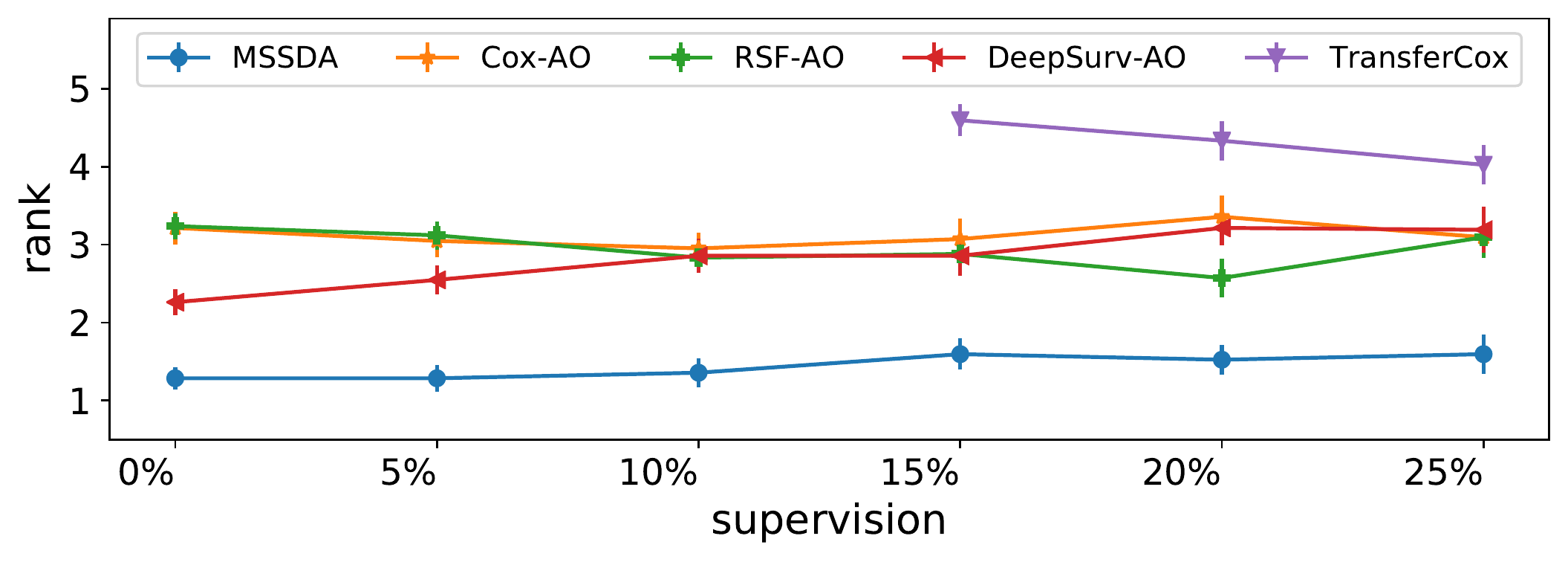}
         \caption{Rank of the different methods based on the $\text{C-index}$ on miRNA.}
         \label{fig:miRNA_RSF_Rank_c_index_rem}
     \end{subfigure}
        \caption{Performance comparison and ranking, on the miRNA and mRNA data, in terms of the $\text{C-index}$.}
        \label{fig:miRNA_mRNA_Rank_C-index_and_rem_strict_C-index}
\end{figure*}
\section{Empirical Evaluation}\label{sec:eval}
To investigate the usefulness of our proposed method to adapt to a target survival domain, we address the following three questions:
\begin{itemize}
    \item Does the multi-source domain adaption work on survival target domains? How does it perform if the labels for a portion of the target data were used? (Section \ref{subsec:eval}.)
    \item Can we recommend treatment better than what was offered to the patients? (Section \ref{subsec:rec}.)
    \item Do the learned weights on the source domains reveal any useful information about the underlying cancer types? (Section \ref{subsec:xai}.)
    \item How essential is the proposed symmetric discordance index ($SDI$) for aligning the conditional distributions compared to other domain-invariant regularisation approaches? (Section \ref{subsec:ablation}.)
\end{itemize}

\paragraph{Datasets.} We utilize two data sets from The Cancer Genome Atlas project (TCGA)\footnote{\url{https://www.cancer.gov/about-nci/organization/ccg/research/structural-genomics/tcga}}. This project analyzes the molecular profiles and the clinical data of 33 cancer types. 
\textit{(i)} The Messenger RNA data (mRNA)~\cite{li2016transfer}, which includes eight cancer types. Each patient is represented by 19171 binary features; see Table~\ref{tab:mRNA}.
\textit{(ii)} The micro-RNA data (miRNA) that includes 21 cancer types~\cite{wang2017multi}; each has a varying number of patients. Table~\ref{tab:miRNA} depicts the total number of patients for each cancer and the number of patients that experienced the event (died) during the time of the clinical study ($\delta=1$).
We also extract the treatment performed for each cancer type (if available).

\paragraph{Baselines.} For the evaluation, we compare with 1) the Cox proportional hazards model fitted by maximizing the log of the partial likelihood, 2) DeepSurv~\cite{katzman2018deepsurv} that introduces the proportional hazards to neural networks, and 3) the survival random forests (RSF)~\cite{ishwaran2007random}\footnote{\url{https://square.github.io/pysurvival/}}.
These methods deal with single domains; therefore, we perform separate training on each source domain and use the trained model as a ranking function for the target domain. Each ranker orders the target's instances; we average these orders over all rankers, hence, the abbreviation \textquote{average order} (AO). To answer the second part of the first question, we compare with 4) TransferCox~\cite{li2016transfer}, a transfer learning method that employs multi-task learning on survival domains and requires labels in all domains without prioritizing the target domain. 

For both MSSDA\footnote{https://github.com/shaker82/MSSDA} and DeepSurv, we use the same architecture, a two-layered feature extractor with 200 and 20 units in the first and second hidden layers, respectively. The detailed architecture and the hyper-parameters search are explained in the supplementary material. We model the log-risk function as the non-linear function $h\circ \phi_\theta(x)$ learned by the fitted network architecture, i.e., $r(x) = e^{h\circ \phi_\theta(x)}$. MSSDA and DeepSurv are trained for 20 epochs. 

In the supervised target case, we allow a small portion of the target domain to be labeled and used for training. We use these percentages, 5\%, 10\%, 15\%, 20\%, and 25\%. Except for the TransferCox, the target samples are appended to each source domain's samples. For TransferCox, the target samples are added as a new domain, which is why TransferCox can not be tested when the target domain contains very few samples (less than 20\%).

\paragraph{Evaluation.} To measure the performance of each method, we employ the $\text{C-index}$, Eq.(\ref{eq:C-index}), to measure the concordance between the inversely ordered predicted risks and the actual lifetime. For the unsupervised and supervised cases, we measure the $\text{C-index}$ on the target domain's samples after removing the ones used for the supervision. We also propose $\text{C-index}$${^\prime}$ that measures the concordance on the whole target domain, including the samples used for supervision. Notice that $\text{C-index}$${^\prime}$ includes only a tiny portion of the pairs used for the training. In the case of 25\% supervision, we show in the supplementary material the advantage of measuring the $\text{C-index}{^\prime}$ and that the ratio of reused pairs is only 6.25\%. All results are averaged over five folds.

\paragraph{Optimizing the $SDI$} 
The counting-based comparison in the SDI is implemented using the MarginRankingLoss MRS\footnote{\url{https://pytorch.org/docs/stable/generated/torch.nn.MarginRankingLoss.html}}: $MRS(x1,x2,y) = max(0,-y(x1-x2)+m)$, where $m$ is the margin. 
For example, we implement $I\Big((r_1(\boldsymbol{x}_i) < r_2(\boldsymbol{x}_j)\Big)$ in Eq. (\ref{eq:SDI}) 
using the surrogate $MRS(exp(r_1(\boldsymbol{x}_i) - r_2(\boldsymbol{x}_j)),0,1)$ with $m=1$.

\subsection{Evaluation of Survival Prediction}\label{subsec:eval}
Table~\ref{tab:C-index-rem-mRNA-RSF_0_5} depicts the performance in terms of the $\text{C-index}$ on the mRNA data; it shows that MSSDA outperforms all other methods in both the unsupervised and the partially supervised (5\%) cases while always achieving the first rank (the last row).
In the supplementary material, Tables~\ref{tab:Supp-mRNA-C-index-rem-1-RSF} and~\ref{tab:Supp-mRNA-C-index-rem-2-RSF} show that MSSDA still dominates the remaining supervision settings at 10\%, 15\%, 20\%, and 25\%;
these results are graphically depicted in Figures~\ref{fig:mRNANo_Supervision_RSF} and~\ref{fig:mRNA_RSF_Rank_c_index_rem} and confirm the superiority of MSSDA performance and its first rank. In general, MSSDA performs best on five of seven cancer types, achieving the best rank, followed by RSF in low supervision and TransferCox in high supervision settings. A similar performance is evident when considering the $\text{C-index}^{\prime}$, as confirmed in Tables~\ref{tab:Supp-mRNA-C-index-1-RSF} and \ref{tab:Supp-mRNA-C-index-2-RSF} in the supplementary material.

\begin{table}[t!]
\centering
\begin{tabular}{|lp{0.05\textwidth}p{0.05\textwidth}p{0.05\textwidth}p{0.05\textwidth}|p{0.05\textwidth}p{0.05\textwidth}p{0.05\textwidth}p{0.05\textwidth}|}
\toprule
superv.&\multicolumn{4}{c|}{.00\%} & \multicolumn{4}{c|}{5\%}\\
method& \rot{MSSDA} & \rot{Cox-AO} & \rot{RSF-AO} & \rot{\shortstack[l]{DeepSurv\\-AO}}  & \rot{MSSDA} & \rot{Cox-AO} & \rot{RSF-AO} & 
\rot{\shortstack[l]{DeepSurv\\-AO}}\\
\midrule

KIRC &  \textbf{.604} &         .435 &           .480 &              .466 &  \textbf{.618} &         .444 &         .521 &              .479   \\
&    \tiny{(.028)} &         \tiny{(.008)} &         \tiny{(.002)} &              \tiny{(.007)} &    \tiny{(.019)} &         \tiny{(.007)} &         \tiny{(.004)} &              \tiny{(.004)}   \\
OV   &  \textbf{.539} &         .537 &           .497 &              .538 &  \textbf{.563} &         .523 &         .490 &              .552   \\
&    \tiny{(.023)} &         \tiny{(.008)} &         \tiny{(.002)} &              \tiny{(.002)} &    \tiny{(.015)} &         \tiny{(.013)} &         \tiny{(.006)} &              \tiny{(.002)}   \\
GBM  &  \textbf{.516} &         .428 &           .436 &              .511 &           .493 &         .443 &         .494 &     \textbf{.515}   \\
&    \tiny{(.016)} &         \tiny{(.004)} &         \tiny{(.002)} &              \tiny{(.006)} &    \tiny{(.017)} &         \tiny{(.009)} &         \tiny{(.008)} &              \tiny{(.007)}   \\
LUAD &  \textbf{.646} &         .536 &           .466 &              .530 &  \textbf{.661} &         .556 &         .466 &              .507   \\
&    \tiny{(.018)} &         \tiny{(.020)} &         \tiny{(.008)} &              \tiny{(.006)} &    \tiny{(.014)} &         \tiny{(.020)} &         \tiny{(.015)} &              \tiny{(.005)}   \\
LUSC &  \textbf{.659} &         .471 &           .515 &              .520 &  \textbf{.658} &         .533 &         .494 &              .508   \\
&    \tiny{(.013)} &         \tiny{(.012)} &         \tiny{(.005)} &              \tiny{(.005)} &    \tiny{(.024)} &         \tiny{(.007)} &         \tiny{(.029)} &              \tiny{(.006)}   \\
BRCA &           .597 &         .474 &  \textbf{.602} &              .437 &  \textbf{.599} &         .492 &         .536 &              .418   \\
&    \tiny{(.029)} &         \tiny{(.009)} &         \tiny{(.004)} &              \tiny{(.007)} &    \tiny{(.033)} &         \tiny{(.019)} &         \tiny{(.017)} &              \tiny{(.009)}   \\
HNSC &  \textbf{.617} &         .513 &           .514 &              .458 &  \textbf{.646} &         .430 &         .504 &              .441   \\
&    \tiny{(.022)} &         \tiny{(.013)} &         \tiny{(.003)} &              \tiny{(.003)} &    \tiny{(.017)} &         \tiny{(.009)} &         \tiny{(.016)} &              \tiny{(.005)}   \\
LAML &           .476 &         .456 &           .486 &     \textbf{.570} &           .533 &         .451 &         .514 &     \textbf{.551}   \\
&    \tiny{(.033)} &         \tiny{(.010)} &         \tiny{(.002)} &              \tiny{(.010)} &    \tiny{(.024)} &         \tiny{(.003)} &         \tiny{(.014)} &              \tiny{(.008)}   \\
\bottomrule
P-value&
&\scriptsize{5e-3}&\scriptsize{2e-2}&\scriptsize{2e-2}&
&\scriptsize{3e-3}&\scriptsize{9e-3}&\scriptsize{9e-3}
\\
\bottomrule
Rank& \textbf{1.38}& 3.38& 2.63& 2.63& \textbf{1.38}& 3.25& 2.88& 2.50 \\
\bottomrule
\end{tabular}
\caption{The performance comparison on the eight cancer types in the mRNA data in terms of $\text{C-index}$. The following settings were used: no supervision and 5\% supervision. The numbers in brackets depict the standard error. The last row shows the rank in each supervision group.
The p-value row depicts the p-value for the upper-tailed Wilcoxon signed-ranks test between each method and MDSSA.
The null hypothesis can be rejected at the significance level of $0.05$.}\label{tab:C-index-rem-mRNA-RSF_0_5}
\end{table}

\begin{table*}
      \centering
    \begin{tabular}{|p{0.005\textwidth}p{0.05\textwidth}|p{0.03\textwidth}p{0.05\textwidth}|p{0.03\textwidth}p{0.05\textwidth}||cp{0.0\textwidth}c|cp{0.0\textwidth}c|cp{0.0\textwidth}c|cp{0.0\textwidth}c|}
        \toprule
&&\multicolumn{4}{|c||}{[A] Success rate on 5 folds}&\multicolumn{12}{c|}{[B] Median time of all folds merged}\\
\cline{3-18}
&&\multicolumn{2}{|c|}{MSSDA}&\multicolumn{2}{c||}{DeepSurv-AO}
&\multicolumn{6}{c|}{MSSDA}&\multicolumn{6}{c|}{DeepSurv-AO}\\
&&Rad.&Pharma.&Rad.&Pharma.&
\multicolumn{3}{c|}{Rad.}&\multicolumn{3}{|c|}{Pharma.}&
\multicolumn{3}{|c|}{Rad.}&\multicolumn{3}{|c|}{Pharma.}\\
ID&Name&\rot{\shortstack[l]{success\\ratio}}&\rot{\shortstack[l]{success\\ratio}}&
\rot{\shortstack[l]{success\\ratio}}&\rot{\shortstack[l]{success\\ratio}}&
\rot{\shortstack[l]{median\\recomm.}}&&\rot{\shortstack[l]{median\\anti-recom.}}&
\rot{\shortstack[l]{median\\recomm.}}&&\rot{\shortstack[l]{median\\anti-recom.}}&
\rot{\shortstack[l]{median\\recomm.}}&&\rot{\shortstack[l]{median\\anti-recom.}}&
\rot{\shortstack[l]{median\\recomm.}}&&\rot{\shortstack[l]{median\\anti-recom.}}\\

\midrule
0 & BLCA &2/5 & 2 / 5&4/5&2/5&\st{370}&&\st{370}&536.5&$\ngeqslant$&547&651&$\ge$&324&539&$\ngeqslant$&544\\
1 & BRCA &1/2 & 4 / 5&3/3&3/5&\st{1330.5}&&\st{1330.5}&\st{1032}&&\st{1032}&2296&$\ge$&365&\st{1032}&&\st{1032}\\
2 & ESCA &2/4 & 4 / 4&3/5&3/5&\st{283}&&\st{283}&496&$\ge$&480&\st{283}&&\st{283}&496&$\ge$&480\\
3 & LGG &4/5 & 3 / 5&1/5&4/5&1368&$\ge$&794&1106&$\ge$&933&1011&$\ngeqslant$&1335&1106&$\ge$&758\\
4 & LIHC &5/5 & 4 / 5&1/4&1/5&643&$\ge$&432&639&$\ge$&633&432&$\ngeqslant$&643&612&$\ngeqslant$&639\\
5 & LUAD &4/5 & 3 / 5&2/5&2/5&677&$\ge$&561&574&$\ngeqslant$&594&\st{633}&&\st{633}&503&$\ngeqslant$&594\\
6 & LUSC &2/5 & 2 / 5&2/5&1/5&387&$\ge$&345&559&$\ngeqslant$&562&\st{387}&&\st{387}&559&$\ngeqslant$&573\\
7 & SARC &3/5 & 5 / 5&3/5&2/5&\st{695}&&\st{695}&1013.5&$\ge$&550&\st{695}&&\st{695}&591&$\ngeqslant$&599\\
8 & UCEC &4/5 & 2 / 5&1/5&3/5&1279&$\ge$&1127&666&$\ge$&610&1016&$\ngeqslant$&1317&670&$\ge$&610\\
        \hline
$\sum$ & &\textbf{5.8/9} & \textbf{6/9}&4.5/9&4.2/9
&\multicolumn{3}{c|}{\textbf{5/5}}
&\multicolumn{3}{c|}{\textbf{5/8}}
&\multicolumn{3}{c|}{2/5}
&\multicolumn{3}{c|}{3/8}\\        
        \hline
    \end{tabular}
        \caption{Results of the treatment recommendation experiments on the miRNA data. Table A, on five folds, depicts the ratios of better-recommended treatments over the valid folds. Table B presents the median survival times of the recommendation and anti-recommendation groups across all folds after removing the censored patients.}\label{tab:Treatment_recommendation}
\end{table*}

Similarly, MSSDA achieves a superb performance on miRNA on supervision and no supervision settings, as confirmed by the dominating $\text{C-index}$ in Figure~\ref{fig:miRNANo_Supervision_RSF}, and the best rank in Figure~\ref{fig:miRNA_RSF_Rank_c_index_rem}. These figures summarize Tables~\ref{tab:Supp-miRNA-C-index-rem-1-RSF}, \ref{tab:Supp-miRNA-C-index-rem-2-RSF}, \ref{tab:Supp-miRNA-C-index-1-RSF}, and \ref{tab:Supp-miRNA-C-index-2-RSF}
and Figures~\ref{fig:all_miRNA_C-index_rem_strict} and \ref{fig:all_miRNA_C-index_strict} that are kept in the supplementary material. Again MSSDA performs best on 17, 18, 18, 15, 15, 16 out of 21 cancer types for the 0\%, 5\%, 10\% 15\%, 20\% and 25\% supervision settings, respectively. Hence, MSSDA always ranks first.

The supplementary material includes Figure~\ref{fig:Supp_miRNA_mRNA_Rank_C-index_and_rem_strict_C-index} that depicts the rank when using $\text{C-index}$${^\prime}$, and a deeper discussion on the discrepancy between the result of $\text{C-index}$ and $\text{C-index}^{\prime}$.

We compute the p-value for the upper-tailed Wilcoxon signed-ranks test between each method and MDSSA in each setting on both data sets.
The null hypothesis can be rejected on all mRNA data at the significance level of $\alpha=0.05$ for both performance measures. On the miRNA, the null hypothesis can be rejected in all cases at $\alpha=0.1$ except for the two cases (RSF, 20\% supervision, $\text{C-index}$) and (TransferCox, 25\% supervision, $\text{C-index}^\prime$).



\subsection{Treatment Recommendation}\label{subsec:rec}

For the treatment recommendation experiments, we collect the type of treatment (either pharmaceutical (P) or radiation (R)) given for each of the patients (if available). Table \ref{tab:miRNA} depicts the number of cases for each treatment type. We finally select only those cancer types with at least two non-censored and one censored patient for each treatment type, as indicated in column \textquote{TR}; see Table~\ref{tab:miRNA}.
Following the procedure proposed in DeepSurv~\cite{katzman2018deepsurv}, we annotate each instance from  the source domains by a dummy binary attribute identifying the type of treatment (P or R). 
After learning on the source domains, for each sample $\boldsymbol{x}$ from the target domain, we measure the recommendation $rec(x)_{PR} = log \frac{r (x^P)}{r (x^R)}= h\circ \phi_\theta(x^P) - h\circ \phi_\theta(x^R)$, where $\boldsymbol{x}^P$ and $\boldsymbol{x}^R$ are the same target domain sample once considered to be treated by pharmaceutical and once by radiation, respectively. A positive $rec(x)_{PR}$ means that the patient has a higher risk when treated by \textquote{P} than when treated by \textquote{R}. Hence, it is recommended to prescribe \textquote{R}. By comparing with the ground truth treatments, we group the patients into the $\Upsilon_{recom}$ group when the recommended treatment aligns with the true treatment and the $\Upsilon_{anti}$ group containing patients that received a recommendation contradicting the actual treatment. Thereafter, the median survival times of the two groups are compared. 

A smaller median survival time of the $\Upsilon_{anti}$ indicates that the patients could have had a potentially longer survival time had they been given the model's recommended treatment.
MSSDA is employed as in the previous experiments using the labeled multi-source domains and an unlabelled target domain.
For DeepSurv, we train a different model for each source domain and then compute the average order (rank) of each sample of the target domain for each treatment. On DeepSurv, the groups $\Upsilon_{recom}$ and $\Upsilon_{anti}$ are computed using the difference in the predicted patients' ranks for the two treatments. 
Cox model is omitted since it recommends the same treatment for all instances, as proven in ~\cite{katzman2018deepsurv}. TransferCox is also omitted since it requires labeled target data.


Part[A] of Table~\ref{tab:Treatment_recommendation} shows that in the case of contradiction with the administered treatment, our method in 64.4\% and 66\% of the cases gives a better recommendation for the radiation and pharmaceutical treatments, compared to 50\% and 46.6\% achieved by DeepSurv. This result is computed over five folds for each treatment and each cancer type. Part[B] of Table~\ref{tab:Treatment_recommendation} shows a detailed comparison of the median survival time for each treatment and cancer pair when merging the samples of all folds. The medians are struck through upon equality and compared otherwise. Again, the results show a median survival time in the $\Upsilon_{anti}$ group smaller than that of the $\Upsilon_{recom}$ group in 5/5 and 5/8 of the cases when MSSDA is employed. Whereas, Deepsurv achieves this only on 2/5 and 3/8 of the cases. RSF fails experimentally to identify and employ the treatment indicator, which has led to failing to induce two different risk models for the different treatments. Therefore, RSF is omitted in the experiment.


\begin{figure}
\centering
\includegraphics[scale=0.45]{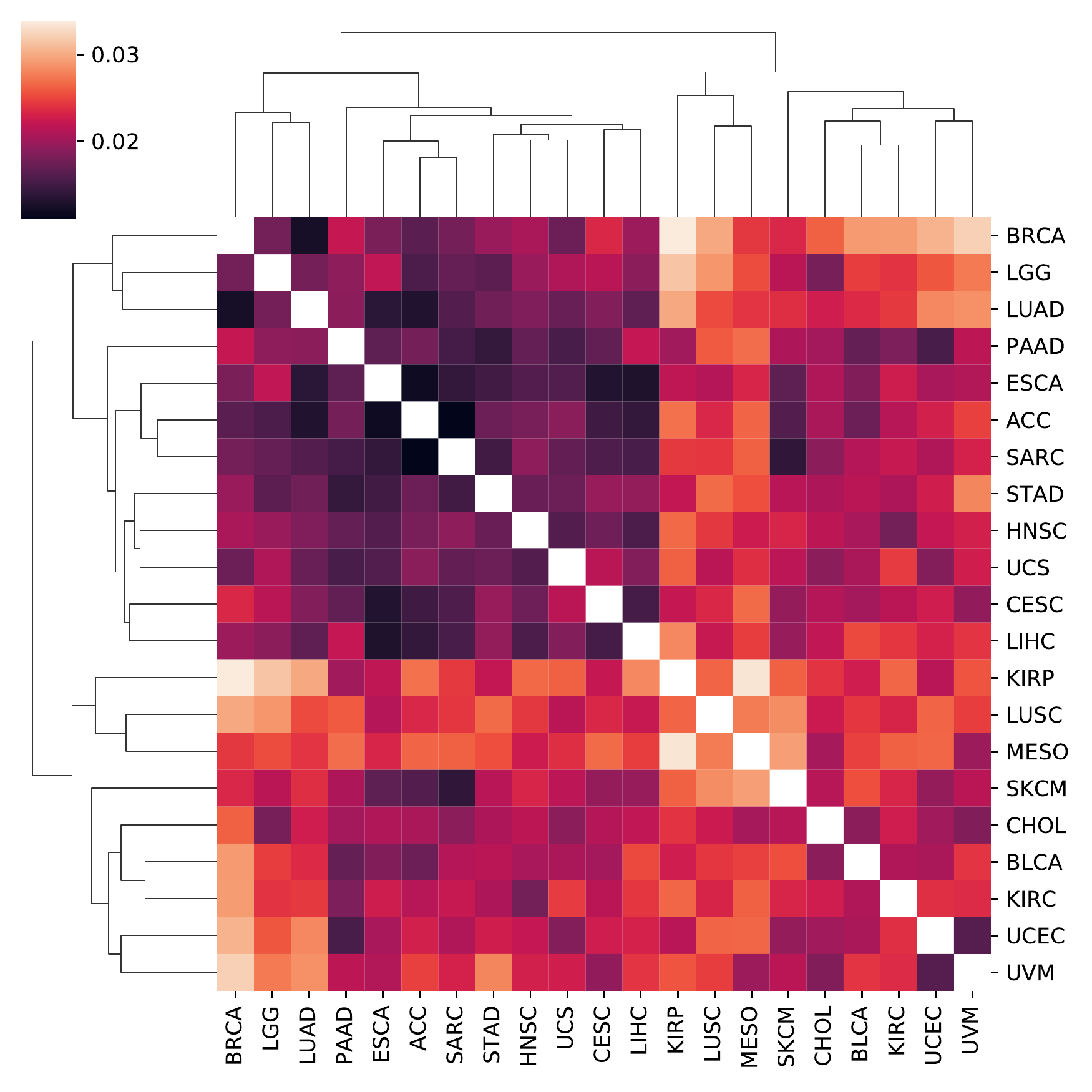}
\caption{Heatmap of the matrix computed from the learned weights' distances on the miRNA data.}
\label{fig:heatmap_weights_miRNA}
\end{figure}

\begin{table}
\centering
\begin{tabular}{|p{.07\textwidth}p{.07\textwidth}p{.07\textwidth}p{.07\textwidth}p{.07\textwidth}p{.07\textwidth}p{.07\textwidth}|}
\toprule
\rot{cancer}& \rot{MSSDA} & \rot{MDAN} & \rot{\shortstack[l]{KuiperUB\\-KM}} & \rot{MMD} & \rot{MMD-KM}  & \rot{D-index} \\

\midrule
ACC&\textbf{.784} \tiny{(.008)}&.729 \tiny{(.004)}&.690 \tiny{(.005)}&.688 \tiny{(.006)}&.682 \tiny{(.005)}&.702 \tiny{(.008)}\\
BLCA&\textbf{.538} \tiny{(.007)}&.507 \tiny{(.003)}&.505 \tiny{(.001)}&.513 \tiny{(.001)}&.507 \tiny{(.001)}&.520 \tiny{(.009)}\\
BRCA&\textbf{.614} \tiny{(.018)}&.560 \tiny{(.001)}&.538 \tiny{(.003)}&.542 \tiny{(.000)}&.534 \tiny{(.001)}&.562 \tiny{(.011)}\\
CESC&\textbf{.671} \tiny{(.005)}&.620 \tiny{(.004)}&.615 \tiny{(.002)}&.606 \tiny{(.002)}&.617 \tiny{(.002)}&.632 \tiny{(.013)}\\
CHOL&\textbf{.639} \tiny{(.011)}&.553 \tiny{(.005)}&.579 \tiny{(.007)}&.598 \tiny{(.002)}&.582 \tiny{(.007)}&.582 \tiny{(.009)}\\
ESCA&\textbf{.600} \tiny{(.008)}&.557 \tiny{(.005)}&.555 \tiny{(.004)}&.561 \tiny{(.004)}&.553 \tiny{(.004)}&.589 \tiny{(.016)}\\
HNSC&\textbf{.599} \tiny{(.004)}&.579 \tiny{(.001)}&.561 \tiny{(.001)}&.564 \tiny{(.001)}&.560 \tiny{(.000)}&.576 \tiny{(.010)}\\
KIRC&\textbf{.606} \tiny{(.008)}&.571 \tiny{(.003)}&.575 \tiny{(.003)}&.583 \tiny{(.002)}&.576 \tiny{(.003)}&.588 \tiny{(.010)}\\
KIRP&\textbf{.782} \tiny{(.004)}&.707 \tiny{(.003)}&.677 \tiny{(.002)}&.669 \tiny{(.005)}&.676 \tiny{(.003)}&.691 \tiny{(.013)}\\
LGG&\textbf{.635} \tiny{(.011)}&.566 \tiny{(.003)}&.553 \tiny{(.004)}&.542 \tiny{(.005)}&.554 \tiny{(.004)}&.565 \tiny{(.013)}\\
LIHC&\textbf{.595} \tiny{(.003)}&.554 \tiny{(.001)}&.546 \tiny{(.002)}&.554 \tiny{(.002)}&.542 \tiny{(.002)}&.561 \tiny{(.011)}\\
LUAD&\textbf{.604} \tiny{(.004)}&.566 \tiny{(.002)}&.547 \tiny{(.001)}&.544 \tiny{(.001)}&.545 \tiny{(.001)}&.559 \tiny{(.008)}\\
LUSC&\textbf{.569} \tiny{(.003)}&.554 \tiny{(.001)}&.539 \tiny{(.004)}&.538 \tiny{(.004)}&.537 \tiny{(.004)}&.554 \tiny{(.007)}\\
MESO&.621 \tiny{(.002)}&\textbf{.632} \tiny{(.003)}&.600 \tiny{(.003)}&.588 \tiny{(.002)}&.595 \tiny{(.002)}&.610 \tiny{(.006)}\\
PAAD&\textbf{.582} \tiny{(.004)}&.568 \tiny{(.001)}&.555 \tiny{(.001)}&.553 \tiny{(.001)}&.553 \tiny{(.001)}&.570 \tiny{(.005)}\\
SARC&\textbf{.601} \tiny{(.006)}&.573 \tiny{(.008)}&.571 \tiny{(.005)}&.573 \tiny{(.004)}&.571 \tiny{(.004)}&.583 \tiny{(.010)}\\
SKCM&\textbf{.663} \tiny{(.011)}&.595 \tiny{(.010)}&.551 \tiny{(.003)}&.533 \tiny{(.004)}&.545 \tiny{(.005)}&.566 \tiny{(.010)}\\
STAD&.539 \tiny{(.006)}&.515 \tiny{(.004)}&.527 \tiny{(.003)}&.531 \tiny{(.003)}&.526 \tiny{(.003)}&\textbf{.543} \tiny{(.011)}\\
UCEC&\textbf{.657} \tiny{(.007)}&.547 \tiny{(.005)}&.529 \tiny{(.006)}&.531 \tiny{(.004)}&.526 \tiny{(.005)}&.548 \tiny{(.014)}\\
UCS&\textbf{.524} \tiny{(.007)}&.496 \tiny{(.003)}&.496 \tiny{(.003)}&.493 \tiny{(.003)}&.494 \tiny{(.005)}&.504 \tiny{(.006)}\\
UVM&\textbf{.696} \tiny{(.010)}&.537 \tiny{(.013)}&.551 \tiny{(.003)}&.534 \tiny{(.006)}&.553 \tiny{(.004)}&.586 \tiny{(.026)}\\

\bottomrule
Rank& \textbf{1.1}& 3.19 & 4.57 & 4.62 & 5.1 & 2.43\\
\bottomrule
\end{tabular}
\caption{The performance comparison in the ablation analysis on the miRNA data in terms of $\text{C-index}$. The numbers in brackets depict the standard error. The last row shows the rank of each method.}\label{tab:miRNA-C-index-ablation}
\end{table}

\subsection{Explanation of Learned Weights}\label{subsec:xai}
Finally, we would like to investigate if our method has learned any meaningful relations between the different cancer types. Therefore, we compute the matrix of pair-wise Euclidean distance between each pair of cancer types $i$ and $j$ by removing the $i$-th and $j$-th entries from their learned weight vectors. After that, we perform hierarchical clustering on the computed matrix, as shown in Figure~\ref{fig:heatmap_weights_miRNA}. We notice two major groups of cancer types in the resulting clustering. Following the classification of the solid tumor types in~\cite{hoadley2018cell},
the figure shows closeness in the hierarchical clustering between cancers from the same solid tumor types. For example, for the urologic type, we find that BLCA and KIRC are clustered together, and KIRP also belongs to the same major group that contains BLCA and KIRC. We observe the same for the thoracic type (LUSC and MESO). ESCA and STAD, from the core gastrointestinal type, are within a small distance. The same applies to the types: cancers affecting melanocytes in skin and eye (UVM and SKCM) and soft tissues (SARC and UCS). We find a weaker confirmation for the gynecologic types where BRCA and CESC are in the same major cluster. The same can be observed for cancers in the developmental gastrointestinal type (LICH and PAAD).

Moreover, we found overlaps and similarities when comparing with the unsupervised clustering performed by~Hoadley et al. on the DNA methylation. For example, HNSC, CESC, and ESCA were clustered within small proximity by MSSDA and belong to the same clusters (METH2 and MET3)~\cite{hoadley2018cell}. The same observation can be made for ESCA and STAD that we find to be within a small distance and belong to the same branch of clusters



Our observations are of high importance since our system learned the relations between the cancer types by only fitting the risk functions of unlabeled targets and not directly from the data as in~\cite{hoadley2018cell}.

\subsection{Ablation Analysis}\label{subsec:ablation}
In this section, we perform an ablation study by replacing the proposed  ($SDI$) with the following domain-invariant distances and regularizers:
\textit{(i)} MDAN, a domain classifier as proposed in~\cite{fernandez2019maximum}, 
\textit{(ii)} KuiperUB-KM which tightens the upper bound of the p-value of the two-sample Kuiper test~\cite{kuiper1960tests} that is applied on the Kaplan–Meier (KM) curve~\cite{kaplan1958nonparametric},
\textit{(iii)} MMD, the maximum mean discrepancy~\cite{gretton2006kernel} (which does not take censoring into consideration),
\textit{(iv)} MMD-KM,  the maximum mean discrepancy on the KM curve, 
and \textit{(v)} the D-index. Fernandez and Gretton propose in~\cite{fernandez2019maximum} an adaptation to the maximum mean discrepancy (MMD) for data with censored cases. We couldn't compare with this distance since it is a one-sample test against the uniform distribution.

Results in Table~\ref{tab:miRNA-C-index-ablation} show the superiority and benefit of $SDI$ over the other methods in forcing the representation's conditional invariance. This is mainly because $SDI$ takes censoring into consideration (which is ignored by MDAN and MMD), aligns  the conditional distributions (which is ignored also by MDAN and MMD), and guarantees symmetry (symmetry is not guaranteed in KuiperUB-KM and D-index).

\section{Related Work}

\textbf{Multi-Source Domain Adaptation (MSDA)}
Ben-David et al.~\cite{ben2007analysis} define the distance $d_A$ between two distributions and prove a VC dimension-based generalization bound for domain adaptation in binary tasks. Mansour, Mohri, and
Rostamizadeh~\cite{mansour2009domain} generalized this bound further to a broader set of problems and used it in a tighter bound with the Rademacher complexity. Ben-David et al.~\cite{ben2010theory} introduce the $H \Delta H$ as a discrepancy measure between distributions and show how to approximate it merely from a finite sample of unlabeled target data. Cortes and Mohri~\cite{cortes2014domain} define the discrepancy measure $D_{disc}$ between distributions regardless of the true labeling function and present an algorithm for adaptation using Discrepancy minimization. Most MSDA methods employ bounds based on these seminal works; for example, domain adversarial neural network
(DANN)~\cite{ganin2016domain} performs distribution matching by a min-max game; this work was extended to the multiple domains in MDAN~\cite{zhao2017multiple}. Li et al.~\cite{li2016transfer} show a tighter bound using a Wasserstein-like distance extending the $H \Delta H$ divergence. Richard et al.~\cite{richard2020unsupervised} employ the $D_{disc}$ for regression target domains. Shaker, Yu, and Onoro-Rubio~\cite{shaker2022learningVCD} propose to align the conditional distributions in the multi-source domain adaptation setting using a symmetric form of the conditional von Neumann divergence~\cite{yu2020measuring}.
Our proposed method can be interpreted as aligning the conditional distributions in the feature space while conditioning on the rankings in the output space.

\textbf{Machine Learning for Survival Analysis }
While the non-parametric methods, such as the Kaplan-Meier (KM) estimator~\cite{Kaplan:1958:NEIO}, can be efficient for moderate data volumes, they have a major limitation in relating the survival function to the covariates. Cox proportional hazards models~\cite{Cox:1972:RMLT,Cox:1984:ASD} assume the proportionality of hazards between instances and model the risk by a log-linear function of the instance's covariates. 
A broad spectrum of machine learning methods has been adapted to deal with the challenge of censoring. Ridge-Cox~\cite{tibshirani1997lasso} and lasso-Cox~\cite{verweij1994penalized} add $l_1$ and $l_2$ regularization terms to the original Cox model, respectively.
Wang, Li, and Reddy~\cite{wang2019machine} reveal a recent survey on the intersection between survival analysis and machine learning research. Survival random forest (RSF) adopts ensemble learning to cope with censored cases~\cite{ishwaran2007random,ishwaran2008random}, and Khan and Zubek~\cite{khan2008support} introduce support vector regression for censored data (SVRC). In \cite{shaker2014survival,krempl2014open}, the authors propose the continuous and adaptive learning of parallel hazard functions in non-stationary environments under the instantaneous PH assumption, whereas, Lee et al.~\cite{lee2018deephit} deal with learning time-variant survival functions while allowing multiple events and risks per patient, thus, relaxing the PH assumption.
Knowledge-transfer between survival models has been the focus of transfer~\cite{li2016transfer} and multi-task learning~\cite{li2016multi,wang2017multi} for survival analysis. DeepSurv~\cite{katzman2018deepsurv} implement the PH assumption using a deep neural network. The work in \cite{mouli2019deep}~defines a clustering objective over survival distributions of samples by tightening the upper bound of the p-value of the two-sample Kuiper test~\cite{kuiper1960tests}. In \cite{nagpal2021deep}, individual survival distributions are fit as a mixture of Cox regression functions.
Despite these advancements in research, there is still the need for methods that perform adaptation between survival domains. This work is the first attempt to fill this gap.

\section{Conclusion}
We presented multi-source survival domain adaptation (MSSDA), which is, to the best of our knowledge, the first multi-source domain adaptation work for survival domains. Adapting to a particular target survival domain is essential for rare or new illness types. 
In survival analysis, we are faced with the additional difficulty of censored data. To not lose this partial information about survival, we define a new symmetric index for survival data that can handle censored data, show that it is a metric, and use it to bound the generalization error on target domains. This bound is explicitly employed in our method MSSDA. We confirm in experimental results that: (1) our method outperforms existing methods on target survival domains in terms of survival ranking; (2) it can offer better treatment recommendations; (3) it allows us to inspect how different domains relate, offering medical professionals additional insights. We hope our method can aid in identifying better treatments for rare or new illnesses. In the future, we hope to extend our method so that medical professionals can better understand its predictions to improve precision medicine for individuals.

\section{Acknowledgments}
We thank Brandon Malone and Anja Moesch for their feedback on the paper and the insightful discussions.




\bibliographystyle{abbrv}
\bibliography{SubmissionArxiv.bib}

\clearpage
\pagebreak

\appendix
\section{Ethics Statement and Potential Societal Impacts} 
Our work aims to introduce domain adaptation to the field of survival data.
This line of work can have a positive impact on saving human life by providing precision medicine in the form of personalized treatment recommendations and a better understanding of how diseases could be related and correlated with each other. However, our method is still in the research stage. Therefore, we do not recommend its use in a medical setting without first extensively verifying that a learned model performs as expected. Ultimately all medical decisions should remain in the hands of a medical professional, who is better qualified to judge whether an AI model's prediction should be followed or not.

In addition, domain adaptation, with knowledge transfer, helps learn from fewer data. This positively affects the environment by reducing computational power and run-time to train models, hence, less electricity consumption and less $CO_2$ emissions.

\section{Mathematical proofs} 
In this section, we provide the proves for Theorems~\ref{theorem:SDI-index} and~~\ref{theorem:SDI_bound}.

\subsection{Proof of Theorem~\ref{theorem:SDI-index}}
\begin{proof} [Proof of Theorem~\ref{theorem:SDI-index}]
We know from Eq.~(\ref{eq:SDI}) that the $SDI$ is a metric when $D$ contains no censoring instances.
We begin by creating the subsets:
\begin{itemize}
    \item $C_{ev} = \{(\boldsymbol{x}_i,t_i,\delta_i)| (\boldsymbol{x}_i,t_i,\delta_i)\in D \land \delta_i=1\}$ which contains all non-censored instances from $D$. On $C_{ev}$'s instances, define the permutations $\tau_{r_1}$,$\tau_{r_2}$ that rank the instances based on the risks estimated by the functions $r_1$ and $r_2$, respectively.
    $\kappa(\tau_{r_1},\tau_{r_2})$ measures the discrepancy between the two rankings $\tau_{r_1}$ and $\tau_{r_2}$ Eq. (\ref{eq:kindalltau1}).
    \item For each censored instance $(\boldsymbol{x},t,\delta)\in D$ and each risk function $r$, define the set $C_{r,x} = \{(\boldsymbol{x}_j,t_j,\delta_j)| (\boldsymbol{x},t,\delta)\in D \land \delta = 1 \land r(x_j)>r(x) \}$. 
\end{itemize}
From the previously defined sets and quantities, let 
\begin{align}
m =& \kappa(\tau_{r_1},\tau_{r_2})
+ \sum_{\substack{(\boldsymbol{x}_i,t_i,\delta_i) \in D\\ \delta_i = 0 }} \frac{|C_{r_1,x_i} \triangle C_{r_2,x_i} |}{|C_{r_1,x_i} \cup C_{r_2,x_i}|}
 \label{eq:D-index_measure1} \\
 =& \kappa(\tau_{r_1},\tau_{r_2})
+  \sum_{\substack{(\boldsymbol{x}_i,t_i,\delta_i) \in D\\ \delta_i = 0 }} 1 - \frac{|C_{r_1,x_i} \cap C_{r_2,x_i} |}{|C_{r_1,x_i} \cup C_{r_2,x_i}|}
 \label{eq:D-index_measure2} \enspace .
\end{align}\noindent
Notice that the second term of (\ref{eq:D-index_measure2}) is indeed the sum over the Jaccard metric distance on $C_{r_1,x_i}$ and $C_{r_2,x_i}$\footnote{This follows from the property$|A \triangle B|=|(A\cup B)\setminus (A\cap B)| =|(A\cup B)|-|(A\cup B)\cap(A\cap B)|=|(A\cup B)|-|(A\cap B)|$.}. Therefore, $m$ is a metric since it is a rescaling and a sum of the metrics: \textit{(i)} $\kappa(\tau_{r_1},\tau_{r_2})$ which is the Kendall tau metric on the rankings of the set $C_{ev}$, and \textit{(ii)} the sum of the Jaccard metrics on 
$C_{r_1,x_i}$ and $C_{r_2,x_i}$ for each censored instance
$(\boldsymbol{x}_i,t_i,\delta_i)\in D$.

Notice that $m$ is equivalent to the $SDI(r_1,r_2;D)$ as defined in Eq. (\ref{eq:SDI}), before taking the weighted average of its two parts by $\alpha_1$ and $\alpha_2$. 
Since these weights are independent of the ranking functions $\tau_1$ and $\tau_2$, they form a simple rescaling of summed metrics. 
Hence, $SDI$ defines a metric on the set $D$ with respect to the risk estimators $r_1,r_2:\mathbb{R}^d \rightarrow \mathbb{R}^+$.
\end{proof}

\subsection{Proof of Theorem~\ref{theorem:SDI_bound}}
\begin{proof} [Proof of Theorem~\ref{theorem:SDI_bound}]	
For the single source $D_{s_i}$ with true mapping function $f_{s_i}$, the following bound holds for each $h \in \mathcal{H}$:
\begin{align} 
&SDI(r_h,f_t; M_t) \leq SDI(r_h,f_{s_i}; M_{s_i}) + 
 \left|SDI(r_h,f_t; M_t) - SDI(r_h,f_{s_i}; M_{s_i})  \right| \label{eq:singlesource1} \\
& \leq SDI(r_h,f_{s_i}; M_{s_i})+  \color{black}{\left| SDI(r_h,r_{h^*}; M_t)-  SDI(r_h,f_t; M_t) \right|} \nonumber \\
& + \color{black}{\left|SDI(r_h,r_{h^*}; M_{s_i})-SDI(r_h,f_{s_i}; M_{s_i})\right|} 
 + \color{black}{\left|SDI(r_h,r_{h^*}; M_{s_i})-SDI(r_h,r_{h^*}; M_t)\right|} \label{eq:singlesource2}  \\
&\leq SDI(r_h,f_{s_i}; M_{s_i})+  \color{black}{SDI(r_{h^*},f_t; M_t)} \color{black}{+} 
\color{black}{SDI(r_{h^*},f_{s_i}; M_{s_i})} + \color{black}{D_{SDI-disc}(P_{s_i},P_t)} \enspace . \label{eq:singlesource3} 
\end{align}
Inequality~(\ref{eq:singlesource1}) holds since $SDI$ is non-negative. Inequality~(\ref{eq:singlesource3}) follows from the triangular inequality of $SDI$ i.e., 
\textcolor{black}{$\left| SDI(r_h,r_{h^*}; M_t)-  SDI(r_h,f_t; M_t) \right|\leq   SDI(r_{h^*},f_t; M_t) \enspace $} and\\ 
\textcolor{black}{$\left|SDI(r_h,r_{h^*}; M_{s_i})-SDI(r_h,f_{s_i}; M_{s_i}) \right| \leq SDI(r_{h^*},f_{s_i}; M_{s_i}),$}\\
where $h^*$ is a hypothesis satisfying:
\begin{align} 
h^* \in \arg_{h \in \mathcal{H}}\min
SDI(r_{h},f_t; M_t)+ \sum_{i=1}^k w_i \cdot SDI(r_{h},f_{s_i}; M_{s_i}) \enspace .
\end{align}
Note also that $\textcolor{black}{\left|SDI(r_h,r_{h^*}; M_{s_i})-SDI(r_h,r_{h^*}; M_t)\right|} \leq D_{SDI-disc}(P_{s_i},P_t)$ by the Definition \ref{def:discordance_hypothesis_divergence}.

For the multiple domains in $S$, we multiply both sides of (\ref{eq:singlesource3}) by $w_i$ for each source $s_i \in S$ to obtain the $k$ inequalities: 
{\fontsize{9.5pt}{10.8pt} \selectfont  \begin{align} \label{eq:multisource} 
&w_i \cdot SDI(r_h,f_t; M_t)   \leq w_i \cdot \bigg(
SDI(r_h,f_{s_i}; M_{s_i})+ SDI(r_{h^*},f_t; M_t) 
+SDI(r_{h^*},f_{s_i}; M_{s_i}) + D_{SDI-disc}(P_{s_i},P_t)
\bigg) \enspace .
\end{align}}
By summing (\ref{eq:multisource}) over all $s_i \in S$, we obtain
\begin{align} \label{eq:multisource2} 
&SDI(r_h,f_t; M_t) \leq \sum_{i=1}^k w_i \cdot \bigg(SDI(r_h,f_{s_i}; M_{s_i}) + 
D_{SDI-disc}(P_{s_i},P_t) \bigg)+ \eta_{D}(f_S,f_t) \enspace ,
\end{align}\noindent
where $\eta_{D}(f_{S},f_t)= \min_{h^{*} \in \mathcal{H}}
SDI(r_{h^*},f_t; M_t)+ \sum_{i=1}^k w_i \cdot SDI(r_{h^*},f_{s_i}; M_{s_i})$ is the minimum joint empirical losses on source $D_{s_i}$ and the target $D_t$, achieved by an optimal hypothesis $h^{*}$. Equation (\ref{eq:multisource2}) concludes the proof.
\end{proof}

\section{Hyperparameter Search}
Since our problem is survival domain adaption where the target domains should remain unlabeled, we try to search for the best hyper-parameters that should perform well on the the source domains. This in turn implicitly leads to a better generalization on the target domains as proven in Theorem~\ref{theorem:SDI_bound}.

To ensure a fair comparison, we start with a grid-based hyperparameter search by selecting five random domains from miRNA, and three from mRNA. We perform the gird search on the labeled domains while considering only $1/3$ of the data. The performance is measured in terms of the $\text{C-index}$.\\ 	

\begin{itemize}
	\item DeepSurv~\cite{katzman2018deepsurv} found hyperparameters:
	\begin{itemize}
	\item learning rate: $  \{$(miRNA, mRNA) 0.0001, 0.001, 0.01,0.1 $\}$ 
	\item dropout rate: $\{$0, (miRNA) 0.05, (mRNA) 0.07, 0.1, 0.2$\}$
	\item library: \url{https://github.com/jaredleekatzman/DeepSurv}
	\end{itemize}
	\item Random Survival Forest~\cite{ishwaran2007random,ishwaran2008random} found hyperparameters:
	\begin{itemize}
	\item max\_depth:  $\{$(mRNA) 5, (miRNA) 7, 10, 12 $\}$ 
	\item min\_node\_size:  $\{$5, 7, 10, (mRNA) 12, (miRNA) 20 $\}$ 
	\item library: \url{https://square.github.io/pysurvival}
	\end{itemize}
	
	\item TransferCox~\cite{li2016transfer} found hyperparameters:
	\begin{itemize}
	\item weight of target dataset: $\{$0.5, 1 (mRNA), 2 (miRNA), 5 $\}$ 
	\item number of $\lambda$ to search for: 10
	\item smallest searching $\lambda$: 0.06
	\item library: \url{https://github.com/MLSurvival/TransferCox}
	\end{itemize}
    \item Coy~\cite{Cox:1972:RMLT}. Here, we apply PCA and select the first 15 components and decrease this number in case of colinarities.
\end{itemize}
	
Without any further tuning, we adopt the same found parameters for DeepSurv to our proposed method, MSSDA.

\section{Further Results} 

\subsection{$\text{C-index}$ versus $\text{C-index}$${^\prime}$}
To show the advantage $\text{C-index}$${^\prime}$ over $\text{C-index}$ on the target domain, let us assume a survival domain containing $n+c$ instances where $n$ and $c$ are the numbers of the non-censored and censored instances, respectively. Assume also that $\lambda_s \in ]0,1[$ is the percentage of instances used for the supervision.

The $\text{C-index}$${^\prime}$ inspects $n^2+\frac{1}{2}n\cdot c$ pairs\footnote{Under the assumption that censored cases are uniformly distributed among the non-censored ones.}, whereas the $\text{C-index}$ would inspect $(1-\lambda_s)^2(n^2+\frac{1}{2}n\cdot c)$. Hence, for $\lambda_s= 0.25$, only 6.25\% would be reused in $\text{C-index}$${^\prime}$. More importantly, the $\text{C-index}$ misses $\lambda_s(1-\lambda_s)(n^2 + n\cdot c)$ of all pairs; each of these pairs includes one sample from the supervision set and one that is not used in the supervision set. These pairs do not include any leaked information since the target here is the ranking of samples (pair-wise comparison) and not the absolute risks.

\subsection{Data Sets}
In this section, we describe the used data set and partially repeat what is explained in the main manuscript. In our experiments we use two data sets from The Cancer Genome Atlas (TCGA) project: \textit{(i)} The Messenger RNA data (mRNA)~\cite{li2016transfer}, which includes eight cancer types. Each patient is represented by 19171 binary features; see Table~\ref{tab:mRNA}.
\textit{(ii)} The micro-RNA data (miRNA) that includes 21 cancer types~\cite{wang2017multi}; each has a varying number of patients. Table~\ref{tab:miRNA} depicts the total number of patients for each cancer and the number of patients that experienced the event (died) during the time of the clinical study ($\delta=1$).
We also extract the treatment performed for each cancer type (if available). Table~\ref{tab:Supp_miRNA} includes the full name and the site of each cancer type.

\subsection{Evaluation of Survival Prediction}
In the following, we present a list of Figures and Tables that depict the performance and rank of the different methods on the mRNA and miRNA data.
\begin{enumerate}
    \item Figure~\ref{fig:all_mRNA_C-index_rem_strict} depicts the performance comparison, on the mRNA data, in terms of the $\text{C-index}$. This figure summarizes Tables~\ref{tab:Supp-mRNA-C-index-rem-1-RSF} and~\ref{tab:Supp-mRNA-C-index-rem-2-RSF}.
    \item Figure~\ref{fig:all_mRNA_C-index_strict} depicts the performance comparison, on the mRNA data, in terms of the $\text{C-index}^\prime$.
    This figure summarizes Tables~\ref{tab:Supp-mRNA-C-index-1-RSF} and~\ref{tab:Supp-mRNA-C-index-2-RSF}.
    \item Figure~\ref{fig:all_miRNA_C-index_rem_strict} depicts the performance comparison, on the miRNA data, in terms of the $\text{C-index}$.
    This figure summarizes Tables~\ref{tab:Supp-miRNA-C-index-rem-1-RSF} and~\ref{tab:Supp-miRNA-C-index-rem-2-RSF}.
    \item Figure~\ref{fig:all_miRNA_C-index_strict} depicts performance comparison, on the miRNA data, in terms of the $\text{C-index}^\prime$.
    This figure summarizes Tables~\ref{tab:Supp-miRNA-C-index-1-RSF} and~\ref{tab:Supp-miRNA-C-index-2-RSF}.
    \item Figure~\ref{fig:Supp_miRNA_mRNA_Rank_C-index_and_rem_strict_C-index} shows the performance comparison in terms of ranking, on the miRNA and mRNA data for both measures $\text{C-index}^\prime$ and $\text{C-index}$. This figure shows an aggregation of results in Tables~\ref{tab:Supp-mRNA-C-index-rem-1-RSF}, \ref{tab:Supp-mRNA-C-index-rem-2-RSF}, \ref{tab:Supp-mRNA-C-index-1-RSF}, \ref{tab:Supp-mRNA-C-index-2-RSF}, \ref{tab:Supp-miRNA-C-index-rem-1-RSF}, \ref{tab:Supp-miRNA-C-index-rem-2-RSF}, \ref{tab:Supp-miRNA-C-index-1-RSF}, \ref{tab:Supp-miRNA-C-index-2-RSF}.
    
    \item Table~\ref{tab:Supp-mRNA-C-index-rem-1-RSF} presents performance on the mRNA data in terms of $\text{C-index}$ for the no supervision, 5\%, and 10\% supervision settings. 
    \item Table~\ref{tab:Supp-mRNA-C-index-rem-2-RSF} presents the performance on the mRNA data in terms of $\text{C-index}$ for 15\%, 20\%, and 25\% supervision settings.
    \item Table~\ref{tab:Supp-mRNA-C-index-1-RSF} presents performance on the mRNA data in terms of $\text{C-index}^\prime$ for the no supervision, 5\%, and 10\% supervision settings. 
    \item Table~\ref{tab:Supp-mRNA-C-index-2-RSF} presents the performance on the mRNA data in terms of $\text{C-index}^\prime$ for 15\%, 20\%, and 25\% supervision settings.

    \item Table~\ref{tab:Supp-miRNA-C-index-rem-1-RSF} presents performance on the miRNA data in terms of $\text{C-index}$ for the no supervision, 5\%, and 10\% supervision settings. 
    \item Table~\ref{tab:Supp-miRNA-C-index-rem-2-RSF} presents the performance on the miRNA data in terms of $\text{C-index}$ for 15\%, 20\%, and 25\% supervision settings.
    \item Table~\ref{tab:Supp-miRNA-C-index-1-RSF} presents performance on the miRNA data in terms of $\text{C-index}^\prime$ for the no supervision, 5\%, and 10\% supervision settings. 
    \item Table~\ref{tab:Supp-miRNA-C-index-2-RSF} presents the performance on the miRNA data in terms of $\text{C-index}^\prime$ for 15\%, 20\%, and 25\% supervision settings.
\end{enumerate}

Besides the discussion mentioned in the Section~\ref{subsec:eval}, we would like to raise awareness of the fact that all methods often manage to exploit the available labeled target data and improve in terms of $\text{C-index}$${^\prime}$, see Tables~\ref{tab:Supp-mRNA-C-index-1-RSF}, \ref{tab:Supp-mRNA-C-index-2-RSF}, \ref{tab:Supp-miRNA-C-index-1-RSF}, and ~\ref{tab:Supp-miRNA-C-index-2-RSF}. While the same is expected when employing the $\text{C-index}$, this is not the case. The reason for that is the different test sets used for each supervision setting when evaluating the $\text{C-index}$ due to the removal of the samples used in the supervision.

\begin{table*}
\footnotesize
      \centering
    \begin{tabular}{|c ccccc|cc|cc|c|}
        \toprule
Task ID&Cancer name&Primary Site& Acronym&\multicolumn{2}{c}{Instances}&\multicolumn{2}{|c|}{Pharmaceutical}&\multicolumn{2}{c|}{Radiation}&TR\\
&&&&count&$\delta=1$&count&$\delta=1$&count&$\delta=1$&\\

        \midrule
1&Adrenocortical Carcinom&Adrenal Gland&ACC&80&29&34&13&2&0&\\
2&Bladder Urothelial Carcinoma&Bladder&BLCA&407&178&111&40&25&15&X\\
3&Breast Invasive Carcinoma&Breast&BRCA&754&105&238&15&31&2&X\\
4&Cervical Squamous Cell Carcinoma&Cervix&CESC&307&72&4&2&35&9&\\
&and Endocervical Adenocarcinoma&&&&&&&&&\\
5&Cholangiocarcinoma&Bile Duct&CHOL&36&18&13&7&1&1&\\
6&Esophageal Carcinoma&Esophagus&ESCA&184&77&12&5&22&5&X\\
7&Head and Neck Squamous Cell Carcinoma&Head and Neck&HNSC&484&203&6&2&134&46&\\
8&Kidney Renal Clear Cell Carcinoma&Kidney&KIRC&254&76&24&16&7&3&\\
9&Kidney Renal Papillary Cell Carcinoma&Kidney&KIRP&290&44&18&13&4&4&\\
10&Brain Lower Grade Glioma&Brain&LGG&510&124&50&10&70&16&X\\
11&Liver Hepatocellular Carcinoma&Liver&LIHC&371&128&39&18&12&4&X\\
12&Lung Adenocarcinoma&Lung&LUAD&441&157&103&36&34&24&X\\
13&Lung Squamous Cell Carcinoma&Lung&LUSC&338&137&72&20&19&13&X\\
14&Mesothelioma&Pleura&MESO&86&73&29&26&2&1&\\
15&Pancreatic Adenocarcinoma&Pancreas&PAAD&178&93&75&41&-&-&\\
16&Sarcoma&Soft Tissue&SARC&259&98&46&22&48&15&X\\
17&Skin Cutaneous Melanoma&Skin&SKCM&97&26&22&9&1&0&\\
18&Stomach Adenocarcinoma&Stomach&STAD&382&147&106&42&1&0&\\
19&Uterine Corpus Endometrial Carcinoma&Uterus&UCEC&410&72&55&18&82&10&X\\
20&Uterine Carcinosarcoma&Uterus&UCS&56&34&15&12&5&5&\\
21&Uveal Melanoma&Eye&UVM&80&23&11&6&4&2&\\
        \hline
    \end{tabular}
        \caption{Properties of the miRNA data. The treatment columns (Pharmaceutical and Radiation) are collected by matching the data with The Cancer Genome Atlas (TCGA). The TR column indicates whether or not the cancer type is used for the treatment recommendation.}\label{tab:Supp_miRNA}
\end{table*}

\begin{figure*}
     \centering
     \begin{subfigure}[b]{.49\textwidth}
         \centering
         \includegraphics[width=.98\textwidth]{figures/mRNANo_Supervision_RSF_rem.pdf}
         \caption{$\text{C-index}$ on mRNA with no supervision.}
         \label{fig:all_mRNANo_Supervision_RSF_rem}
     \end{subfigure}
     \begin{subfigure}[b]{.49\textwidth}
         \centering
         \includegraphics[width=.98\textwidth]{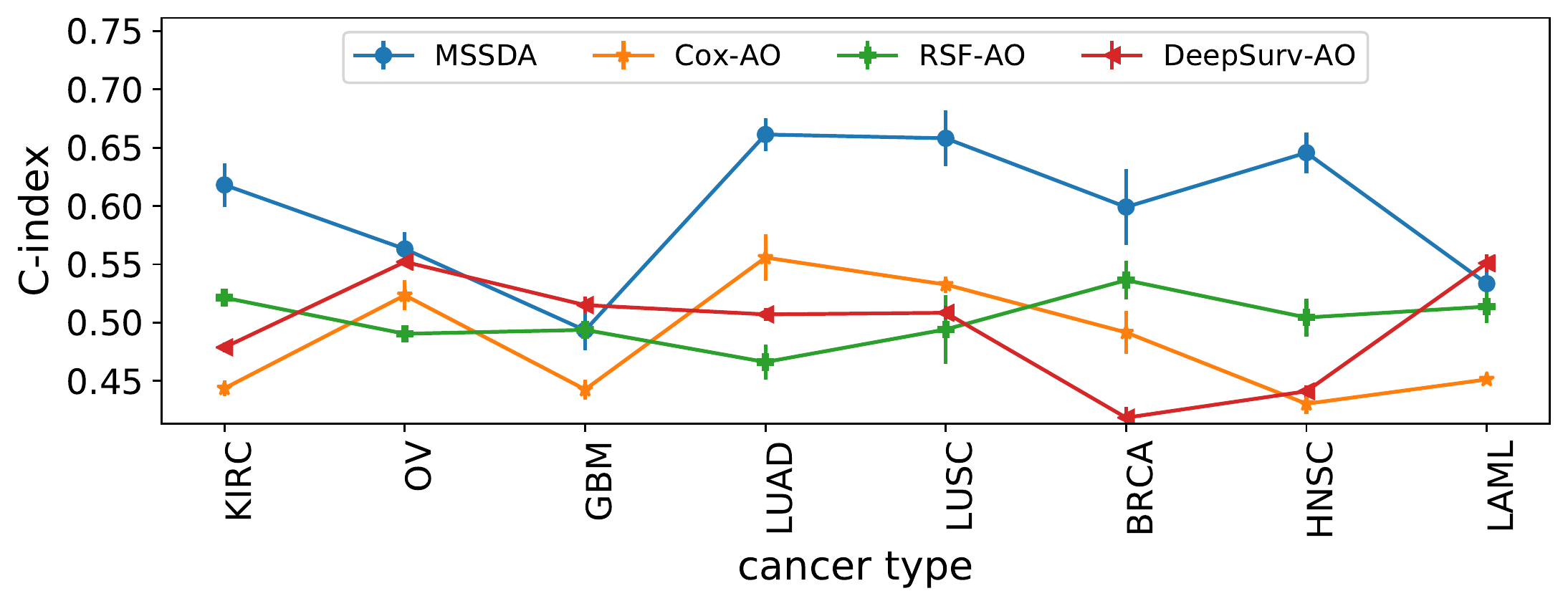}
         \caption{$\text{C-index}$ on mRNA with 5\% supervision.}
         \label{fig:all_mRNAPartial_Supervision_0.05_RSF_rem}
     \end{subfigure}

     \begin{subfigure}[b]{0.49\textwidth}
         \centering
         \includegraphics[width=.98\linewidth]{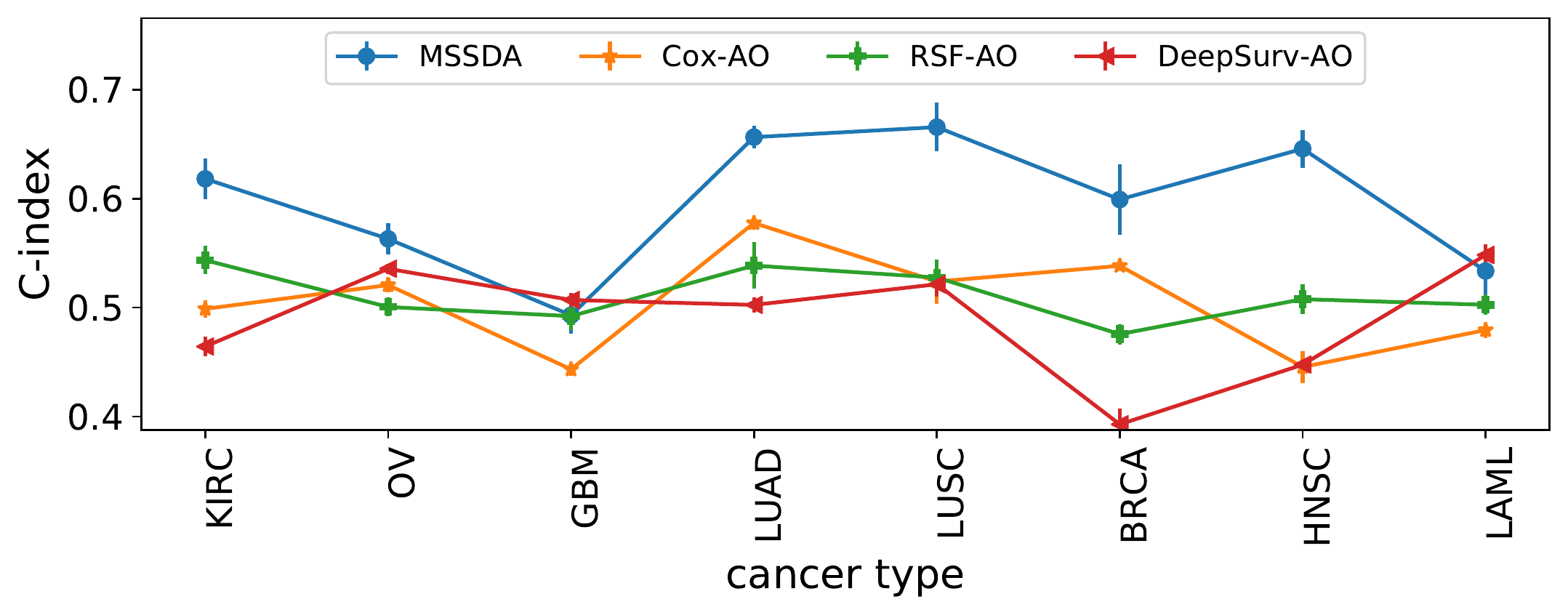}
         \caption{$\text{C-index}$ on mRNA with 10\% supervision.}
         \label{fig:all_mRNAPartial_Supervision_0.1_RSF_rem}
     \end{subfigure}
     \begin{subfigure}[b]{0.49\textwidth}
         \centering
         \includegraphics[width=.98\linewidth]{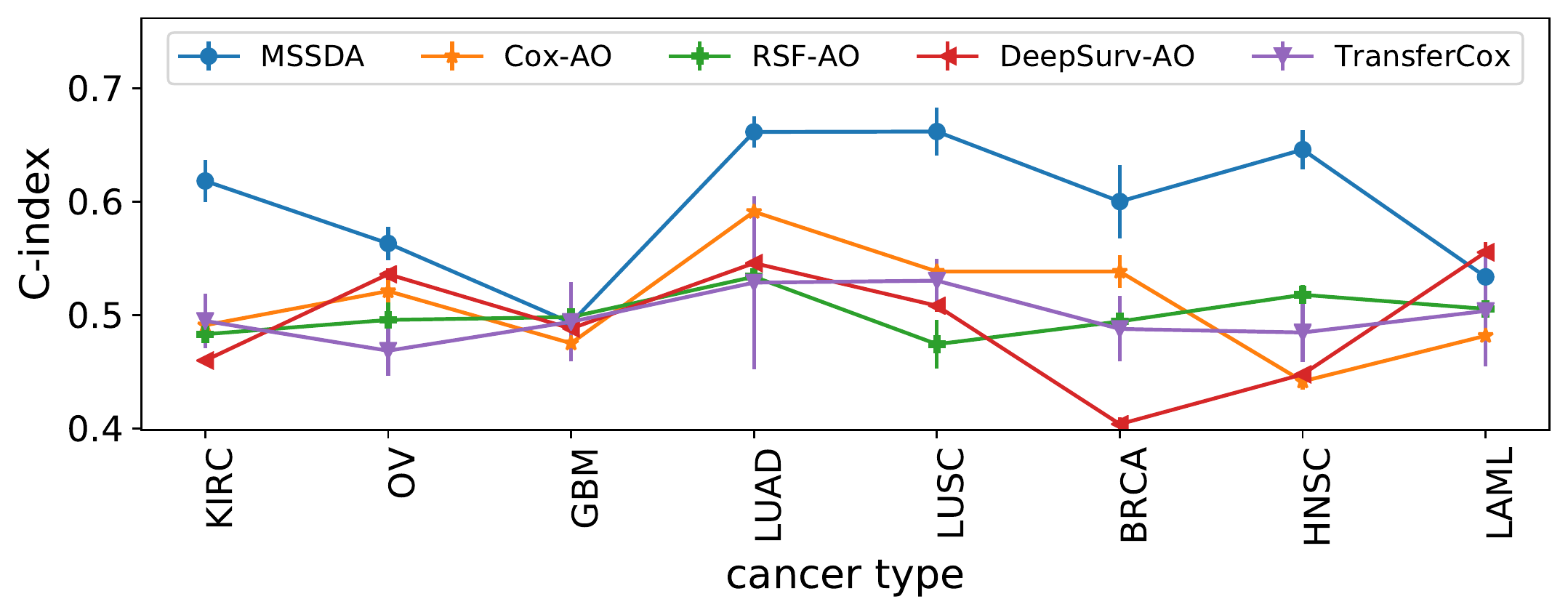}
         \caption{$\text{C-index}$ on mRNA with 15\% supervision.}
         \label{fig:all_mRNAPartial_Supervision_0.15_RSF_rem}
     \end{subfigure}

     \begin{subfigure}[b]{0.49\textwidth}
         \centering
         \includegraphics[width=.98\linewidth]{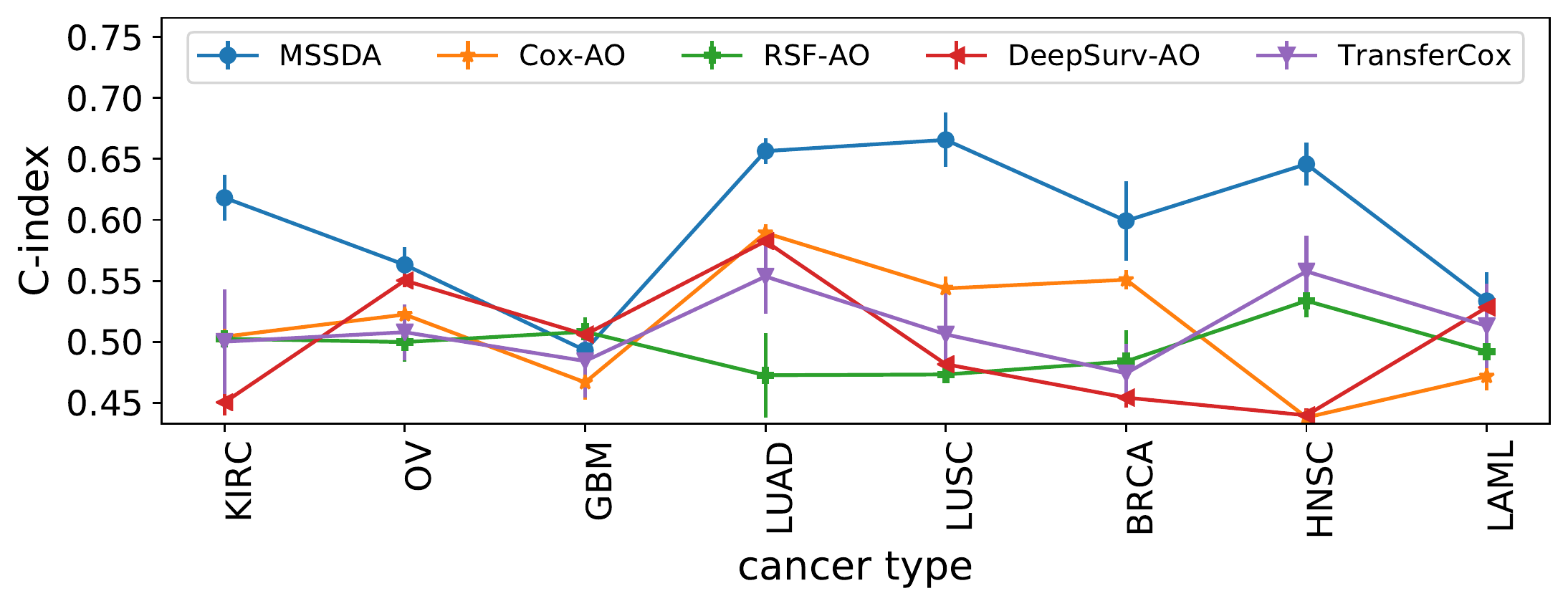}
         \caption{$\text{C-index}$ on mRNA with 20\% supervision.}
         \label{fig:all_mRNAPartial_Supervision_0.2_RSF_rem}
     \end{subfigure}
     \begin{subfigure}[b]{0.49\textwidth}
         \centering
         \includegraphics[width=.98\linewidth]{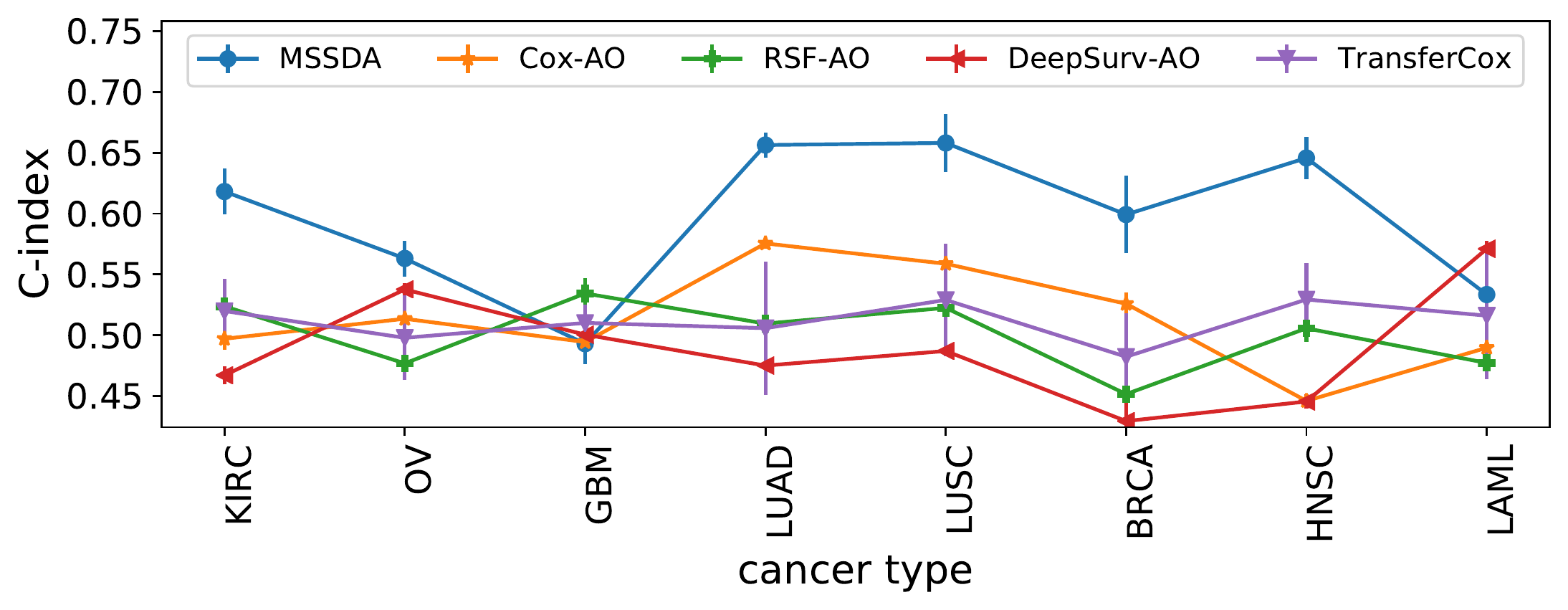}
         \caption{$\text{C-index}$ on mRNA with 25\% supervision.}
         \label{fig:all_mRNAPartial_Supervision_0.25_RSF_rem}
     \end{subfigure}     
        \caption{Performance comparison, on the mRNA data, in terms of the $\text{C-index}$.}
        \label{fig:all_mRNA_C-index_rem_strict}

\end{figure*}

\begin{figure*}
     \centering
     \begin{subfigure}[b]{.49\textwidth}
         \centering
         \includegraphics[width=.98\textwidth]{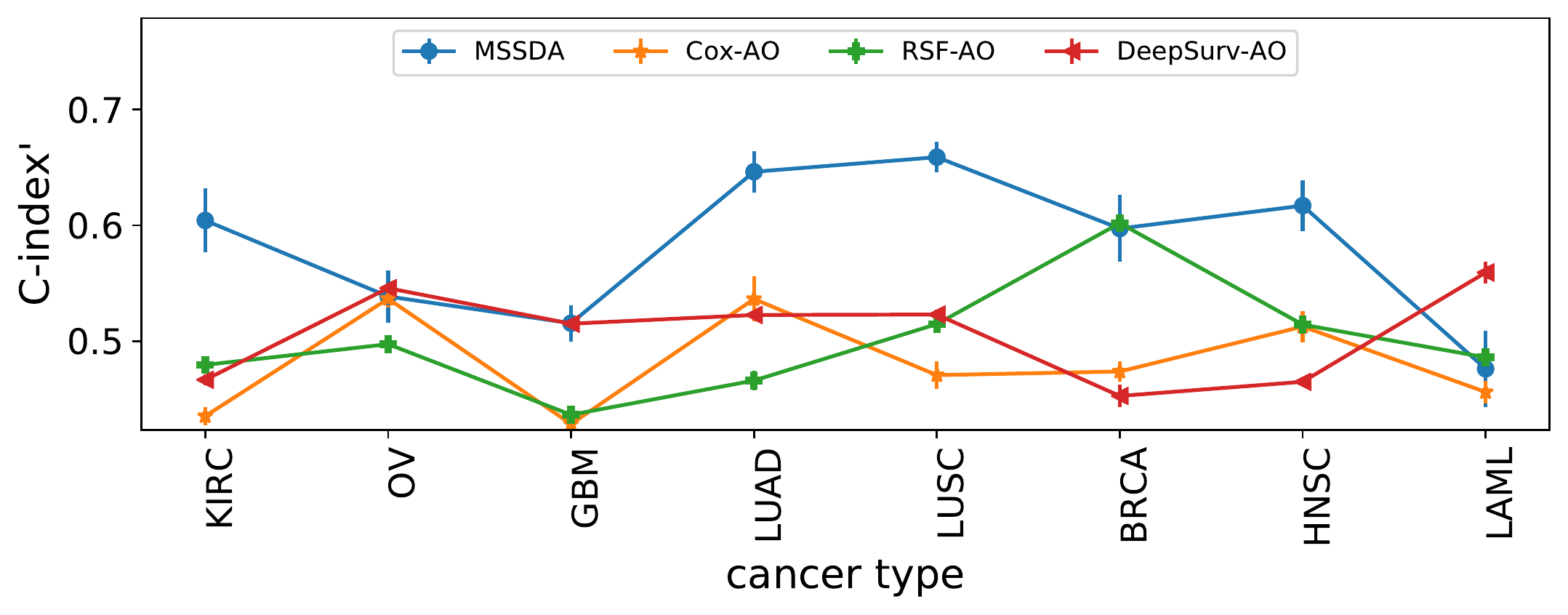}
         \caption{$\text{C-index}^\prime$ on mRNA with no supervision.}
         \label{fig:all_mRNANo_Supervision_RSF}
     \end{subfigure}
     \begin{subfigure}[b]{.49\textwidth}
         \centering
         \includegraphics[width=.98\textwidth]{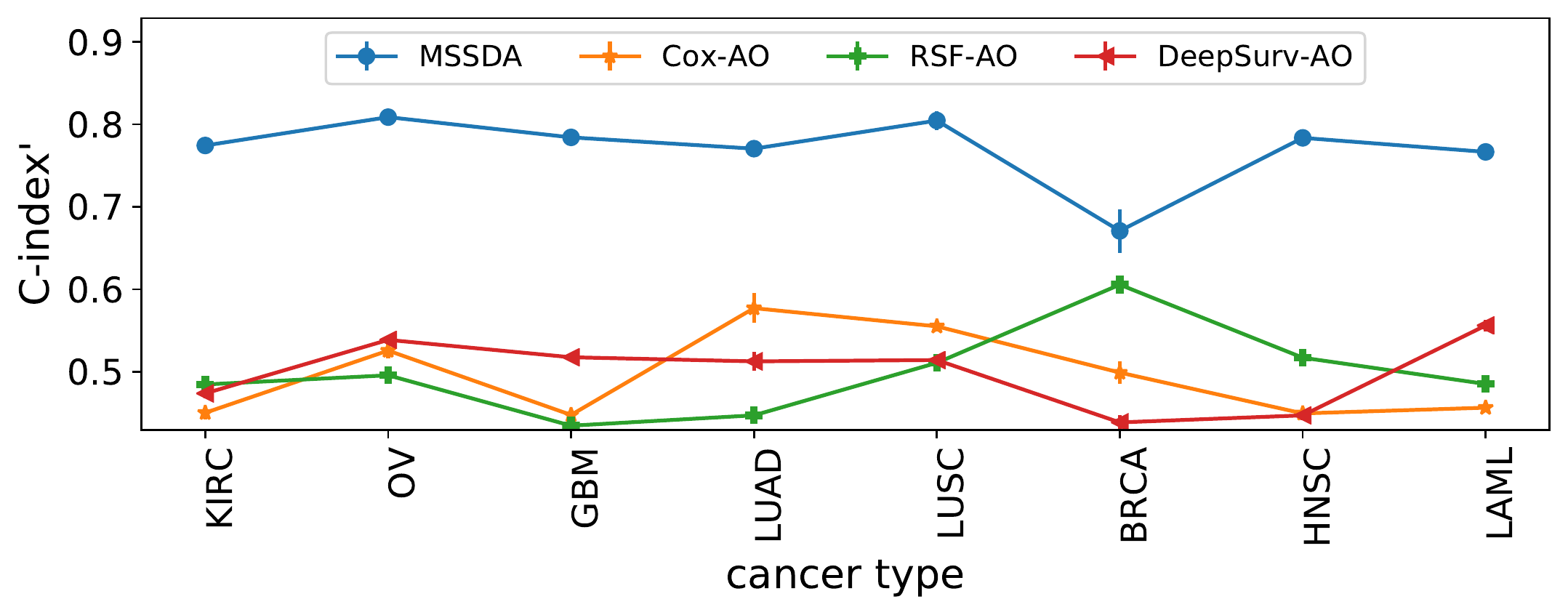}
         \caption{$\text{C-index}^\prime$ on mRNA with 5\% supervision.}
         \label{fig:all_mRNAPartial_Supervision_0.05_RSF}
     \end{subfigure}

     \begin{subfigure}[b]{0.49\textwidth}
         \centering
         \includegraphics[width=.98\linewidth]{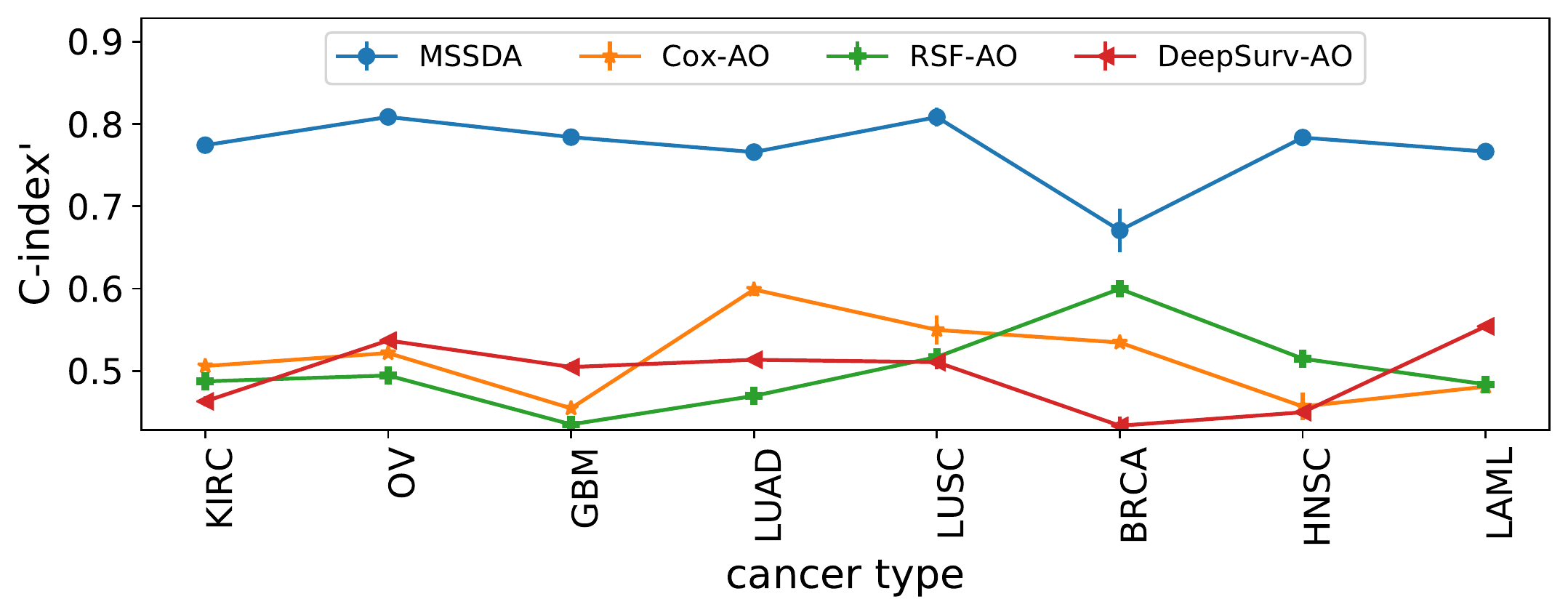}
         \caption{$\text{C-index}^\prime$ on mRNA with 10\% supervision.}
         \label{fig:all_mRNAPartial_Supervision_0.1_RSF}
     \end{subfigure}
     \begin{subfigure}[b]{0.49\textwidth}
         \centering
         \includegraphics[width=.98\linewidth]{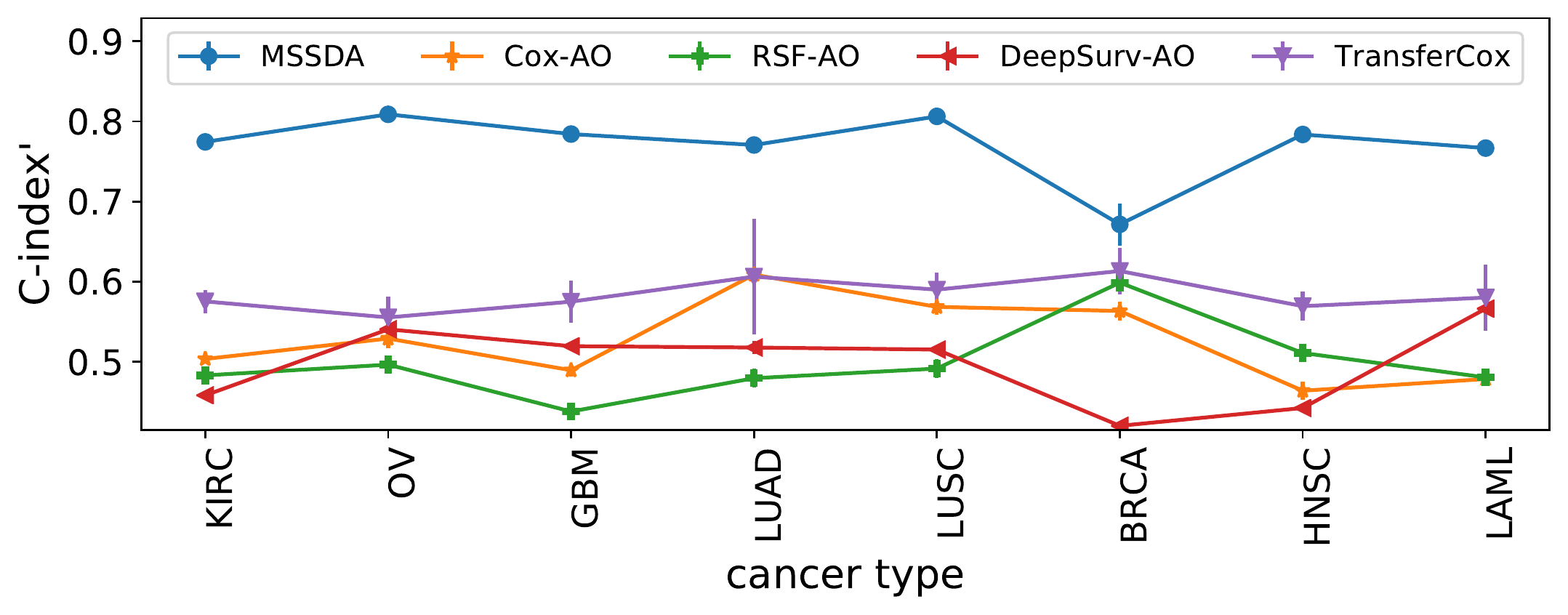}
         \caption{$\text{C-index}^\prime$ on mRNA with 15\% supervision.}
         \label{fig:all_mRNAPartial_Supervision_0.15_RSF}
     \end{subfigure}

     \begin{subfigure}[b]{0.49\textwidth}
         \centering
         \includegraphics[width=.98\linewidth]{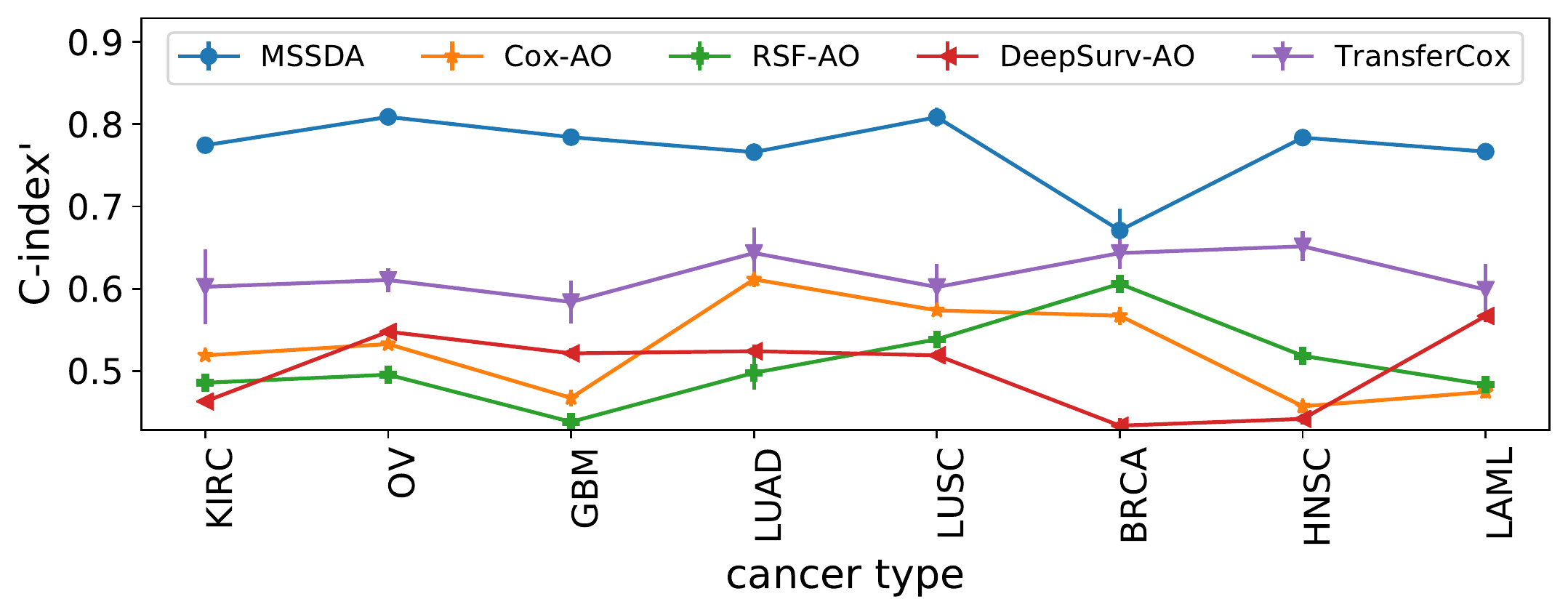}
         \caption{$\text{C-index}^\prime$ on mRNA with 20\% supervision.}
         \label{fig:all_mRNAPartial_Supervision_0.2_RSF}
     \end{subfigure}
     \begin{subfigure}[b]{0.49\textwidth}
         \centering
         \includegraphics[width=.98\linewidth]{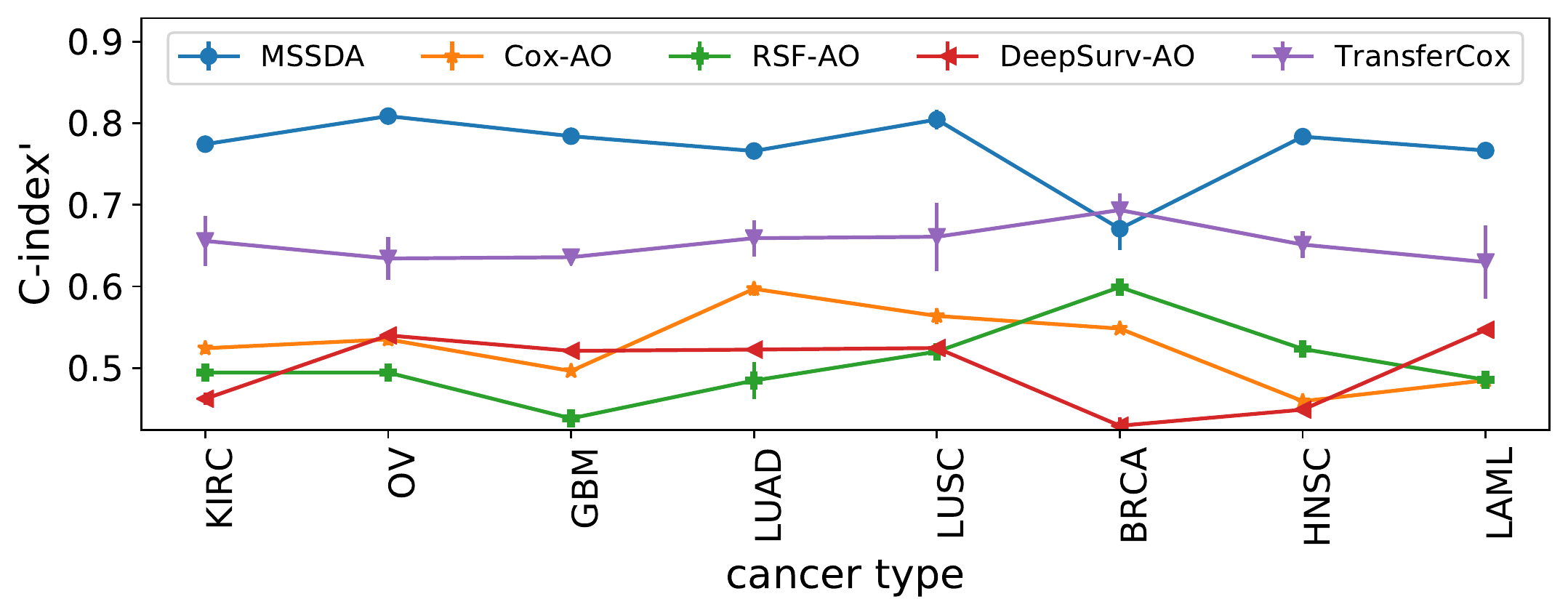}
         \caption{$\text{C-index}^\prime$ on mRNA with 25\% supervision.}
         \label{fig:all_mRNAPartial_Supervision_0.25_RSF}
     \end{subfigure}     
        \caption{Performance comparison, on the mRNA data, in terms of the $\text{C-index}^\prime$.}
        \label{fig:all_mRNA_C-index_strict}

\end{figure*}
\begin{figure*}
     \centering
     \begin{subfigure}[b]{.49\textwidth}
         \centering
         \includegraphics[width=.98\textwidth]{figures/miRNANo_Supervision_RSF_rem.pdf}
         \caption{$\text{C-index}$ on miRNA with no supervision.}
         \label{fig:all_miRNANo_Supervision_RSF_rem}
     \end{subfigure}
     \begin{subfigure}[b]{.49\textwidth}
         \centering
         \includegraphics[width=.98\textwidth]{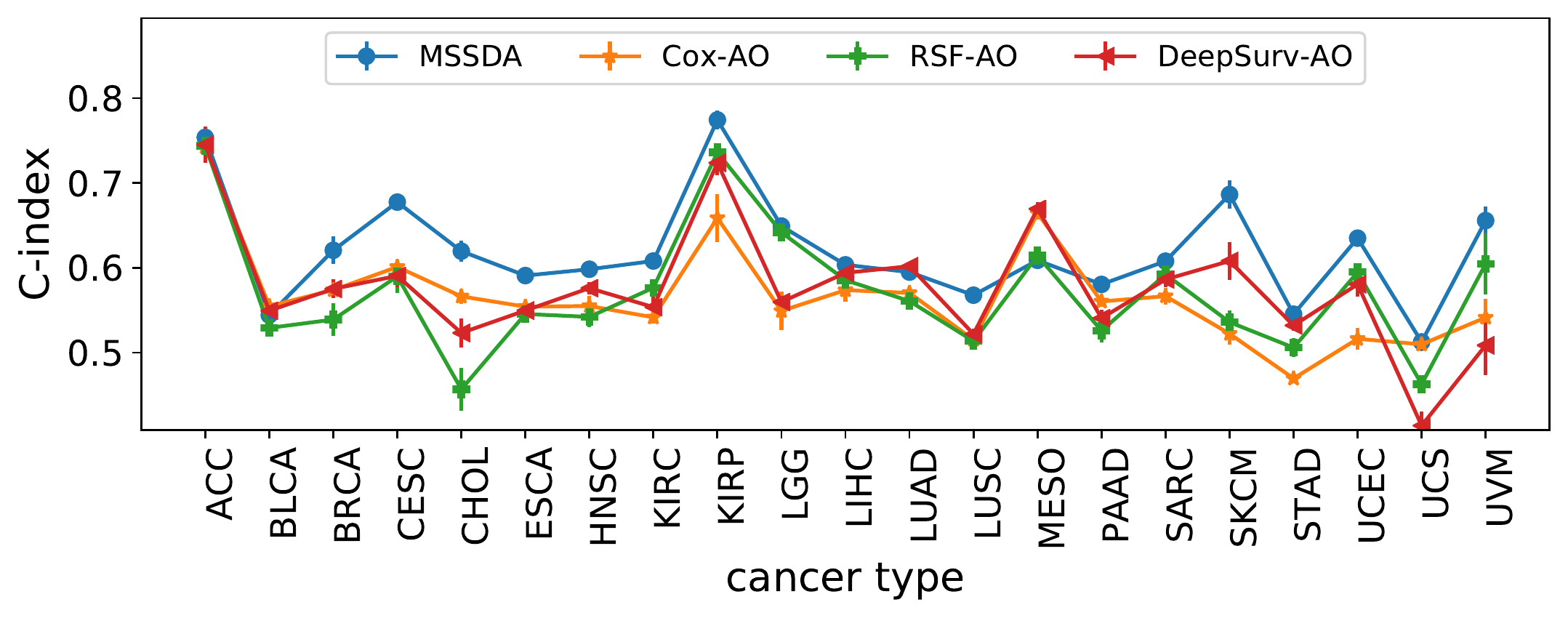}
         \caption{$\text{C-index}$ on miRNA with 5\% supervision.}
         \label{fig:all_miRNAPartial_Supervision_0.05_RSF_rem}
     \end{subfigure}

     \begin{subfigure}[b]{0.49\textwidth}
         \centering
         \includegraphics[width=.98\linewidth]{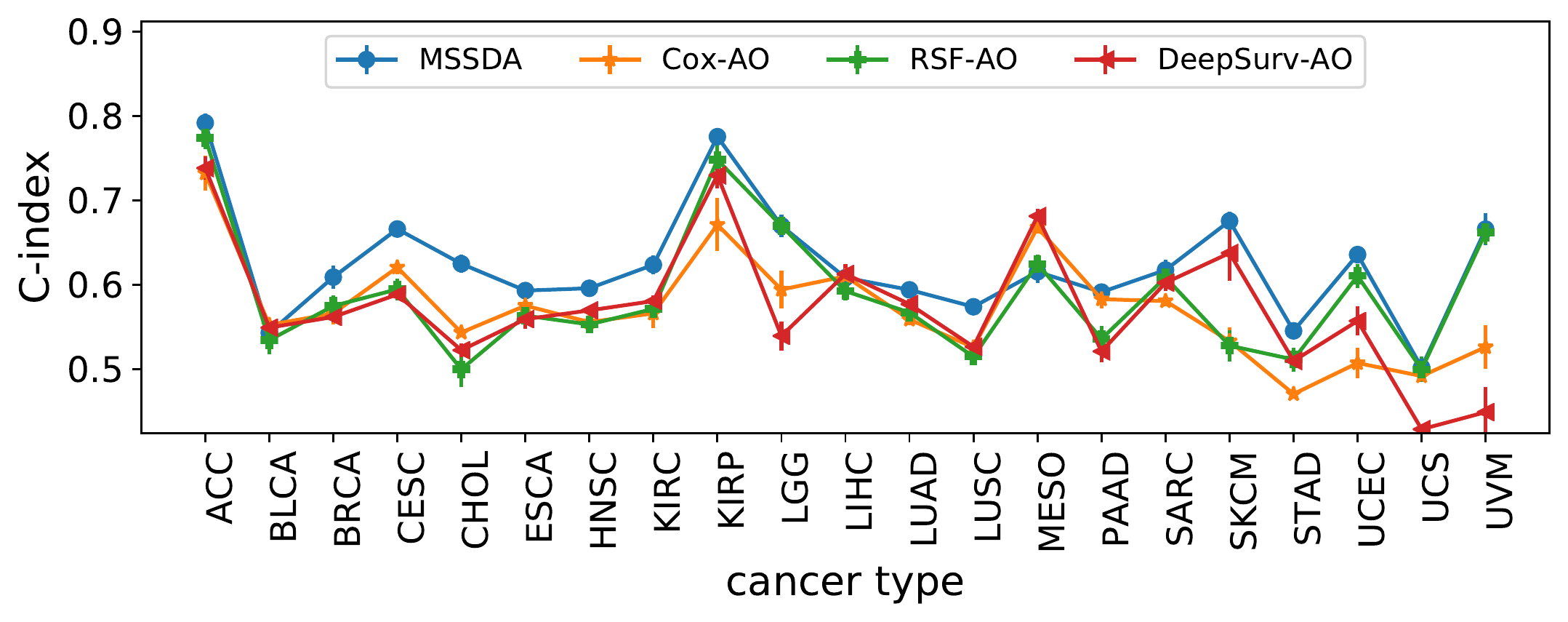}
         \caption{$\text{C-index}$ on miRNA with 10\% supervision.}
         \label{fig:all_miRNAPartial_Supervision_0.1_RSF_rem}
     \end{subfigure}
     \begin{subfigure}[b]{0.49\textwidth}
         \centering
         \includegraphics[width=.98\linewidth]{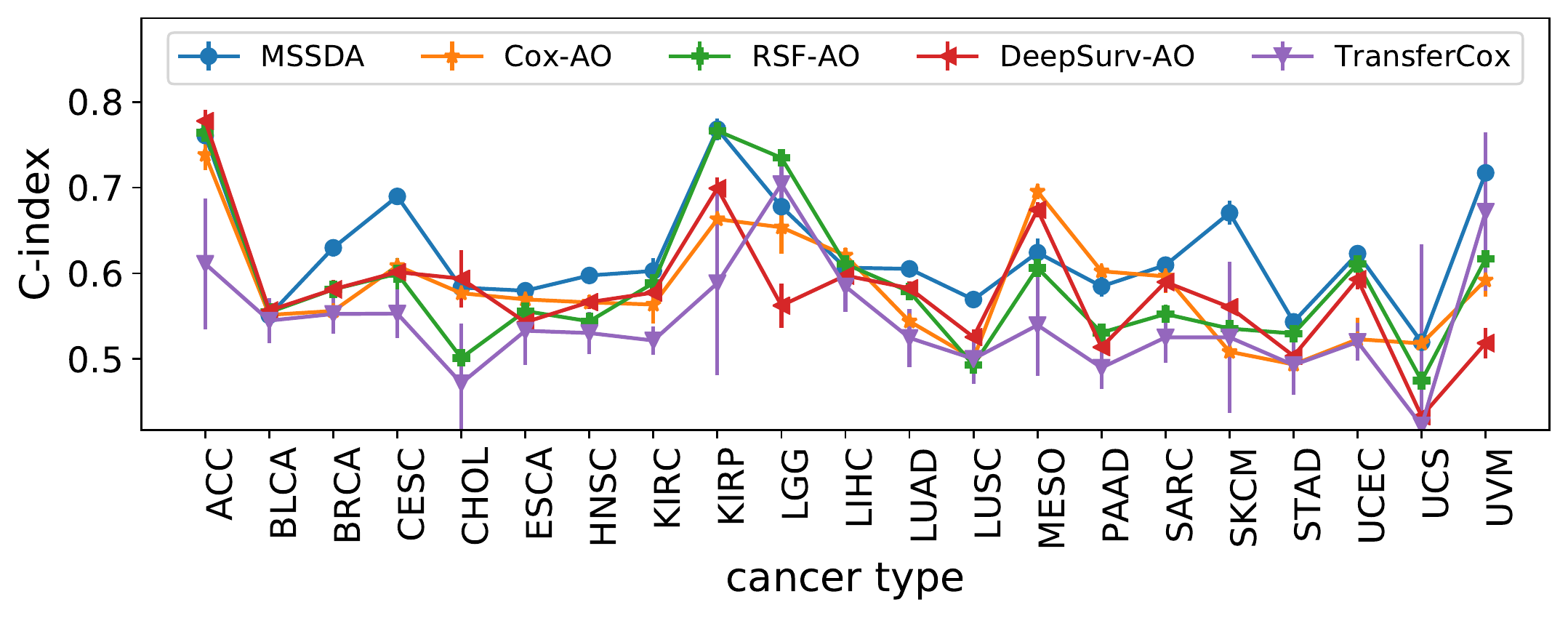}
         \caption{$\text{C-index}$ on miRNA with 15\% supervision.}
         \label{fig:all_miRNAPartial_Supervision_0.15_RSF_rem}
     \end{subfigure}

     \begin{subfigure}[b]{0.49\textwidth}
         \centering
         \includegraphics[width=.98\linewidth]{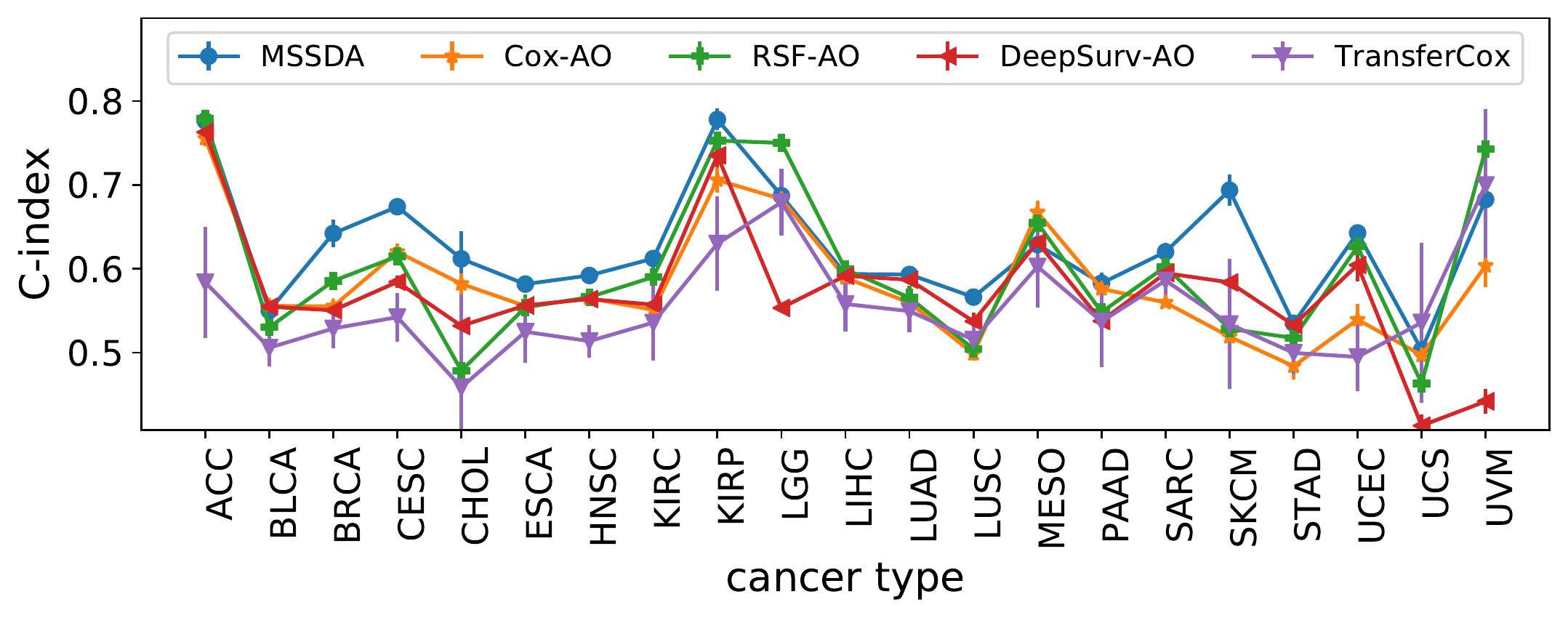}
         \caption{$\text{C-index}$ on miRNA with 20\% supervision.}
         \label{fig:all_miRNAPartial_Supervision_0.2_RSF_rem}
     \end{subfigure}
     \begin{subfigure}[b]{0.49\textwidth}
         \centering
         \includegraphics[width=.98\linewidth]{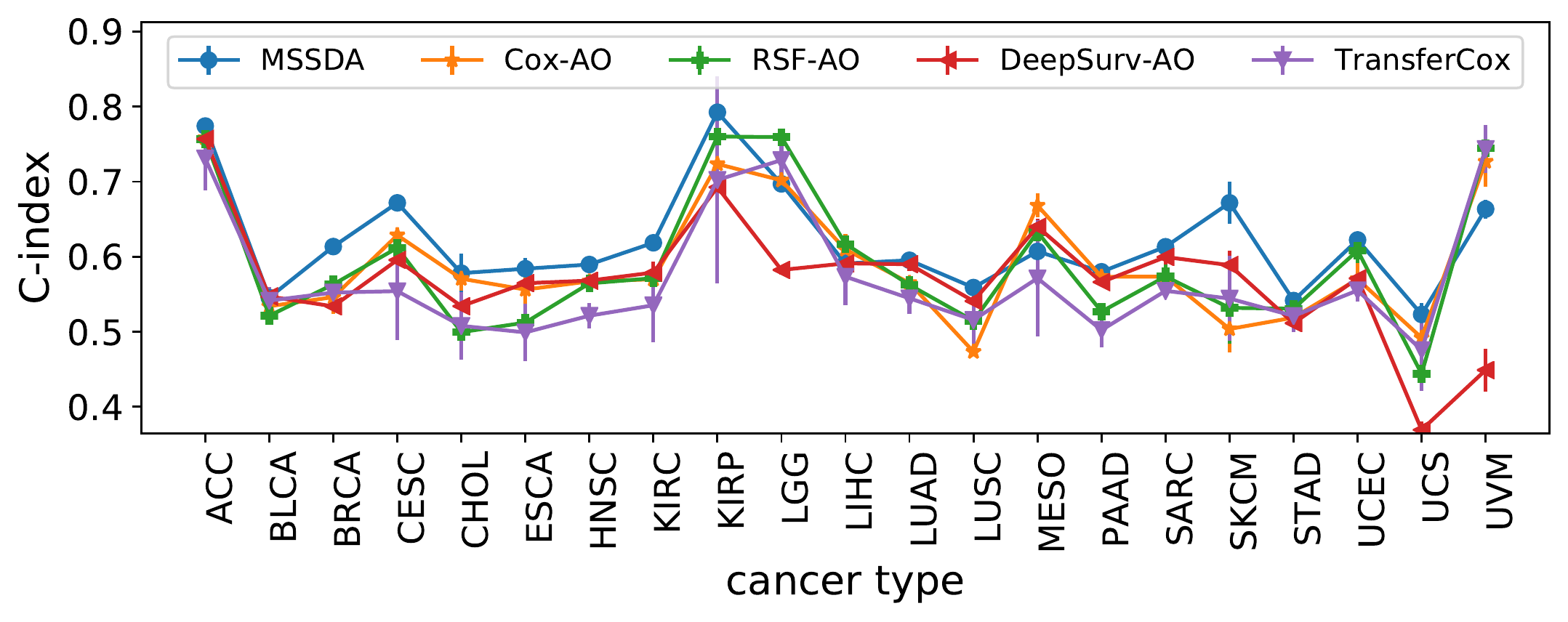}
         \caption{$\text{C-index}$ on miRNA with 25\% supervision.}
         \label{fig:all_miRNAPartial_Supervision_0.25_RSF_rem}
     \end{subfigure}     
        \caption{Performance comparison, on the miRNA data, in terms of the $\text{C-index}$.}
        \label{fig:all_miRNA_C-index_rem_strict}

\end{figure*}

\begin{figure*}
     \centering
     \begin{subfigure}[b]{.49\textwidth}
         \centering
         \includegraphics[width=.98\textwidth]{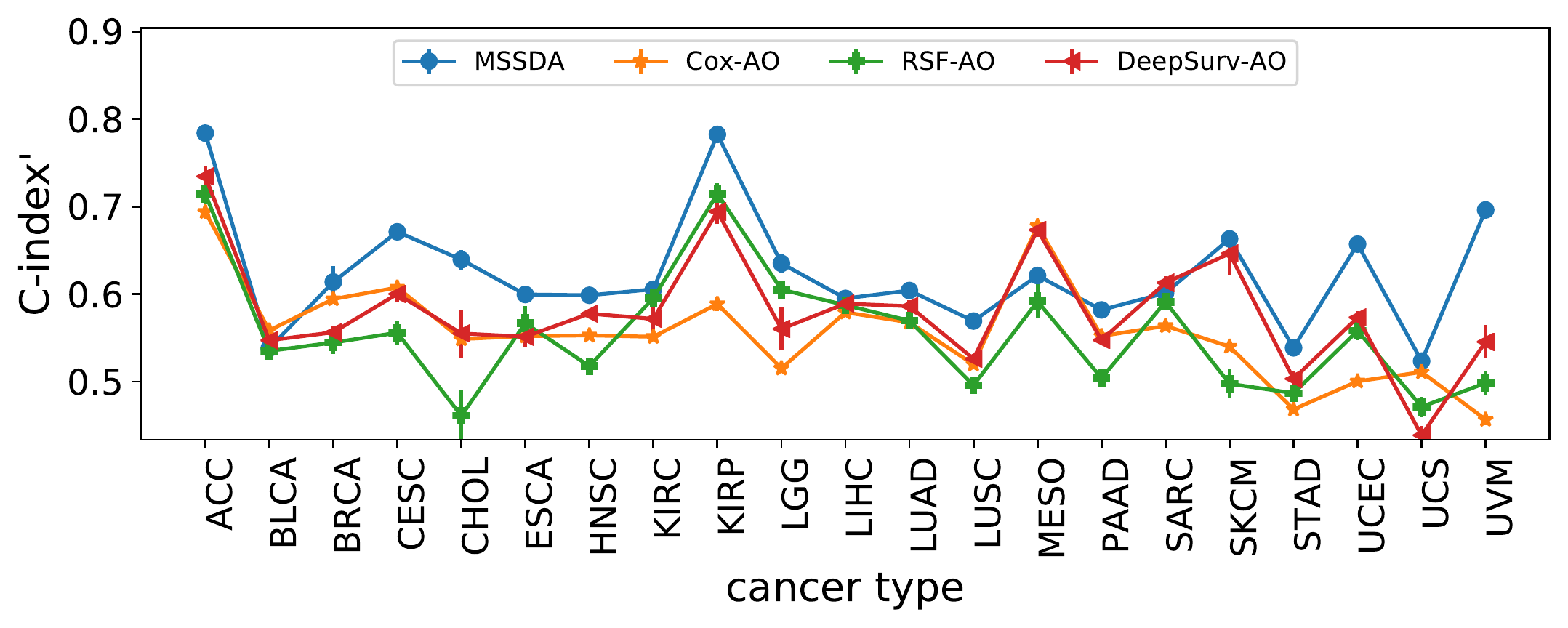}
         \caption{$\text{C-index}^\prime$ on miRNA with no supervision.}
         \label{fig:all_miRNANo_Supervision_RSF}
     \end{subfigure}
     \begin{subfigure}[b]{.49\textwidth}
         \centering
         \includegraphics[width=.98\textwidth]{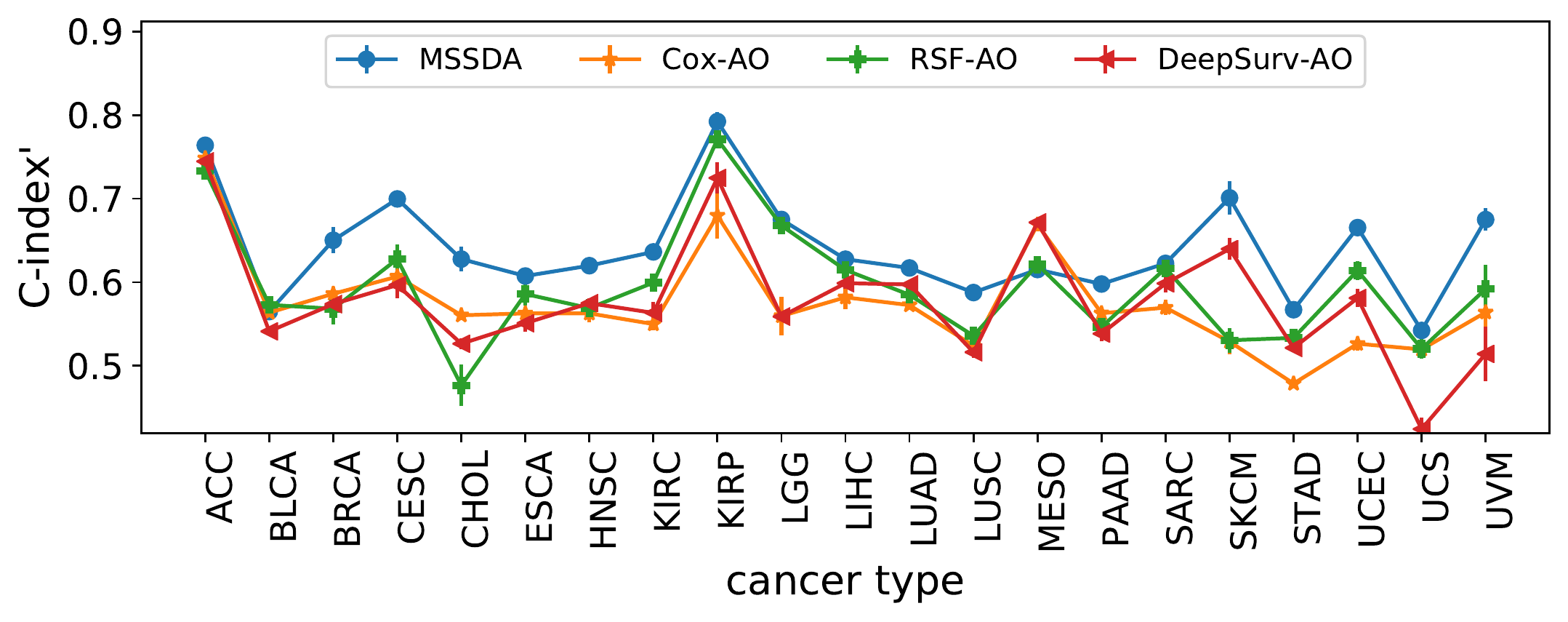}
         \caption{$\text{C-index}^\prime$ on miRNA with 5\% supervision.}
         \label{fig:all_miRNAPartial_Supervision_0.05_RSF}
     \end{subfigure}

     \begin{subfigure}[b]{0.49\textwidth}
         \centering
         \includegraphics[width=.98\linewidth]{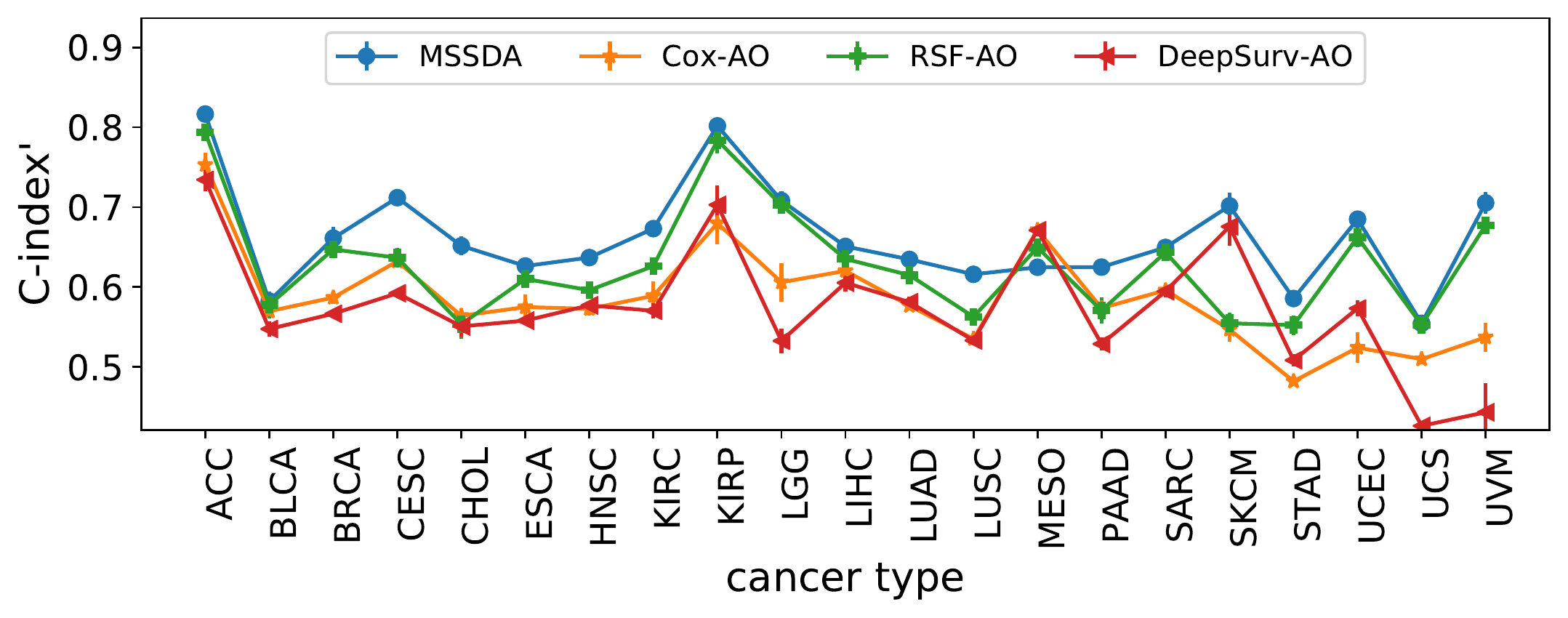}
         \caption{$\text{C-index}^\prime$ on miRNA with 10\% supervision.}
         \label{fig:all_miRNAPartial_Supervision_0.1_RSF}
     \end{subfigure}
     \begin{subfigure}[b]{0.49\textwidth}
         \centering
         \includegraphics[width=.98\linewidth]{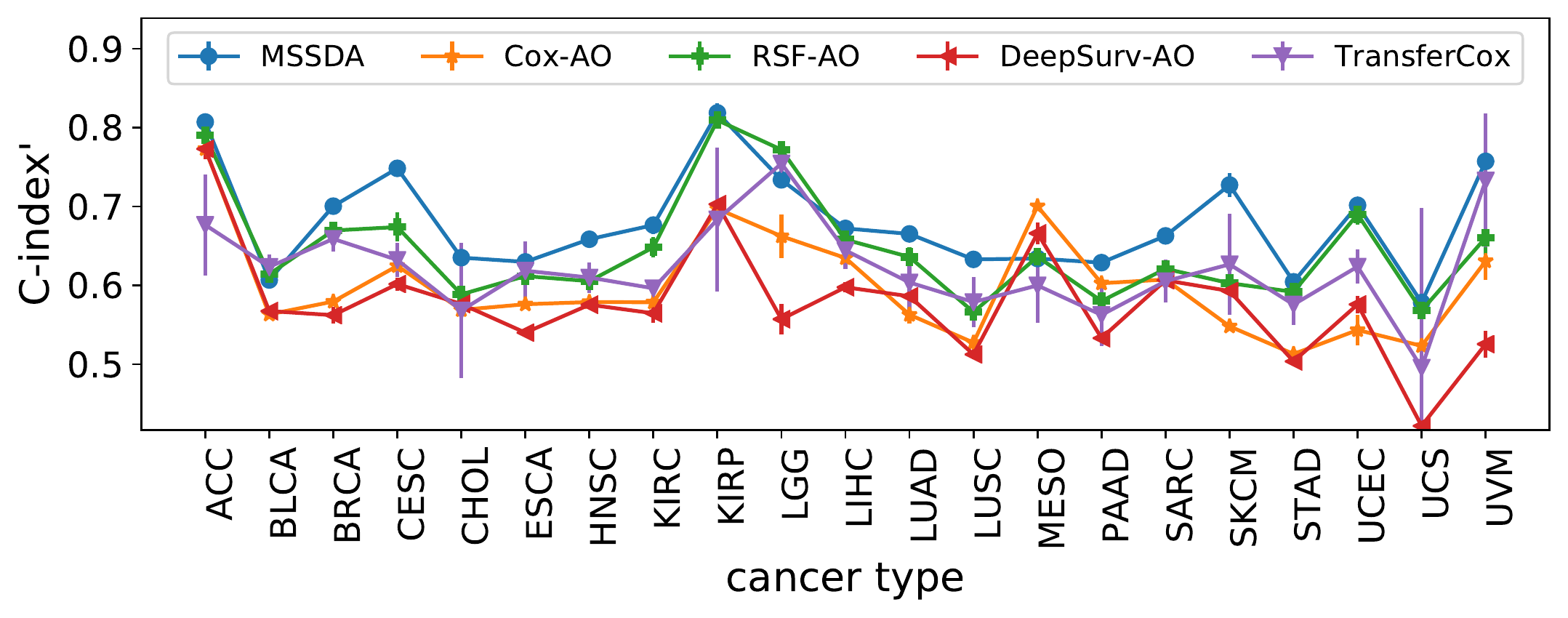}
         \caption{$\text{C-index}^\prime$ on miRNA with 15\% supervision.}
         \label{fig:all_miRNAPartial_Supervision_0.15_RSF}
     \end{subfigure}

     \begin{subfigure}[b]{0.49\textwidth}
         \centering
         \includegraphics[width=.98\linewidth]{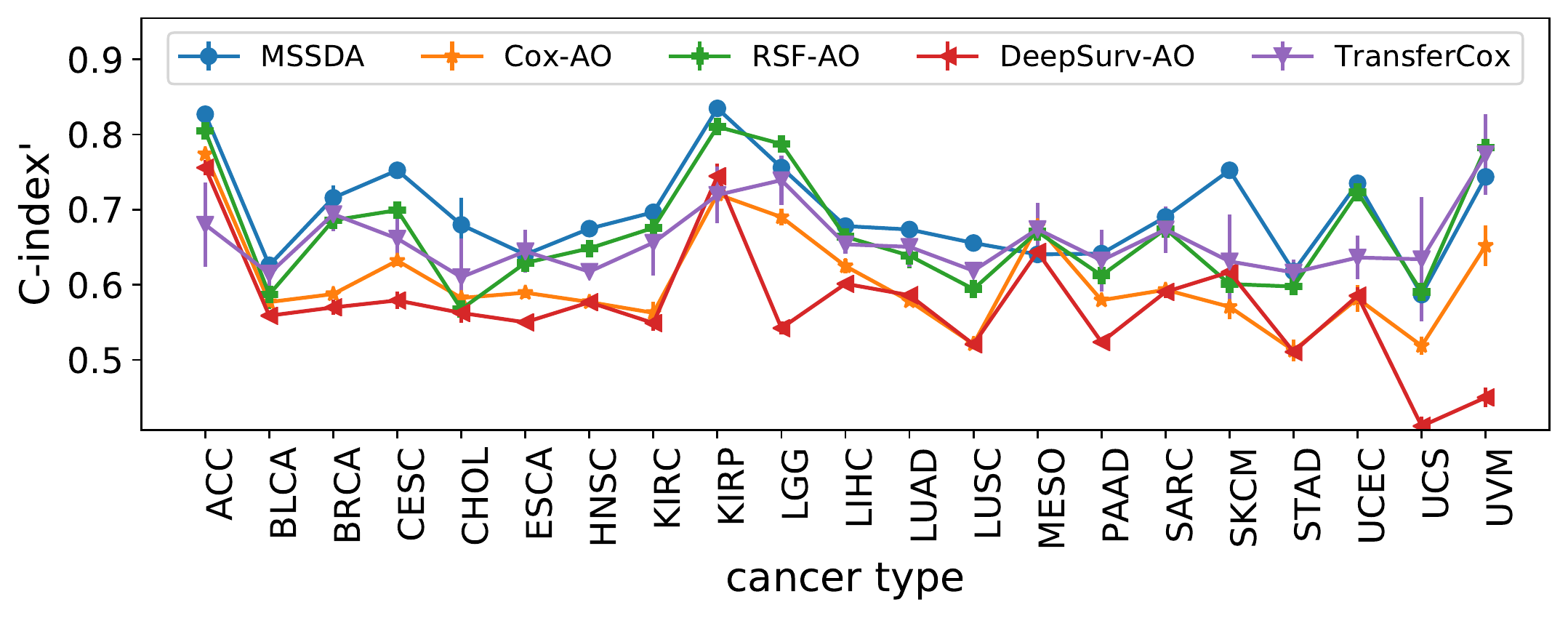}
         \caption{$\text{C-index}^\prime$ on miRNA with 20\% supervision.}
         \label{fig:all_miRNAPartial_Supervision_0.2_RSF}
     \end{subfigure}
     \begin{subfigure}[b]{0.49\textwidth}
         \centering
         \includegraphics[width=.98\linewidth]{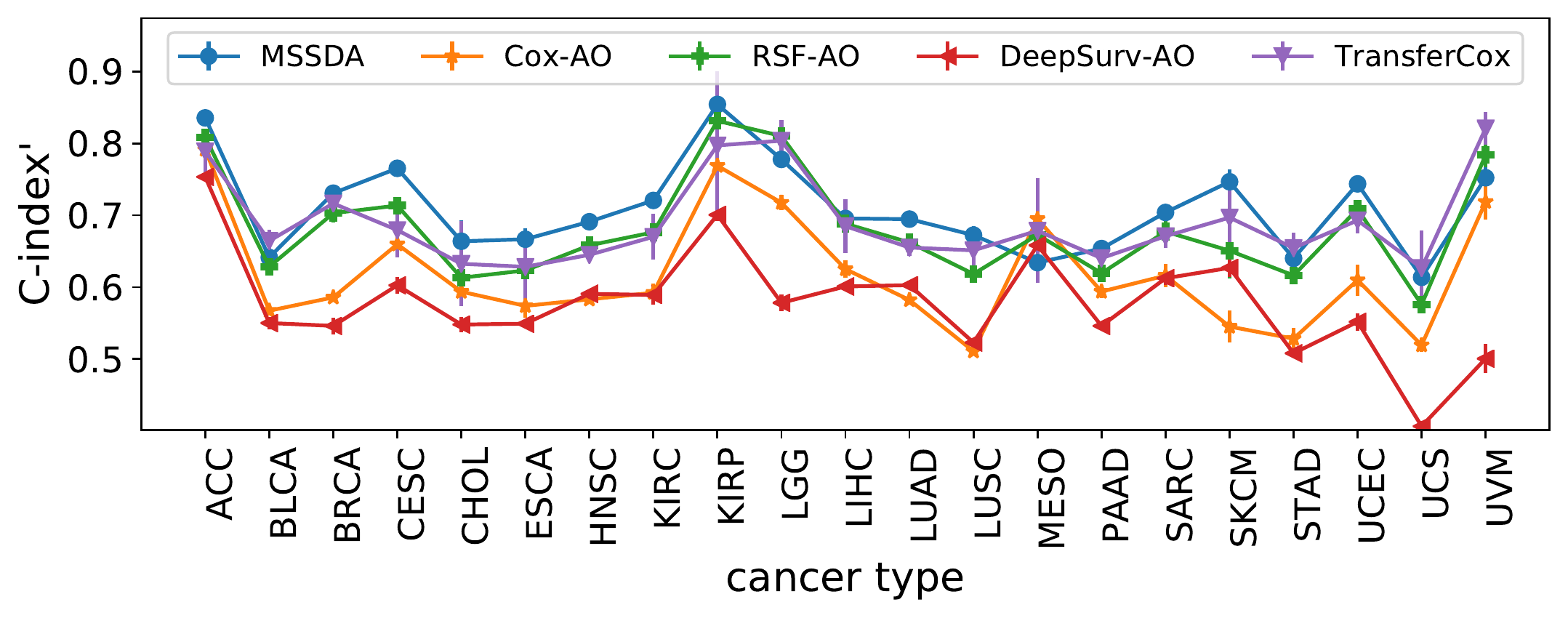}
         \caption{$\text{C-index}^\prime$ on miRNA with 25\% supervision.}
         \label{fig:all_miRNAPartial_Supervision_0.25_RSF}
     \end{subfigure}     
        \caption{Performance comparison, on the miRNA data, in terms of the $\text{C-index}^\prime$.}
        \label{fig:all_miRNA_C-index_strict}

\end{figure*}
\begin{figure*}
     \centering
     \begin{subfigure}[b]{0.49\textwidth}
         \centering
         \includegraphics[width=.98\linewidth]{figures/mRNA_RSF_Rank_c_index_rem.pdf}
         \caption{Rank of the different methods based on the $\text{C-index}$ on mRNA.}
         \label{fig:Supp_mRNA_RSF_Rank_c_index_rem}
     \end{subfigure}
     \begin{subfigure}[b]{0.49\textwidth}
         \centering
         \includegraphics[width=.98\linewidth]{figures/miRNA_RSF_Rank_c_index_rem.pdf}
         \caption{Rank of the different methods based on the $\text{C-index}$ on miRNA.}
         \label{fig:Supp_miRNA_RSF_Rank_c_index_rem}
     \end{subfigure}

     \begin{subfigure}[b]{0.49\textwidth}
         \centering
         \includegraphics[width=.98\linewidth]{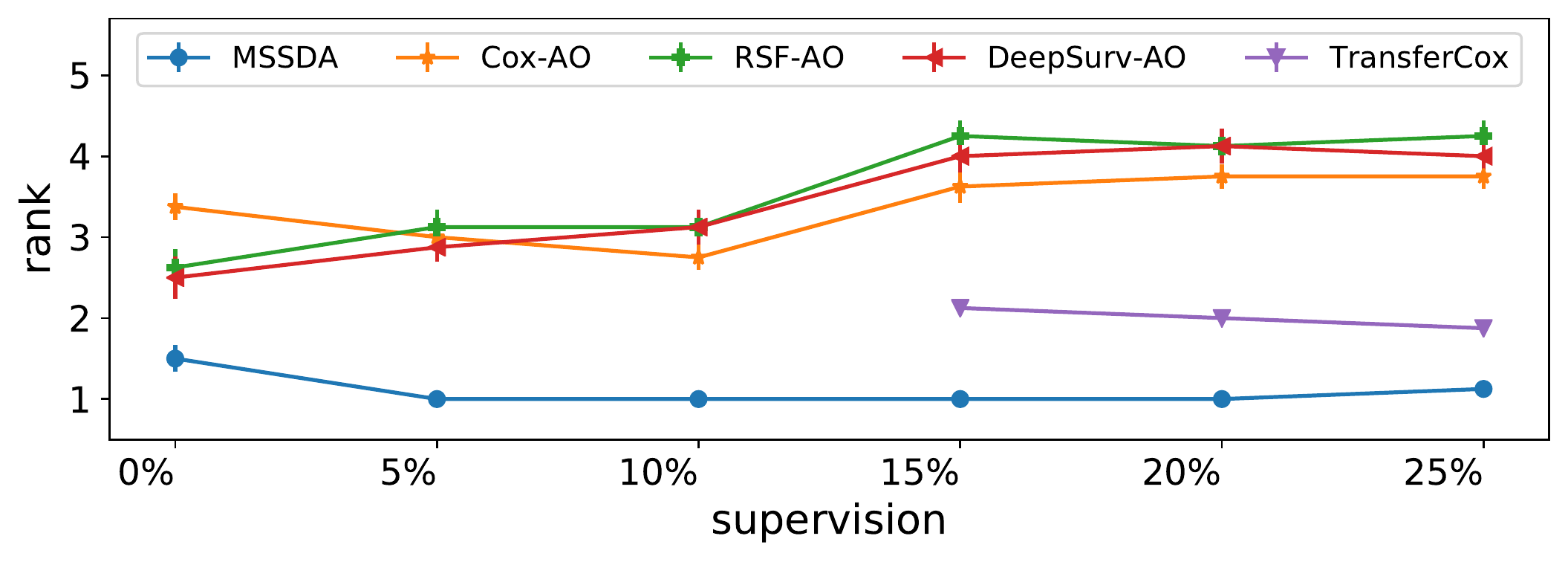}
         \caption{Rank of the different methods based on the $\text{C-index}^\prime$ on mRNA.}
         \label{fig:Supp_mRNA_RSF_Rank_c_index}
     \end{subfigure}     
     \begin{subfigure}[b]{0.49\textwidth}
         \centering
         \includegraphics[width=.98\linewidth]{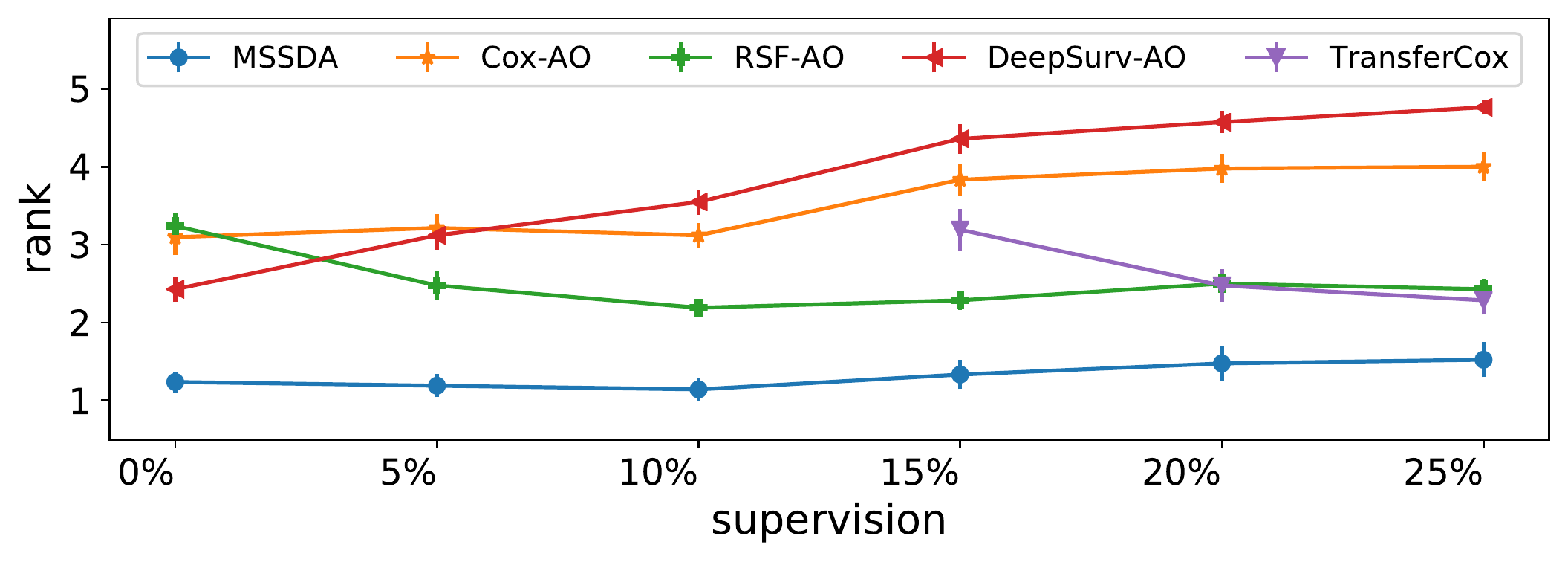}
         \caption{Rank of the different methods based on the $\text{C-index}^\prime$ on miRNA.}
         \label{fig:Supp_miRNA_RSF_Rank_c_index}
     \end{subfigure}     
        \caption{Performance comparison in terms of ranking, on the miRNA and mRNA data.}
        \label{fig:Supp_miRNA_mRNA_Rank_C-index_and_rem_strict_C-index}
\end{figure*}
\begin{table*}
\centering
\begin{tabular}{|lp{0.035\textwidth}p{0.035\textwidth}p{0.035\textwidth}p{0.035\textwidth}|p{0.035\textwidth}p{0.035\textwidth}p{0.035\textwidth}p{0.035\textwidth}|p{0.035\textwidth}p{0.035\textwidth}p{0.035\textwidth}p{0.035\textwidth}|}

\toprule
superv.&\multicolumn{4}{c|}{.00\%} & \multicolumn{4}{c|}{5\%} & \multicolumn{4}{c|}{10\%}  \\
method& \rot{MSSDA} & \rot{Cox-AO} & \rot{RSF-AO} & \rot{DeepSurv-AO}  & \rot{MSSDA} & \rot{Cox-AO} & \rot{RSF-AO} & \rot{DeepSurv-AO} & \rot{MSSDA} & \rot{Cox-AO} & \rot{RSF-AO} & \rot{DeepSurv-AO}  \\
\midrule

KIRC &  \textbf{.604} &         .435 &           .480 &              .466 &  \textbf{.618} &         .444 &         .521 &              .479 &  \textbf{.618} &         .499 &         .543 &              .464 \\
&    \tiny{(.028)} &         \tiny{(.008)} &         \tiny{(.002)} &              \tiny{(.007)} &    \tiny{(.019)} &         \tiny{(.007)} &         \tiny{(.004)} &              \tiny{(.004)} &    \tiny{(.019)} &         \tiny{(.008)} &         \tiny{(.013)} &              \tiny{(.009)} \\
OV   &  \textbf{.539} &         .537 &           .497 &              .538 &  \textbf{.563} &         .523 &         .490 &              .552 &  \textbf{.563} &         .521 &         .501 &              .536 \\
&    \tiny{(.023)} &         \tiny{(.008)} &         \tiny{(.002)} &              \tiny{(.002)} &    \tiny{(.015)} &         \tiny{(.013)} &         \tiny{(.006)} &              \tiny{(.002)} &    \tiny{(.015)} &         \tiny{(.007)} &         \tiny{(.009)} &              \tiny{(.005)} \\
GBM  &  \textbf{.516} &         .428 &           .436 &              .511 &           .493 &         .443 &         .494 &     \textbf{.515} &           .493 &         .443 &         .492 &     \textbf{.507} \\
&    \tiny{(.016)} &         \tiny{(.004)} &         \tiny{(.002)} &              \tiny{(.006)} &    \tiny{(.017)} &         \tiny{(.009)} &         \tiny{(.008)} &              \tiny{(.007)} &    \tiny{(.017)} &         \tiny{(.006)} &         \tiny{(.013)} &              \tiny{(.006)} \\
LUAD &  \textbf{.646} &         .536 &           .466 &              .530 &  \textbf{.661} &         .556 &         .466 &              .507 &  \textbf{.656} &         .578 &         .539 &              .502 \\
&    \tiny{(.018)} &         \tiny{(.020)} &         \tiny{(.008)} &              \tiny{(.006)} &    \tiny{(.014)} &         \tiny{(.020)} &         \tiny{(.015)} &              \tiny{(.005)} &    \tiny{(.010)} &         \tiny{(.006)} &         \tiny{(.021)} &              \tiny{(.007)} \\
LUSC &  \textbf{.659} &         .471 &           .515 &              .520 &  \textbf{.658} &         .533 &         .494 &              .508 &  \textbf{.666} &         .524 &         .528 &              .521 \\
&    \tiny{(.013)} &         \tiny{(.012)} &         \tiny{(.005)} &              \tiny{(.005)} &    \tiny{(.024)} &         \tiny{(.007)} &         \tiny{(.029)} &              \tiny{(.006)} &    \tiny{(.022)} &         \tiny{(.020)} &         \tiny{(.017)} &              \tiny{(.012)} \\
BRCA &           .597 &         .474 &  \textbf{.602} &              .437 &  \textbf{.599} &         .492 &         .536 &              .418 &  \textbf{.599} &         .538 &         .476 &              .393 \\
&    \tiny{(.029)} &         \tiny{(.009)} &         \tiny{(.004)} &              \tiny{(.007)} &    \tiny{(.033)} &         \tiny{(.019)} &         \tiny{(.017)} &              \tiny{(.009)} &    \tiny{(.033)} &         \tiny{(.004)} &         \tiny{(.009)} &              \tiny{(.014)} \\
HNSC &  \textbf{.617} &         .513 &           .514 &              .458 &  \textbf{.646} &         .430 &         .504 &              .441 &  \textbf{.646} &         .445 &         .508 &              .448 \\
&    \tiny{(.022)} &         \tiny{(.013)} &         \tiny{(.003)} &              \tiny{(.003)} &    \tiny{(.017)} &         \tiny{(.009)} &         \tiny{(.016)} &              \tiny{(.005)} &    \tiny{(.017)} &         \tiny{(.015)} &         \tiny{(.014)} &              \tiny{(.004)} \\
LAML &           .476 &         .456 &           .486 &     \textbf{.570} &           .533 &         .451 &         .514 &     \textbf{.551} &           .534 &         .479 &         .503 &     \textbf{.548} \\
&    \tiny{(.033)} &         \tiny{(.010)} &         \tiny{(.002)} &              \tiny{(.010)} &    \tiny{(.024)} &         \tiny{(.003)} &         \tiny{(.014)} &              \tiny{(.008)} &    \tiny{(.024)} &         \tiny{(.008)} &         \tiny{(.009)} &              \tiny{(.010)} \\
\bottomrule
P-value&
&\scriptsize{5e-3}&\scriptsize{2e-2}&\scriptsize{2e-2}&
&\scriptsize{3e-3}&\scriptsize{9e-3}&\scriptsize{9e-3}&
&\scriptsize{1e-2}&\scriptsize{1e-2}&\scriptsize{9e-3}\\
\bottomrule
Rank& \textbf{1.38}& 3.38& 2.63& 2.63& \textbf{1.38}& 3.25& 2.88& 2.50 & \textbf{1.25}& 3.12& 2.75& 2.88 \\
\bottomrule
\end{tabular}
\caption{The performance comparison on the mRNA data in terms of $\text{C-index}$. The following settings were used: no supervision, 5\%, and 10\% supervision. The numbers in brackets depict the standard error. The last row shows the rank in each supervision group.
The p-value row depicts the p-value for the upper-tailed Wilcoxon signed-ranks test between each method and MDSSA.}\label{tab:Supp-mRNA-C-index-rem-1-RSF}
\end{table*}

\begin{table*}
\centering
\begin{tabular}{|lp{0.035\textwidth}p{0.035\textwidth}p{0.035\textwidth}p{0.035\textwidth}p{0.035\textwidth}|p{0.035\textwidth}p{0.035\textwidth}p{0.035\textwidth}p{0.035\textwidth}p{0.035\textwidth}|p{0.035\textwidth}p{0.035\textwidth}p{0.035\textwidth}p{0.035\textwidth}p{0.035\textwidth}|}
\toprule
superv.&\multicolumn{5}{c|}{15\%} & \multicolumn{5}{c|}{20\%} & \multicolumn{5}{c|}{25\%} \\
method& \rot{MSSDA} & \rot{TransferCox} & \rot{Cox-AO} & \rot{RSF-AO} & \rot{DeepSurv-AO} &\rot{MSSDA} & \rot{TransferCox} & \rot{Cox-AO} & \rot{RSF-AO} & \rot{DeepSurv-AO} &\rot{MSSDA} & \rot{TransferCox} & \rot{Cox-AO} & \rot{RSF-AO} & \rot{DeepSurv-AO}  \\
\midrule
KIRC &  \textbf{.618} &       .495 &         .491 &           .483 &              .460 &  \textbf{.618} &       .500 &         .504 &           .503 &              .450 &  \textbf{.618} &       .520 &         .497 &           .524 &              .467 \\
&    \tiny{(.019)} &       \tiny{(.024)} &         \tiny{(.009)} &         \tiny{(.011)} &              \tiny{(.004)} &    \tiny{(.019)} &       \tiny{(.043)} &         \tiny{(.006)} &         \tiny{(.001)} &              \tiny{(.011)} &    \tiny{(.019)} &       \tiny{(.026)} &         \tiny{(.009)} &         \tiny{(.021)} &              \tiny{(.007)} \\
OV   &  \textbf{.563} &       .468 &         .521 &           .496 &              .536 &  \textbf{.563} &       .508 &         .522 &           .500 &              .550 &  \textbf{.563} &       .498 &         .514 &           .477 &              .538 \\
&    \tiny{(.015)} &       \tiny{(.022)} &         \tiny{(.013)} &         \tiny{(.016)} &              \tiny{(.005)} &    \tiny{(.015)} &       \tiny{(.023)} &         \tiny{(.007)} &         \tiny{(.016)} &              \tiny{(.005)} &    \tiny{(.015)} &       \tiny{(.035)} &         \tiny{(.014)} &         \tiny{(.008)} &              \tiny{(.005)} \\
GBM  &           .493 &       .494 &         .475 &  \textbf{.498} &              .488 &           .493 &       .484 &         .467 &  \textbf{.508} &              .506 &           .493 &       .510 &         .494 &  \textbf{.534} &              .500 \\
&    \tiny{(.017)} &       \tiny{(.035)} &         \tiny{(.005)} &         \tiny{(.010)} &              \tiny{(.005)} &    \tiny{(.017)} &       \tiny{(.030)} &         \tiny{(.014)} &         \tiny{(.012)} &              \tiny{(.007)} &    \tiny{(.017)} &       \tiny{(.015)} &         \tiny{(.006)} &         \tiny{(.013)} &              \tiny{(.004)} \\
LUAD &  \textbf{.661} &       .528 &         .591 &           .534 &              .546 &  \textbf{.656} &       .554 &         .589 &           .473 &              .583 &  \textbf{.656} &       .506 &         .575 &           .510 &              .475 \\
&    \tiny{(.014)} &       \tiny{(.076)} &         \tiny{(.006)} &         \tiny{(.022)} &              \tiny{(.010)} &    \tiny{(.010)} &       \tiny{(.031)} &         \tiny{(.007)} &         \tiny{(.035)} &              \tiny{(.007)} &    \tiny{(.010)} &       \tiny{(.055)} &         \tiny{(.004)} &         \tiny{(.030)} &              \tiny{(.008)} \\
LUSC &  \textbf{.662} &       .530 &         .538 &           .474 &              .508 &  \textbf{.666} &       .506 &         .544 &           .473 &              .482 &  \textbf{.658} &       .529 &         .559 &           .522 &              .487 \\
&    \tiny{(.021)} &       \tiny{(.019)} &         \tiny{(.009)} &         \tiny{(.021)} &              \tiny{(.005)} &    \tiny{(.022)} &       \tiny{(.035)} &         \tiny{(.010)} &         \tiny{(.007)} &              \tiny{(.006)} &    \tiny{(.024)} &       \tiny{(.046)} &         \tiny{(.009)} &         \tiny{(.025)} &              \tiny{(.004)} \\
BRCA &  \textbf{.600} &       .488 &         .538 &           .494 &              .404 &  \textbf{.599} &       .474 &         .551 &           .484 &              .454 &  \textbf{.599} &       .482 &         .526 &           .452 &              .429 \\
&    \tiny{(.032)} &       \tiny{(.029)} &         \tiny{(.014)} &         \tiny{(.016)} &              \tiny{(.006)} &    \tiny{(.033)} &       \tiny{(.024)} &         \tiny{(.008)} &         \tiny{(.025)} &              \tiny{(.008)} &    \tiny{(.032)} &       \tiny{(.036)} &         \tiny{(.009)} &         \tiny{(.016)} &              \tiny{(.017)} \\
HNSC &  \textbf{.646} &       .484 &         .441 &           .518 &              .448 &  \textbf{.646} &       .558 &         .438 &           .534 &              .440 &  \textbf{.646} &       .529 &         .446 &           .506 &              .446 \\
&    \tiny{(.017)} &       \tiny{(.026)} &         \tiny{(.007)} &         \tiny{(.009)} &              \tiny{(.003)} &    \tiny{(.017)} &       \tiny{(.029)} &         \tiny{(.006)} &         \tiny{(.014)} &              \tiny{(.006)} &    \tiny{(.017)} &       \tiny{(.030)} &         \tiny{(.006)} &         \tiny{(.011)} &              \tiny{(.006)} \\
LAML &           .534 &       .503 &         .482 &           .505 &     \textbf{.555} &  \textbf{.533} &       .513 &         .472 &           .492 &              .528 &           .533 &       .516 &         .490 &           .477 &     \textbf{.571} \\
&    \tiny{(.024)} &       \tiny{(.049)} &         \tiny{(.004)} &         \tiny{(.017)} &              \tiny{(.009)} &    \tiny{(.024)} &       \tiny{(.035)} &         \tiny{(.012)} &         \tiny{(.010)} &              \tiny{(.008)} &    \tiny{(.024)} &       \tiny{(.052)} &         \tiny{(.008)} &         \tiny{(.013)} &              \tiny{(.006)} \\
\bottomrule
P-value&
&\scriptsize{1e-2}&\scriptsize{9e-3}&\scriptsize{9e-3}&\scriptsize{5e-3}&
&\scriptsize{2e-2}&\scriptsize{5e-3}&\scriptsize{9e-3}&\scriptsize{1e-2}&
&\scriptsize{2e-2}&\scriptsize{6e-3}&\scriptsize{6e-3}&\scriptsize{9e-3}\\
\bottomrule
Rank& \textbf{1.38}& 3.50& 3.38& 3.25& 3.50& \textbf{1.25}& 3.50& 3.25& 3.63& 3.38&  \textbf{1.63}& 3.00& 3.19& 3.38& 3.81 \\
\bottomrule
\end{tabular}
\caption{The performance comparison on the mRNA data in terms of $\text{C-index}$. The following settings were used: 15\%, 20\%, and 25\% supervision. The numbers in brackets depict the standard error. The last row shows the rank in each supervision group.
The p-value row depicts the p-value for the upper-tailed Wilcoxon signed-ranks test between each method and MDSSA.}\label{tab:Supp-mRNA-C-index-rem-2-RSF}
\end{table*}

\begin{table*}
\centering
\begin{tabular}{|lp{0.035\textwidth}p{0.035\textwidth}p{0.035\textwidth}p{0.035\textwidth}|p{0.035\textwidth}p{0.035\textwidth}p{0.035\textwidth}p{0.035\textwidth}|p{0.035\textwidth}p{0.035\textwidth}p{0.035\textwidth}p{0.035\textwidth}|}

\toprule
superv.&\multicolumn{4}{c|}{.00\%} & \multicolumn{4}{c|}{5\%} & \multicolumn{4}{c|}{10\%}  \\
method& \rot{MSSDA} & \rot{Cox-AO} & \rot{RSF-AO} & \rot{DeepSurv-AO}  & \rot{MSSDA} & \rot{Cox-AO} & \rot{RSF-AO} & \rot{DeepSurv-AO} & \rot{MSSDA} & \rot{Cox-AO} & \rot{RSF-AO} & \rot{DeepSurv-AO}  \\
\midrule

KIRC &  \textbf{.604} &         .435 &           .480 &              .467 &  \textbf{.774} &         .450 &         .485 &              .474 &  \textbf{.774} &         .506 &         .487 &              .463 \\
&    \tiny{(.028)} &         \tiny{(.008)} &         \tiny{(.002)} &              \tiny{(.005)} &    \tiny{(.010)} &         \tiny{(.008)} &         \tiny{(.001)} &              \tiny{(.006)} &    \tiny{(.010)} &         \tiny{(.006)} &         \tiny{(.003)} &              \tiny{(.006)} \\
OV   &           .539 &         .537 &           .497 &     \textbf{.546} &  \textbf{.809} &         .526 &         .496 &              .539 &  \textbf{.809} &         .522 &         .495 &              .537 \\
&    \tiny{(.023)} &         \tiny{(.008)} &         \tiny{(.002)} &              \tiny{(.004)} &    \tiny{(.009)} &         \tiny{(.009)} &         \tiny{(.001)} &              \tiny{(.004)} &    \tiny{(.009)} &         \tiny{(.005)} &         \tiny{(.002)} &              \tiny{(.004)} \\
GBM  &  \textbf{.516} &         .428 &           .436 &              .515 &  \textbf{.784} &         .448 &         .435 &              .518 &  \textbf{.784} &         .455 &         .435 &              .505 \\
&    \tiny{(.016)} &         \tiny{(.004)} &         \tiny{(.002)} &              \tiny{(.004)} &    \tiny{(.008)} &         \tiny{(.008)} &         \tiny{(.001)} &              \tiny{(.005)} &    \tiny{(.008)} &         \tiny{(.004)} &         \tiny{(.003)} &              \tiny{(.006)} \\
LUAD &  \textbf{.646} &         .536 &           .466 &              .523 &  \textbf{.770} &         .577 &         .447 &              .513 &  \textbf{.766} &         .599 &         .470 &              .514 \\
&    \tiny{(.018)} &         \tiny{(.020)} &         \tiny{(.008)} &              \tiny{(.004)} &    \tiny{(.008)} &         \tiny{(.018)} &         \tiny{(.007)} &              \tiny{(.011)} &    \tiny{(.005)} &         \tiny{(.007)} &         \tiny{(.002)} &              \tiny{(.005)} \\
LUSC &  \textbf{.659} &         .471 &           .515 &              .523 &  \textbf{.805} &         .555 &         .511 &              .514 &  \textbf{.808} &         .550 &         .517 &              .511 \\
&    \tiny{(.013)} &         \tiny{(.012)} &         \tiny{(.005)} &              \tiny{(.004)} &    \tiny{(.012)} &         \tiny{(.006)} &         \tiny{(.010)} &              \tiny{(.005)} &    \tiny{(.012)} &         \tiny{(.018)} &         \tiny{(.010)} &              \tiny{(.009)} \\
BRCA &           .597 &         .474 &  \textbf{.602} &              .453 &  \textbf{.671} &         .499 &         .606 &              .439 &  \textbf{.671} &         .535 &         .600 &              .433 \\
&    \tiny{(.029)} &         \tiny{(.009)} &         \tiny{(.004)} &              \tiny{(.010)} &    \tiny{(.026)} &         \tiny{(.013)} &         \tiny{(.003)} &              \tiny{(.009)} &    \tiny{(.026)} &         \tiny{(.006)} &         \tiny{(.004)} &              \tiny{(.011)} \\
HNSC &  \textbf{.617} &         .513 &           .514 &              .465 &  \textbf{.784} &         .449 &         .517 &              .447 &  \textbf{.784} &         .457 &         .515 &              .450 \\
&    \tiny{(.022)} &         \tiny{(.013)} &         \tiny{(.003)} &              \tiny{(.004)} &    \tiny{(.009)} &         \tiny{(.006)} &         \tiny{(.005)} &              \tiny{(.006)} &    \tiny{(.009)} &         \tiny{(.017)} &         \tiny{(.004)} &              \tiny{(.002)} \\
LAML &           .476 &         .456 &           .486 &     \textbf{.559} &  \textbf{.767} &         .456 &         .485 &              .556 &  \textbf{.767} &         .481 &         .484 &              .554 \\
&    \tiny{(.033)} &         \tiny{(.010)} &         \tiny{(.002)} &              \tiny{(.009)} &    \tiny{(.010)} &         \tiny{(.003)} &         \tiny{(.002)} &              \tiny{(.007)} &    \tiny{(.010)} &         \tiny{(.003)} &         \tiny{(.003)} &              \tiny{(.002)} \\
\bottomrule
P-value&
&\scriptsize{2e-3}&\scriptsize{8e-3}&\scriptsize{1e-2}&
&\scriptsize{4e-4}&\scriptsize{4e-4}&\scriptsize{4e-4}&
&\scriptsize{4e-4}&\scriptsize{4e-4}&\scriptsize{4e-4}\\
\bottomrule
Rank& \textbf{1.50}& 3.36& 2.62& 2.50& \textbf{1.00}& 3.00& 3.12& 2.88 & \textbf{1.00}& 2.75& 3.13& 3.13 \\
\bottomrule
\end{tabular}
\caption{The performance comparison on the mRNA data in terms of $\text{C-index}^\prime$. The following settings were used: no supervision, 5\%, and 10\% supervision. The numbers in brackets depict the standard error. The last row shows the rank in each supervision group.
The p-value row depicts the p-value for the upper-tailed Wilcoxon signed-ranks test between each method and MDSSA.}\label{tab:Supp-mRNA-C-index-1-RSF}
\end{table*}
\begin{table*}
\centering
\begin{tabular}{|lp{0.035\textwidth}p{0.035\textwidth}p{0.035\textwidth}p{0.035\textwidth}p{0.035\textwidth}|p{0.035\textwidth}p{0.035\textwidth}p{0.035\textwidth}p{0.035\textwidth}p{0.035\textwidth}|p{0.035\textwidth}p{0.035\textwidth}p{0.035\textwidth}p{0.035\textwidth}p{0.035\textwidth}|}

\toprule
superv.&\multicolumn{5}{c|}{15\%} & \multicolumn{5}{c|}{20\%} & \multicolumn{5}{c|}{25\%} \\
method& \rot{MSSDA} & \rot{TransferCox} & \rot{Cox-AO} & \rot{RSF-AO} & \rot{DeepSurv-AO} &\rot{MSSDA} & \rot{TransferCox} & \rot{Cox-AO} & \rot{RSF-AO} & \rot{DeepSurv-AO} &\rot{MSSDA} & \rot{TransferCox} & \rot{Cox-AO} & \rot{RSF-AO} & \rot{DeepSurv-AO}  \\
\midrule

KIRC &  \textbf{.774} &       .575 &         .503 &         .483 &              .458 &  \textbf{.774} &       .602 &         .519 &         .486 &              .463 &  \textbf{.774} &           .656 &         .524 &         .495 &              .462 \\
&    \tiny{(.010)} &       \tiny{(.014)} &         \tiny{(.005)} &         \tiny{(.002)} &              \tiny{(.003)} &    \tiny{(.010)} &       \tiny{(.046)} &         \tiny{(.005)} &         \tiny{(.002)} &              \tiny{(.004)} &    \tiny{(.010)} &       \tiny{(.031)} &         \tiny{(.007)} &         \tiny{(.003)} &              \tiny{(.008)} \\
OV   &  \textbf{.809} &       .555 &         .529 &         .496 &              .540 &  \textbf{.809} &       .611 &         .533 &         .495 &              .548 &  \textbf{.809} &           .634 &         .535 &         .494 &              .540 \\
&    \tiny{(.009)} &       \tiny{(.026)} &         \tiny{(.012)} &         \tiny{(.002)} &              \tiny{(.004)} &    \tiny{(.009)} &       \tiny{(.015)} &         \tiny{(.006)} &         \tiny{(.002)} &              \tiny{(.003)} &    \tiny{(.009)} &       \tiny{(.026)} &         \tiny{(.007)} &         \tiny{(.002)} &              \tiny{(.004)} \\
GBM  &  \textbf{.784} &       .575 &         .489 &         .438 &              .519 &  \textbf{.784} &       .584 &         .467 &         .438 &              .521 &  \textbf{.784} &           .636 &         .496 &         .438 &              .521 \\
&    \tiny{(.008)} &       \tiny{(.027)} &         \tiny{(.008)} &         \tiny{(.002)} &              \tiny{(.004)} &    \tiny{(.008)} &       \tiny{(.026)} &         \tiny{(.010)} &         \tiny{(.005)} &              \tiny{(.006)} &    \tiny{(.008)} &       \tiny{(.008)} &         \tiny{(.005)} &         \tiny{(.002)} &              \tiny{(.004)} \\
LUAD &  \textbf{.770} &       .606 &         .609 &         .479 &              .518 &  \textbf{.766} &       .643 &         .611 &         .498 &              .524 &  \textbf{.766} &           .659 &         .597 &         .485 &              .523 \\
&    \tiny{(.008)} &       \tiny{(.072)} &         \tiny{(.004)} &         \tiny{(.012)} &              \tiny{(.008)} &    \tiny{(.005)} &       \tiny{(.031)} &         \tiny{(.010)} &         \tiny{(.020)} &              \tiny{(.008)} &    \tiny{(.005)} &       \tiny{(.022)} &         \tiny{(.008)} &         \tiny{(.022)} &              \tiny{(.004)} \\
LUSC &  \textbf{.806} &       .590 &         .568 &         .491 &              .515 &  \textbf{.808} &       .602 &         .574 &         .538 &              .519 &  \textbf{.805} &           .661 &         .564 &         .520 &              .525 \\
&    \tiny{(.011)} &       \tiny{(.021)} &         \tiny{(.009)} &         \tiny{(.012)} &              \tiny{(.005)} &    \tiny{(.012)} &       \tiny{(.028)} &         \tiny{(.006)} &         \tiny{(.004)} &              \tiny{(.006)} &    \tiny{(.012)} &       \tiny{(.042)} &         \tiny{(.010)} &         \tiny{(.010)} &              \tiny{(.005)} \\
BRCA &  \textbf{.671} &       .613 &         .563 &         .599 &              .420 &  \textbf{.671} &       .643 &         .567 &         .606 &              .433 &           .671 &  \textbf{.694} &         .548 &         .599 &              .429 \\
&    \tiny{(.026)} &       \tiny{(.029)} &         \tiny{(.012)} &         \tiny{(.005)} &              \tiny{(.004)} &    \tiny{(.026)} &       \tiny{(.020)} &         \tiny{(.011)} &         \tiny{(.005)} &              \tiny{(.009)} &    \tiny{(.026)} &       \tiny{(.020)} &         \tiny{(.004)} &         \tiny{(.004)} &              \tiny{(.010)} \\
HNSC &  \textbf{.784} &       .569 &         .464 &         .511 &              .442 &  \textbf{.784} &       .651 &         .457 &         .518 &              .442 &  \textbf{.784} &           .651 &         .459 &         .523 &              .449 \\
&    \tiny{(.009)} &       \tiny{(.018)} &         \tiny{(.012)} &         \tiny{(.011)} &              \tiny{(.004)} &    \tiny{(.009)} &       \tiny{(.018)} &         \tiny{(.008)} &         \tiny{(.008)} &              \tiny{(.007)} &    \tiny{(.009)} &       \tiny{(.016)} &         \tiny{(.006)} &         \tiny{(.007)} &              \tiny{(.005)} \\
LAML &  \textbf{.767} &       .580 &         .478 &         .480 &              .566 &  \textbf{.767} &       .599 &         .475 &         .483 &              .567 &  \textbf{.767} &           .630 &         .485 &         .486 &              .547 \\
&    \tiny{(.010)} &       \tiny{(.042)} &         \tiny{(.002)} &         \tiny{(.004)} &              \tiny{(.009)} &    \tiny{(.010)} &       \tiny{(.031)} &         \tiny{(.008)} &         \tiny{(.003)} &              \tiny{(.007)} &    \tiny{(.010)} &       \tiny{(.045)} &         \tiny{(.006)} &         \tiny{(.002)} &              \tiny{(.005)} \\
\bottomrule
P-value&
&\scriptsize{4e-4}&\scriptsize{4e-4}&\scriptsize{4e-4}&\scriptsize{4e-4}&
&\scriptsize{4e-4}&\scriptsize{4e-4}&\scriptsize{4e-4}&\scriptsize{4e-4}&
&\scriptsize{4e-4}&\scriptsize{4e-4}&\scriptsize{4e-4}&\scriptsize{6e-4}\\
\bottomrule
Rank& \textbf{1.00}& 2.13& 3.63& 4.25& 4.00& \textbf{1.00}& 2.00& 3.75& 4.12& 4.12& \textbf{1.12}& 1.88& 3.75& 4.25& 4.00 \\
\bottomrule
\end{tabular}
\caption{The performance comparison on the mRNA data in terms of $\text{C-index}^\prime$. The following settings were used: 15\%, 20\%, and 25\% supervision. The numbers in brackets depict the standard error. The last row shows the rank in each supervision group.
The p-value row depicts the p-value for the upper-tailed Wilcoxon signed-ranks test between each method and MDSSA.}\label{tab:Supp-mRNA-C-index-2-RSF}
\end{table*}
\begin{table*}
\small
\centering
\begin{tabular}{|lp{.035\textwidth}p{.035\textwidth}p{.035\textwidth}p{.035\textwidth}|p{.035\textwidth}p{.035\textwidth}p{.035\textwidth}p{.035\textwidth}|p{.035\textwidth}p{.035\textwidth}p{.035\textwidth}p{.035\textwidth}|}
\toprule
superv.&\multicolumn{4}{c|}{.00\%} & \multicolumn{4}{c|}{5\%} & \multicolumn{4}{c|}{10\%}  \\
method& \rot{MSSDA} & \rot{Cox-AO} & \rot{RSF-AO} & \rot{DeepSurv-AO}  & \rot{MSSDA} & \rot{Cox-AO} & \rot{RSF-AO} & \rot{DeepSurv-AO} & \rot{MSSDA} & \rot{Cox-AO} & \rot{RSF-AO} & \rot{DeepSurv-AO}  \\
\midrule
ACC  &  \textbf{.784} &           .694 &         .715 &              .731 &  \textbf{.754} &           .742 &         .744 &              .745 &  \textbf{.792} &           .732 &           .774 &              .738 \\
&    \tiny{(.008)} &         \tiny{(.008)} &         \tiny{(.012)} &              \tiny{(.012)} &    \tiny{(.008)} &         \tiny{(.017)} &         \tiny{(.005)} &              \tiny{(.022)} &    \tiny{(.011)} &         \tiny{(.021)} &         \tiny{(.013)} &              \tiny{(.014)} \\
BLCA &           .538 &  \textbf{.559} &         .535 &              .546 &           .544 &  \textbf{.555} &         .529 &              .549 &           .543 &  \textbf{.552} &           .534 &              .549 \\
&    \tiny{(.007)} &         \tiny{(.001)} &         \tiny{(.003)} &              \tiny{(.005)} &    \tiny{(.005)} &         \tiny{(.001)} &         \tiny{(.005)} &              \tiny{(.003)} &    \tiny{(.002)} &         \tiny{(.008)} &         \tiny{(.016)} &              \tiny{(.007)} \\
BRCA &  \textbf{.614} &           .594 &         .545 &              .555 &  \textbf{.621} &           .574 &         .539 &              .575 &  \textbf{.609} &           .566 &           .575 &              .562 \\
&    \tiny{(.018)} &         \tiny{(.001)} &         \tiny{(.013)} &              \tiny{(.006)} &    \tiny{(.017)} &         \tiny{(.008)} &         \tiny{(.019)} &              \tiny{(.011)} &    \tiny{(.014)} &         \tiny{(.013)} &         \tiny{(.013)} &              \tiny{(.006)} \\
CESC &  \textbf{.671} &           .608 &         .556 &              .601 &  \textbf{.677} &           .601 &         .590 &              .590 &  \textbf{.666} &           .621 &           .595 &              .589 \\
&    \tiny{(.005)} &         \tiny{(.001)} &         \tiny{(.014)} &              \tiny{(.010)} &    \tiny{(.008)} &         \tiny{(.005)} &         \tiny{(.020)} &              \tiny{(.013)} &    \tiny{(.005)} &         \tiny{(.008)} &         \tiny{(.013)} &              \tiny{(.005)} \\
CHOL &  \textbf{.639} &           .549 &         .461 &              .564 &  \textbf{.620} &           .566 &         .457 &              .523 &  \textbf{.625} &           .543 &           .500 &              .523 \\
&    \tiny{(.011)} &         \tiny{(.004)} &         \tiny{(.029)} &              \tiny{(.022)} &    \tiny{(.013)} &         \tiny{(.009)} &         \tiny{(.026)} &              \tiny{(.017)} &    \tiny{(.010)} &         \tiny{(.004)} &         \tiny{(.021)} &              \tiny{(.008)} \\
ESCA &  \textbf{.600} &           .552 &         .567 &              .559 &  \textbf{.591} &           .554 &         .545 &              .550 &  \textbf{.593} &           .575 &           .563 &              .559 \\
&    \tiny{(.008)} &         \tiny{(.003)} &         \tiny{(.019)} &              \tiny{(.014)} &    \tiny{(.011)} &         \tiny{(.005)} &         \tiny{(.010)} &              \tiny{(.009)} &    \tiny{(.007)} &         \tiny{(.018)} &         \tiny{(.009)} &              \tiny{(.011)} \\
HNSC &  \textbf{.599} &           .553 &         .518 &              .579 &  \textbf{.598} &           .555 &         .542 &              .576 &  \textbf{.596} &           .556 &           .553 &              .570 \\
&    \tiny{(.004)} &         \tiny{(.002)} &         \tiny{(.010)} &              \tiny{(.001)} &    \tiny{(.004)} &         \tiny{(.012)} &         \tiny{(.012)} &              \tiny{(.007)} &    \tiny{(.003)} &         \tiny{(.007)} &         \tiny{(.010)} &              \tiny{(.005)} \\
KIRC &  \textbf{.606} &           .551 &         .595 &              .569 &  \textbf{.608} &           .542 &         .576 &              .554 &  \textbf{.624} &           .566 &           .571 &              .581 \\
&    \tiny{(.008)} &         \tiny{(.001)} &         \tiny{(.018)} &              \tiny{(.012)} &    \tiny{(.010)} &         \tiny{(.007)} &         \tiny{(.008)} &              \tiny{(.014)} &    \tiny{(.011)} &         \tiny{(.018)} &         \tiny{(.005)} &              \tiny{(.007)} \\
KIRP &  \textbf{.782} &           .589 &         .715 &              .690 &  \textbf{.774} &           .659 &         .736 &              .724 &  \textbf{.776} &           .671 &           .748 &              .730 \\
&    \tiny{(.004)} &         \tiny{(.008)} &         \tiny{(.012)} &              \tiny{(.010)} &    \tiny{(.011)} &         \tiny{(.028)} &         \tiny{(.010)} &              \tiny{(.015)} &    \tiny{(.006)} &         \tiny{(.031)} &         \tiny{(.019)} &              \tiny{(.016)} \\
LGG  &  \textbf{.635} &           .515 &         .605 &              .562 &  \textbf{.649} &           .549 &         .642 &              .560 &  \textbf{.670} &           .594 &  \textbf{.670} &              .539 \\
&    \tiny{(.011)} &         \tiny{(.004)} &         \tiny{(.008)} &              \tiny{(.024)} &    \tiny{(.004)} &         \tiny{(.023)} &         \tiny{(.004)} &              \tiny{(.007)} &    \tiny{(.014)} &         \tiny{(.022)} &         \tiny{(.006)} &              \tiny{(.017)} \\
LIHC &           .595 &           .579 &         .587 &     \textbf{.598} &  \textbf{.603} &           .574 &         .585 &              .594 &           .608 &           .610 &           .593 &     \textbf{.613} \\
&    \tiny{(.003)} &         \tiny{(.006)} &         \tiny{(.002)} &              \tiny{(.007)} &    \tiny{(.008)} &         \tiny{(.014)} &         \tiny{(.003)} &              \tiny{(.002)} &    \tiny{(.009)} &         \tiny{(.013)} &         \tiny{(.011)} &              \tiny{(.012)} \\
LUAD &  \textbf{.604} &           .568 &         .570 &              .587 &           .595 &           .570 &         .561 &     \textbf{.602} &  \textbf{.594} &           .559 &           .567 &              .577 \\
&    \tiny{(.004)} &         \tiny{(.002)} &         \tiny{(.008)} &              \tiny{(.006)} &    \tiny{(.006)} &         \tiny{(.005)} &         \tiny{(.008)} &              \tiny{(.003)} &    \tiny{(.005)} &         \tiny{(.006)} &         \tiny{(.008)} &              \tiny{(.006)} \\
LUSC &  \textbf{.569} &           .520 &         .496 &              .530 &  \textbf{.568} &           .515 &         .514 &              .522 &  \textbf{.574} &           .526 &           .515 &              .527 \\
&    \tiny{(.003)} &         \tiny{(.001)} &         \tiny{(.005)} &              \tiny{(.006)} &    \tiny{(.005)} &         \tiny{(.008)} &         \tiny{(.009)} &              \tiny{(.006)} &    \tiny{(.005)} &         \tiny{(.004)} &         \tiny{(.007)} &              \tiny{(.006)} \\
MESO &           .621 &           .678 &         .592 &     \textbf{.681} &           .609 &           .664 &         .614 &     \textbf{.669} &           .615 &           .668 &           .625 &     \textbf{.681} \\
&    \tiny{(.002)} &         \tiny{(.004)} &         \tiny{(.020)} &              \tiny{(.007)} &    \tiny{(.011)} &         \tiny{(.007)} &         \tiny{(.008)} &              \tiny{(.008)} &    \tiny{(.013)} &         \tiny{(.007)} &         \tiny{(.011)} &              \tiny{(.008)} \\
PAAD &  \textbf{.582} &           .552 &         .504 &              .552 &  \textbf{.580} &           .560 &         .526 &              .541 &  \textbf{.591} &           .583 &           .536 &              .521 \\
&    \tiny{(.004)} &         \tiny{(.002)} &         \tiny{(.010)} &              \tiny{(.004)} &    \tiny{(.007)} &         \tiny{(.006)} &         \tiny{(.013)} &              \tiny{(.010)} &    \tiny{(.010)} &         \tiny{(.011)} &         \tiny{(.015)} &              \tiny{(.013)} \\
SARC &           .601 &           .564 &         .591 &     \textbf{.617} &  \textbf{.608} &           .567 &         .592 &              .586 &  \textbf{.618} &           .581 &           .609 &              .602 \\
&    \tiny{(.006)} &         \tiny{(.004)} &         \tiny{(.005)} &              \tiny{(.007)} &    \tiny{(.005)} &         \tiny{(.010)} &         \tiny{(.008)} &              \tiny{(.010)} &    \tiny{(.012)} &         \tiny{(.007)} &         \tiny{(.009)} &              \tiny{(.010)} \\
SKCM &  \textbf{.663} &           .540 &         .498 &              .629 &  \textbf{.686} &           .522 &         .536 &              .608 &  \textbf{.676} &           .533 &           .528 &              .638 \\
&    \tiny{(.011)} &         \tiny{(.004)} &         \tiny{(.016)} &              \tiny{(.022)} &    \tiny{(.017)} &         \tiny{(.013)} &         \tiny{(.014)} &              \tiny{(.023)} &    \tiny{(.011)} &         \tiny{(.016)} &         \tiny{(.019)} &              \tiny{(.033)} \\
STAD &  \textbf{.539} &           .468 &         .487 &              .507 &  \textbf{.545} &           .469 &         .506 &              .532 &  \textbf{.545} &           .471 &           .511 &              .510 \\
&    \tiny{(.006)} &         \tiny{(.001)} &         \tiny{(.004)} &              \tiny{(.009)} &    \tiny{(.006)} &         \tiny{(.005)} &         \tiny{(.011)} &              \tiny{(.007)} &    \tiny{(.005)} &         \tiny{(.005)} &         \tiny{(.014)} &              \tiny{(.008)} \\
UCEC &  \textbf{.657} &           .500 &         .558 &              .574 &  \textbf{.635} &           .516 &         .595 &              .581 &  \textbf{.636} &           .507 &           .611 &              .557 \\
&    \tiny{(.007)} &         \tiny{(.003)} &         \tiny{(.011)} &              \tiny{(.008)} &    \tiny{(.010)} &         \tiny{(.012)} &         \tiny{(.011)} &              \tiny{(.014)} &    \tiny{(.010)} &         \tiny{(.018)} &         \tiny{(.014)} &              \tiny{(.017)} \\
UCS  &  \textbf{.524} &           .511 &         .471 &              .443 &  \textbf{.513} &           .510 &         .463 &              .414 &  \textbf{.502} &           .492 &           .500 &              .429 \\
&    \tiny{(.007)} &         \tiny{(.001)} &         \tiny{(.012)} &              \tiny{(.007)} &    \tiny{(.010)} &         \tiny{(.008)} &         \tiny{(.010)} &              \tiny{(.016)} &    \tiny{(.008)} &         \tiny{(.003)} &         \tiny{(.015)} &              \tiny{(.004)} \\
UVM  &  \textbf{.696} &           .457 &         .498 &              .535 &  \textbf{.656} &           .541 &         .605 &              .509 &  \textbf{.666} &           .526 &           .662 &              .449 \\
&    \tiny{(.010)} &         \tiny{(.007)} &         \tiny{(.013)} &              \tiny{(.020)} &    \tiny{(.017)} &         \tiny{(.022)} &         \tiny{(.036)} &              \tiny{(.035)} &    \tiny{(.019)} &         \tiny{(.026)} &         \tiny{(.014)} &              \tiny{(.030)} \\
\bottomrule
P-value&
&\scriptsize{9e-4}&\scriptsize{7e-4}&\scriptsize{2e-2}&
&\scriptsize{2e-3}&\scriptsize{6e-3}&\scriptsize{1e-2}&
&\scriptsize{9e-3}&\scriptsize{3e-2}&\scriptsize{1e-2}\\
\bottomrule
Rank& \textbf{1.29}& 3.21& 3.24& 2.26& \textbf{1.29}& 3.05& 3.12& 2.55& \textbf{1.36}& 2.95& 2.83& 2.86\\
\bottomrule
\end{tabular}
\caption{The performance comparison on the miRNA data in terms of $\text{C-index}$. The following settings were used: no supervision, 5\%, and 10\% supervision. The numbers in brackets depict the standard error. The last row shows the rank in each supervision group.
The p-value row depicts the p-value for the upper-tailed Wilcoxon signed-ranks test between each method and MDSSA.}\label{tab:Supp-miRNA-C-index-rem-1-RSF}
\end{table*}
\begin{table*}
\small
\centering
\begin{tabular}{|lp{.035\textwidth}p{.035\textwidth}p{.035\textwidth}p{.035\textwidth}p{.035\textwidth}|p{.035\textwidth}p{.035\textwidth}p{.035\textwidth}p{.035\textwidth}p{.035\textwidth}|p{.035\textwidth}p{.035\textwidth}p{.035\textwidth}p{.035\textwidth}p{.035\textwidth}|}
\toprule
superv.&\multicolumn{5}{c|}{15\%} & \multicolumn{5}{c|}{20\%} & \multicolumn{5}{c|}{25\%} \\
method& \rot{MSSDA} & \rot{TransferCox} & \rot{Cox-AO} & \rot{RSF-AO} & \rot{DeepSurv-AO} &\rot{MSSDA} & \rot{TransferCox} & \rot{Cox-AO} & \rot{RSF-AO} & \rot{DeepSurv-AO} &\rot{MSSDA} & \rot{TransferCox} & \rot{Cox-AO} & \rot{RSF-AO} & \rot{DeepSurv-AO}  \\
\midrule
ACC  &           .761 &       .611 &           .738 &           .764 &     \textbf{.778} &           .776 &           .584 &           .755 &  \textbf{.779} &              .763 &  \textbf{.774} &           .731 &           .758 &           .757 &              .757 \\
&    \tiny{(.005)} &       \tiny{(.077)} &         \tiny{(.018)} &         \tiny{(.012)} &              \tiny{(.013)} &    \tiny{(.010)} &       \tiny{(.066)} &         \tiny{(.009)} &         \tiny{(.005)} &              \tiny{(.015)} &    \tiny{(.006)} &       \tiny{(.043)} &         \tiny{(.017)} &         \tiny{(.003)} &              \tiny{(.008)} \\
BLCA &           .551 &       .545 &           .551 &           .555 &     \textbf{.557} &           .550 &           .506 &  \textbf{.556} &           .530 &              .555 &           .546 &           .541 &           .534 &           .521 &     \textbf{.547} \\
&    \tiny{(.009)} &       \tiny{(.026)} &         \tiny{(.004)} &         \tiny{(.004)} &              \tiny{(.007)} &    \tiny{(.005)} &       \tiny{(.023)} &         \tiny{(.007)} &         \tiny{(.008)} &              \tiny{(.006)} &    \tiny{(.006)} &       \tiny{(.018)} &         \tiny{(.008)} &         \tiny{(.008)} &              \tiny{(.013)} \\
BRCA &  \textbf{.630} &       .553 &           .556 &           .581 &              .581 &  \textbf{.642} &           .529 &           .555 &           .585 &              .550 &  \textbf{.614} &           .552 &           .546 &           .563 &              .534 \\
&    \tiny{(.010)} &       \tiny{(.023)} &         \tiny{(.007)} &         \tiny{(.006)} &              \tiny{(.011)} &    \tiny{(.017)} &       \tiny{(.023)} &         \tiny{(.006)} &         \tiny{(.011)} &              \tiny{(.009)} &    \tiny{(.010)} &       \tiny{(.022)} &         \tiny{(.023)} &         \tiny{(.012)} &              \tiny{(.006)} \\
CESC &  \textbf{.690} &       .553 &           .609 &           .600 &              .601 &  \textbf{.674} &           .542 &           .620 &           .615 &              .584 &  \textbf{.672} &           .554 &           .629 &           .612 &              .596 \\
&    \tiny{(.007)} &       \tiny{(.028)} &         \tiny{(.005)} &         \tiny{(.018)} &              \tiny{(.007)} &    \tiny{(.007)} &       \tiny{(.029)} &         \tiny{(.010)} &         \tiny{(.006)} &              \tiny{(.007)} &    \tiny{(.008)} &       \tiny{(.065)} &         \tiny{(.010)} &         \tiny{(.014)} &              \tiny{(.016)} \\
CHOL &           .583 &       .472 &           .577 &           .501 &     \textbf{.594} &  \textbf{.612} &           .459 &           .582 &           .478 &              .532 &  \textbf{.578} &           .508 &           .571 &           .499 &              .534 \\
&    \tiny{(.012)} &       \tiny{(.070)} &         \tiny{(.011)} &         \tiny{(.018)} &              \tiny{(.034)} &    \tiny{(.033)} &       \tiny{(.111)} &         \tiny{(.012)} &         \tiny{(.028)} &              \tiny{(.013)} &    \tiny{(.026)} &       \tiny{(.045)} &         \tiny{(.016)} &         \tiny{(.017)} &              \tiny{(.012)} \\
ESCA &  \textbf{.580} &       .533 &           .570 &           .556 &              .543 &  \textbf{.582} &           .525 &           .556 &           .554 &              .556 &  \textbf{.584} &           .499 &           .556 &           .512 &              .565 \\
&    \tiny{(.008)} &       \tiny{(.040)} &         \tiny{(.010)} &         \tiny{(.013)} &              \tiny{(.008)} &    \tiny{(.007)} &       \tiny{(.037)} &         \tiny{(.007)} &         \tiny{(.015)} &              \tiny{(.007)} &    \tiny{(.015)} &       \tiny{(.038)} &         \tiny{(.021)} &         \tiny{(.012)} &              \tiny{(.007)} \\
HNSC &  \textbf{.597} &       .530 &           .566 &           .544 &              .566 &  \textbf{.592} &           .514 &           .565 &           .566 &              .564 &  \textbf{.590} &           .521 &           .567 &           .564 &              .568 \\
&    \tiny{(.006)} &       \tiny{(.025)} &         \tiny{(.006)} &         \tiny{(.008)} &              \tiny{(.008)} &    \tiny{(.007)} &       \tiny{(.019)} &         \tiny{(.008)} &         \tiny{(.005)} &              \tiny{(.004)} &    \tiny{(.006)} &       \tiny{(.017)} &         \tiny{(.006)} &         \tiny{(.011)} &              \tiny{(.007)} \\
KIRC &  \textbf{.603} &       .522 &           .564 &           .589 &              .577 &  \textbf{.612} &           .536 &           .551 &           .590 &              .557 &  \textbf{.619} &           .535 &           .570 &           .571 &              .579 \\
&    \tiny{(.015)} &       \tiny{(.017)} &         \tiny{(.022)} &         \tiny{(.011)} &              \tiny{(.014)} &    \tiny{(.007)} &       \tiny{(.046)} &         \tiny{(.017)} &         \tiny{(.014)} &              \tiny{(.011)} &    \tiny{(.010)} &       \tiny{(.049)} &         \tiny{(.012)} &         \tiny{(.008)} &              \tiny{(.015)} \\
KIRP &  \textbf{.768} &       .589 &           .663 &           .766 &              .699 &  \textbf{.778} &           .630 &           .706 &           .753 &              .735 &  \textbf{.793} &           .702 &           .724 &           .760 &              .693 \\
&    \tiny{(.013)} &       \tiny{(.108)} &         \tiny{(.020)} &         \tiny{(.006)} &              \tiny{(.013)} &    \tiny{(.013)} &       \tiny{(.056)} &         \tiny{(.015)} &         \tiny{(.010)} &              \tiny{(.014)} &    \tiny{(.009)} &       \tiny{(.138)} &         \tiny{(.005)} &         \tiny{(.009)} &              \tiny{(.015)} \\
LGG  &           .678 &       .704 &           .654 &  \textbf{.735} &              .562 &           .687 &           .679 &           .683 &  \textbf{.750} &              .553 &           .697 &           .729 &           .702 &  \textbf{.759} &              .582 \\
&    \tiny{(.012)} &       \tiny{(.023)} &         \tiny{(.031)} &         \tiny{(.010)} &              \tiny{(.026)} &    \tiny{(.008)} &       \tiny{(.040)} &         \tiny{(.017)} &         \tiny{(.009)} &              \tiny{(.005)} &    \tiny{(.009)} &       \tiny{(.030)} &         \tiny{(.011)} &         \tiny{(.007)} &              \tiny{(.006)} \\
LIHC &           .607 &       .585 &  \textbf{.621} &           .611 &              .597 &           .594 &           .558 &           .589 &  \textbf{.599} &              .592 &           .591 &           .573 &           .610 &  \textbf{.617} &              .591 \\
&    \tiny{(.016)} &       \tiny{(.029)} &         \tiny{(.008)} &         \tiny{(.012)} &              \tiny{(.006)} &    \tiny{(.010)} &       \tiny{(.032)} &         \tiny{(.011)} &         \tiny{(.011)} &              \tiny{(.008)} &    \tiny{(.014)} &       \tiny{(.038)} &         \tiny{(.021)} &         \tiny{(.005)} &              \tiny{(.010)} \\
LUAD &  \textbf{.605} &       .525 &           .544 &           .579 &              .583 &  \textbf{.593} &           .549 &           .559 &           .566 &              .587 &  \textbf{.595} &           .545 &           .565 &           .562 &              .590 \\
&    \tiny{(.005)} &       \tiny{(.034)} &         \tiny{(.014)} &         \tiny{(.010)} &              \tiny{(.008)} &    \tiny{(.007)} &       \tiny{(.026)} &         \tiny{(.010)} &         \tiny{(.014)} &              \tiny{(.003)} &    \tiny{(.004)} &       \tiny{(.021)} &         \tiny{(.008)} &         \tiny{(.010)} &              \tiny{(.009)} \\
LUSC &  \textbf{.569} &       .500 &           .500 &           .493 &              .526 &  \textbf{.566} &           .515 &           .499 &           .504 &              .537 &  \textbf{.559} &           .516 &           .473 &           .515 &              .541 \\
&    \tiny{(.008)} &       \tiny{(.029)} &         \tiny{(.009)} &         \tiny{(.004)} &              \tiny{(.009)} &    \tiny{(.005)} &       \tiny{(.021)} &         \tiny{(.008)} &         \tiny{(.003)} &              \tiny{(.011)} &    \tiny{(.004)} &       \tiny{(.039)} &         \tiny{(.003)} &         \tiny{(.005)} &              \tiny{(.005)} \\
MESO &           .624 &       .540 &  \textbf{.695} &           .606 &              .674 &           .629 &           .603 &  \textbf{.667} &           .655 &              .632 &           .607 &           .571 &  \textbf{.668} &           .633 &              .640 \\
&    \tiny{(.016)} &       \tiny{(.059)} &         \tiny{(.009)} &         \tiny{(.011)} &              \tiny{(.009)} &    \tiny{(.021)} &       \tiny{(.049)} &         \tiny{(.014)} &         \tiny{(.014)} &              \tiny{(.008)} &    \tiny{(.021)} &       \tiny{(.077)} &         \tiny{(.016)} &         \tiny{(.003)} &              \tiny{(.011)} \\
PAAD &           .585 &       .490 &  \textbf{.602} &           .531 &              .514 &  \textbf{.583} &           .537 &           .576 &           .548 &              .538 &  \textbf{.580} &           .502 &           .573 &           .527 &              .566 \\
&    \tiny{(.012)} &       \tiny{(.025)} &         \tiny{(.007)} &         \tiny{(.007)} &              \tiny{(.007)} &    \tiny{(.013)} &       \tiny{(.055)} &         \tiny{(.012)} &         \tiny{(.009)} &              \tiny{(.008)} &    \tiny{(.011)} &       \tiny{(.023)} &         \tiny{(.008)} &         \tiny{(.010)} &              \tiny{(.006)} \\
SARC &  \textbf{.610} &       .525 &           .596 &           .553 &              .590 &  \textbf{.620} &           .586 &           .560 &           .603 &              .595 &  \textbf{.613} &           .555 &           .573 &           .574 &              .600 \\
&    \tiny{(.006)} &       \tiny{(.030)} &         \tiny{(.020)} &         \tiny{(.011)} &              \tiny{(.012)} &    \tiny{(.011)} &       \tiny{(.025)} &         \tiny{(.006)} &         \tiny{(.012)} &              \tiny{(.008)} &    \tiny{(.011)} &       \tiny{(.009)} &         \tiny{(.012)} &         \tiny{(.020)} &              \tiny{(.010)} \\
SKCM &  \textbf{.670} &       .525 &           .508 &           .536 &              .561 &  \textbf{.694} &           .534 &           .519 &           .528 &              .584 &  \textbf{.672} &           .544 &           .504 &           .532 &              .589 \\
&    \tiny{(.014)} &       \tiny{(.088)} &         \tiny{(.008)} &         \tiny{(.031)} &              \tiny{(.023)} &    \tiny{(.019)} &       \tiny{(.078)} &         \tiny{(.018)} &         \tiny{(.021)} &              \tiny{(.013)} &    \tiny{(.028)} &       \tiny{(.057)} &         \tiny{(.031)} &         \tiny{(.047)} &              \tiny{(.020)} \\
STAD &  \textbf{.543} &       .493 &           .494 &           .530 &              .504 &  \textbf{.535} &           .500 &           .484 &           .517 &              .533 &  \textbf{.542} &           .521 &           .519 &           .531 &              .511 \\
&    \tiny{(.006)} &       \tiny{(.035)} &         \tiny{(.004)} &         \tiny{(.010)} &              \tiny{(.004)} &    \tiny{(.009)} &       \tiny{(.026)} &         \tiny{(.016)} &         \tiny{(.013)} &              \tiny{(.006)} &    \tiny{(.011)} &       \tiny{(.022)} &         \tiny{(.015)} &         \tiny{(.012)} &              \tiny{(.002)} \\
UCEC &  \textbf{.623} &       .520 &           .523 &           .611 &              .593 &  \textbf{.643} &           .495 &           .539 &           .627 &              .604 &  \textbf{.622} &           .555 &           .570 &           .608 &              .571 \\
&    \tiny{(.010)} &       \tiny{(.022)} &         \tiny{(.025)} &         \tiny{(.017)} &              \tiny{(.012)} &    \tiny{(.007)} &       \tiny{(.041)} &         \tiny{(.019)} &         \tiny{(.011)} &              \tiny{(.019)} &    \tiny{(.003)} &       \tiny{(.015)} &         \tiny{(.025)} &         \tiny{(.017)} &              \tiny{(.009)} \\
UCS  &  \textbf{.520} &       .422 &           .518 &           .475 &              .433 &           .503 &  \textbf{.536} &           .497 &           .463 &              .413 &  \textbf{.523} &           .476 &           .492 &           .444 &              .369 \\
&    \tiny{(.012)} &       \tiny{(.211)} &         \tiny{(.016)} &         \tiny{(.031)} &              \tiny{(.005)} &    \tiny{(.013)} &       \tiny{(.095)} &         \tiny{(.022)} &         \tiny{(.012)} &              \tiny{(.013)} &    \tiny{(.015)} &       \tiny{(.055)} &         \tiny{(.010)} &         \tiny{(.016)} &              \tiny{(.011)} \\
UVM  &  \textbf{.717} &       .672 &           .592 &           .617 &              .519 &           .683 &           .700 &           .604 &  \textbf{.743} &              .442 &           .664 &  \textbf{.744} &           .727 &  \textbf{.744} &              .449 \\
&    \tiny{(.015)} &       \tiny{(.093)} &         \tiny{(.019)} &         \tiny{(.028)} &              \tiny{(.018)} &    \tiny{(.006)} &       \tiny{(.091)} &         \tiny{(.026)} &         \tiny{(.016)} &              \tiny{(.015)} &    \tiny{(.013)} &       \tiny{(.032)} &         \tiny{(.034)} &         \tiny{(.022)} &              \tiny{(.028)} \\
\bottomrule
P-value&
&\scriptsize{3e-2}&\scriptsize{5e-2}&\scriptsize{1e-2}&\scriptsize{3e-4}&
&\scriptsize{2e-2}&\scriptsize{1e-1}&\scriptsize{9e-3}&\scriptsize{9e-4}&
&\scriptsize{9e-2}&\scriptsize{7e-2}&\scriptsize{2e-2}&\scriptsize{3e-3}\\
\bottomrule
Rank& \textbf{1.61}& 4.60& 3.07& 2.88& 2.86& \textbf{1.52}& 4.33& 3.36& 2.57& 3.21& \textbf{1.60}& 4.02& 3.10& 3.10& 3.19\\ 
\bottomrule
\end{tabular}
\caption{The performance comparison on the miRNA data in terms of $\text{C-index}$. The following settings were used: 15\%, 20\%, and 25\% supervision. The numbers in brackets depict the standard error. The last row shows the rank in each supervision group.
The p-value row depicts the p-value for the upper-tailed Wilcoxon signed-ranks test between each method and MDSSA.}\label{tab:Supp-miRNA-C-index-rem-2-RSF}
\end{table*}
\begin{table*}
\small
\centering
\begin{tabular}{|lp{.035\textwidth}p{.035\textwidth}p{.035\textwidth}p{.035\textwidth}|p{.035\textwidth}p{.035\textwidth}p{.035\textwidth}p{.035\textwidth}|p{.035\textwidth}p{.035\textwidth}p{.035\textwidth}p{.035\textwidth}|}
\toprule
superv.&\multicolumn{4}{c|}{.00\%} & \multicolumn{4}{c|}{5\%} & \multicolumn{4}{c|}{10\%}  \\
method& \rot{MSSDA} & \rot{Cox-AO} & \rot{RSF-AO} & \rot{DeepSurv-AO}  & \rot{MSSDA} & \rot{Cox-AO} & \rot{RSF-AO} & \rot{DeepSurv-AO} & \rot{MSSDA} & \rot{Cox-AO} & \rot{RSF-AO} & \rot{DeepSurv-AO}  \\
\midrule
ACC  &  \textbf{.784} &           .694 &         .715 &              .735 &  \textbf{.764} &         .748 &           .733 &              .745 &  \textbf{.817} &           .753 &         .794 &              .734 \\
&    \tiny{(.008)} &         \tiny{(.008)} &         \tiny{(.012)} &              \tiny{(.012)} &    \tiny{(.006)} &         \tiny{(.015)} &         \tiny{(.002)} &              \tiny{(.018)} &    \tiny{(.010)} &         \tiny{(.016)} &         \tiny{(.009)} &              \tiny{(.014)} \\
BLCA &           .538 &  \textbf{.559} &         .535 &              .547 &           .565 &         .564 &  \textbf{.573} &              .541 &  \textbf{.582} &           .570 &         .578 &              .548 \\
&    \tiny{(.007)} &         \tiny{(.001)} &         \tiny{(.003)} &              \tiny{(.003)} &    \tiny{(.006)} &         \tiny{(.003)} &         \tiny{(.010)} &              \tiny{(.005)} &    \tiny{(.003)} &         \tiny{(.004)} &         \tiny{(.017)} &              \tiny{(.009)} \\
BRCA &  \textbf{.614} &           .594 &         .545 &              .556 &  \textbf{.650} &         .586 &           .568 &              .574 &  \textbf{.661} &           .587 &         .647 &              .567 \\
&    \tiny{(.018)} &         \tiny{(.001)} &         \tiny{(.013)} &              \tiny{(.008)} &    \tiny{(.016)} &         \tiny{(.008)} &         \tiny{(.019)} &              \tiny{(.008)} &    \tiny{(.014)} &         \tiny{(.008)} &         \tiny{(.011)} &              \tiny{(.001)} \\
CESC &  \textbf{.671} &           .608 &         .556 &              .600 &  \textbf{.700} &         .607 &           .628 &              .597 &  \textbf{.712} &           .632 &         .637 &              .592 \\
&    \tiny{(.005)} &         \tiny{(.001)} &         \tiny{(.014)} &              \tiny{(.010)} &    \tiny{(.007)} &         \tiny{(.006)} &         \tiny{(.017)} &              \tiny{(.016)} &    \tiny{(.006)} &         \tiny{(.007)} &         \tiny{(.013)} &              \tiny{(.003)} \\
CHOL &  \textbf{.639} &           .549 &         .461 &              .555 &  \textbf{.628} &         .560 &           .476 &              .526 &  \textbf{.652} &           .564 &         .553 &              .551 \\
&    \tiny{(.011)} &         \tiny{(.004)} &         \tiny{(.029)} &              \tiny{(.027)} &    \tiny{(.015)} &         \tiny{(.006)} &         \tiny{(.025)} &              \tiny{(.007)} &    \tiny{(.012)} &         \tiny{(.010)} &         \tiny{(.018)} &              \tiny{(.013)} \\
ESCA &  \textbf{.600} &           .552 &         .567 &              .551 &  \textbf{.607} &         .563 &           .586 &              .551 &  \textbf{.626} &           .575 &         .610 &              .558 \\
&    \tiny{(.008)} &         \tiny{(.003)} &         \tiny{(.019)} &              \tiny{(.012)} &    \tiny{(.009)} &         \tiny{(.004)} &         \tiny{(.014)} &              \tiny{(.010)} &    \tiny{(.009)} &         \tiny{(.016)} &         \tiny{(.008)} &              \tiny{(.004)} \\
HNSC &  \textbf{.599} &           .553 &         .518 &              .578 &  \textbf{.620} &         .563 &           .569 &              .575 &  \textbf{.637} &           .573 &         .596 &              .577 \\
&    \tiny{(.004)} &         \tiny{(.002)} &         \tiny{(.010)} &              \tiny{(.001)} &    \tiny{(.004)} &         \tiny{(.011)} &         \tiny{(.012)} &              \tiny{(.006)} &    \tiny{(.003)} &         \tiny{(.008)} &         \tiny{(.009)} &              \tiny{(.005)} \\
KIRC &  \textbf{.606} &           .551 &         .595 &              .572 &  \textbf{.636} &         .550 &           .599 &              .563 &  \textbf{.673} &           .589 &         .626 &              .570 \\
&    \tiny{(.008)} &         \tiny{(.001)} &         \tiny{(.018)} &              \tiny{(.016)} &    \tiny{(.010)} &         \tiny{(.007)} &         \tiny{(.006)} &              \tiny{(.013)} &    \tiny{(.008)} &         \tiny{(.018)} &         \tiny{(.006)} &              \tiny{(.010)} \\
KIRP &  \textbf{.782} &           .589 &         .715 &              .694 &  \textbf{.792} &         .680 &           .771 &              .725 &  \textbf{.802} &           .680 &         .784 &              .703 \\
&    \tiny{(.004)} &         \tiny{(.008)} &         \tiny{(.012)} &              \tiny{(.014)} &    \tiny{(.011)} &         \tiny{(.028)} &         \tiny{(.007)} &              \tiny{(.019)} &    \tiny{(.007)} &         \tiny{(.026)} &         \tiny{(.016)} &              \tiny{(.025)} \\
LGG  &  \textbf{.635} &           .515 &         .605 &              .560 &  \textbf{.675} &         .559 &           .668 &              .559 &  \textbf{.708} &           .606 &         .702 &              .533 \\
&    \tiny{(.011)} &         \tiny{(.004)} &         \tiny{(.008)} &              \tiny{(.025)} &    \tiny{(.005)} &         \tiny{(.023)} &         \tiny{(.005)} &              \tiny{(.005)} &    \tiny{(.011)} &         \tiny{(.024)} &         \tiny{(.006)} &              \tiny{(.015)} \\
LIHC &  \textbf{.595} &           .579 &         .587 &              .589 &  \textbf{.628} &         .582 &           .615 &              .599 &  \textbf{.651} &           .620 &         .635 &              .605 \\
&    \tiny{(.003)} &         \tiny{(.006)} &         \tiny{(.002)} &              \tiny{(.007)} &    \tiny{(.006)} &         \tiny{(.014)} &         \tiny{(.003)} &              \tiny{(.006)} &    \tiny{(.003)} &         \tiny{(.010)} &         \tiny{(.009)} &              \tiny{(.011)} \\
LUAD &  \textbf{.604} &           .568 &         .570 &              .586 &  \textbf{.617} &         .573 &           .585 &              .597 &  \textbf{.635} &           .577 &         .615 &              .581 \\
&    \tiny{(.004)} &         \tiny{(.002)} &         \tiny{(.008)} &              \tiny{(.007)} &    \tiny{(.005)} &         \tiny{(.003)} &         \tiny{(.008)} &              \tiny{(.004)} &    \tiny{(.006)} &         \tiny{(.004)} &         \tiny{(.007)} &              \tiny{(.007)} \\
LUSC &  \textbf{.569} &           .520 &         .496 &              .526 &  \textbf{.588} &         .524 &           .536 &              .516 &  \textbf{.616} &           .535 &         .563 &              .533 \\
&    \tiny{(.003)} &         \tiny{(.001)} &         \tiny{(.005)} &              \tiny{(.005)} &    \tiny{(.005)} &         \tiny{(.007)} &         \tiny{(.006)} &              \tiny{(.006)} &    \tiny{(.004)} &         \tiny{(.005)} &         \tiny{(.005)} &              \tiny{(.004)} \\
MESO &           .621 &  \textbf{.678} &         .592 &              .673 &           .615 &         .669 &           .621 &     \textbf{.672} &           .625 &  \textbf{.671} &         .649 &     \textbf{.671} \\
&    \tiny{(.002)} &         \tiny{(.004)} &         \tiny{(.020)} &              \tiny{(.008)} &    \tiny{(.007)} &         \tiny{(.008)} &         \tiny{(.007)} &              \tiny{(.006)} &    \tiny{(.009)} &         \tiny{(.004)} &         \tiny{(.008)} &              \tiny{(.007)} \\
PAAD &  \textbf{.582} &           .552 &         .504 &              .548 &  \textbf{.598} &         .563 &           .546 &              .538 &  \textbf{.625} &           .574 &         .571 &              .529 \\
&    \tiny{(.004)} &         \tiny{(.002)} &         \tiny{(.010)} &              \tiny{(.003)} &    \tiny{(.007)} &         \tiny{(.004)} &         \tiny{(.013)} &              \tiny{(.009)} &    \tiny{(.007)} &         \tiny{(.008)} &         \tiny{(.016)} &              \tiny{(.008)} \\
SARC &           .601 &           .564 &         .591 &     \textbf{.613} &  \textbf{.623} &         .570 &           .617 &              .598 &  \textbf{.650} &           .596 &         .643 &              .595 \\
&    \tiny{(.006)} &         \tiny{(.004)} &         \tiny{(.005)} &              \tiny{(.008)} &    \tiny{(.005)} &         \tiny{(.009)} &         \tiny{(.005)} &              \tiny{(.010)} &    \tiny{(.010)} &         \tiny{(.007)} &         \tiny{(.008)} &              \tiny{(.007)} \\
SKCM &  \textbf{.663} &           .540 &         .498 &              .646 &  \textbf{.701} &         .529 &           .530 &              .640 &  \textbf{.702} &           .547 &         .555 &              .676 \\
&    \tiny{(.011)} &         \tiny{(.004)} &         \tiny{(.016)} &              \tiny{(.024)} &    \tiny{(.020)} &         \tiny{(.015)} &         \tiny{(.015)} &              \tiny{(.013)} &    \tiny{(.016)} &         \tiny{(.015)} &         \tiny{(.014)} &              \tiny{(.024)} \\
STAD &  \textbf{.539} &           .468 &         .487 &              .503 &  \textbf{.567} &         .479 &           .533 &              .521 &  \textbf{.586} &           .482 &         .552 &              .508 \\
&    \tiny{(.006)} &         \tiny{(.001)} &         \tiny{(.004)} &              \tiny{(.009)} &    \tiny{(.005)} &         \tiny{(.004)} &         \tiny{(.011)} &              \tiny{(.004)} &    \tiny{(.004)} &         \tiny{(.005)} &         \tiny{(.012)} &              \tiny{(.008)} \\
UCEC &  \textbf{.657} &           .500 &         .558 &              .573 &  \textbf{.665} &         .526 &           .614 &              .581 &  \textbf{.685} &           .524 &         .662 &              .573 \\
&    \tiny{(.007)} &         \tiny{(.003)} &         \tiny{(.011)} &              \tiny{(.007)} &    \tiny{(.006)} &         \tiny{(.008)} &         \tiny{(.011)} &              \tiny{(.011)} &    \tiny{(.008)} &         \tiny{(.019)} &         \tiny{(.013)} &              \tiny{(.010)} \\
UCS  &  \textbf{.524} &           .511 &         .471 &              .439 &  \textbf{.542} &         .519 &           .520 &              .424 &  \textbf{.555} &           .510 &         .553 &              .426 \\
&    \tiny{(.007)} &         \tiny{(.001)} &         \tiny{(.012)} &              \tiny{(.011)} &    \tiny{(.006)} &         \tiny{(.007)} &         \tiny{(.012)} &              \tiny{(.014)} &    \tiny{(.007)} &         \tiny{(.006)} &         \tiny{(.011)} &              \tiny{(.004)} \\
UVM  &  \textbf{.696} &           .457 &         .498 &              .546 &  \textbf{.675} &         .564 &           .592 &              .514 &  \textbf{.705} &           .537 &         .677 &              .443 \\
&    \tiny{(.010)} &         \tiny{(.007)} &         \tiny{(.013)} &              \tiny{(.019)} &    \tiny{(.013)} &         \tiny{(.025)} &         \tiny{(.029)} &              \tiny{(.032)} &    \tiny{(.014)} &         \tiny{(.018)} &         \tiny{(.011)} &              \tiny{(.036)} \\
\bottomrule
P-value&
&\scriptsize{4e-4}&\scriptsize{4e-4}&\scriptsize{1e-2}&
&\scriptsize{1e-4}&\scriptsize{5e-3}&\scriptsize{9e-4}&
&\scriptsize{1e-4}&\scriptsize{3e-2}&\scriptsize{2e-4}\\
\bottomrule
Rank& \textbf{1.24}& 3.10& 3.24& 2.43& \textbf{1.19}& 3.21& 2.48& 3.12 & \textbf{1.14}& 3.12& 2.19& 3.55\\
\bottomrule
\end{tabular}
\caption{The performance comparison on the miRNA data in terms of $\text{C-index}^\prime$. The following settings were used: no supervision, 5\%, and 10\% supervision. The numbers in brackets depict the standard error. The last row shows the rank in each supervision group.
The p-value row depicts the p-value for the upper-tailed Wilcoxon signed-ranks test between each method and MDSSA.}\label{tab:Supp-miRNA-C-index-1-RSF}
\end{table*}
\begin{table*}
\small
\centering
\begin{tabular}{|lp{.035\textwidth}p{.035\textwidth}p{.035\textwidth}p{.035\textwidth}p{.035\textwidth}|p{.035\textwidth}p{.035\textwidth}p{.035\textwidth}p{.035\textwidth}p{.035\textwidth}|p{.035\textwidth}p{.035\textwidth}p{.035\textwidth}p{.035\textwidth}p{.035\textwidth}|}
\toprule
superv.&\multicolumn{5}{c|}{15\%} & \multicolumn{5}{c|}{20\%} & \multicolumn{5}{c|}{25\%} \\
method& \rot{MSSDA} & \rot{TransferCox} & \rot{Cox-AO} & \rot{RSF-AO} & \rot{DeepSurv-AO} &\rot{MSSDA} & \rot{TransferCox} & \rot{Cox-AO} & \rot{RSF-AO} & \rot{DeepSurv-AO} &\rot{MSSDA} & \rot{TransferCox} & \rot{Cox-AO} & \rot{RSF-AO} & \rot{DeepSurv-AO}  \\
\midrule
ACC  &  \textbf{.807} &           .677 &           .773 &           .790 &              .773 &  \textbf{.827} &           .680 &           .774 &           .805 &              .756 &  \textbf{.836} &           .789 &           .792 &           .808 &              .753 \\
&    \tiny{(.003)} &       \tiny{(.064)} &         \tiny{(.010)} &         \tiny{(.011)} &              \tiny{(.013)} &    \tiny{(.007)} &       \tiny{(.056)} &         \tiny{(.005)} &         \tiny{(.004)} &              \tiny{(.010)} &    \tiny{(.003)} &       \tiny{(.030)} &         \tiny{(.009)} &         \tiny{(.008)} &              \tiny{(.006)} \\
BLCA &           .607 &  \textbf{.623} &           .563 &           .614 &              .567 &  \textbf{.626} &           .615 &           .577 &           .587 &              .559 &           .641 &  \textbf{.664} &           .567 &           .628 &              .550 \\
&    \tiny{(.006)} &       \tiny{(.016)} &         \tiny{(.004)} &         \tiny{(.005)} &              \tiny{(.005)} &    \tiny{(.002)} &       \tiny{(.023)} &         \tiny{(.006)} &         \tiny{(.019)} &              \tiny{(.003)} &    \tiny{(.005)} &       \tiny{(.015)} &         \tiny{(.005)} &         \tiny{(.007)} &              \tiny{(.009)} \\
BRCA &  \textbf{.700} &           .659 &           .579 &           .669 &              .562 &  \textbf{.716} &           .694 &           .588 &           .686 &              .570 &  \textbf{.731} &           .716 &           .586 &           .703 &              .546 \\
&    \tiny{(.009)} &       \tiny{(.016)} &         \tiny{(.006)} &         \tiny{(.005)} &              \tiny{(.010)} &    \tiny{(.017)} &       \tiny{(.022)} &         \tiny{(.007)} &         \tiny{(.010)} &              \tiny{(.010)} &    \tiny{(.010)} &       \tiny{(.025)} &         \tiny{(.005)} &         \tiny{(.013)} &              \tiny{(.012)} \\
CESC &  \textbf{.748} &           .632 &           .625 &           .674 &              .602 &  \textbf{.752} &           .661 &           .632 &           .699 &              .579 &  \textbf{.765} &           .679 &           .660 &           .714 &              .603 \\
&    \tiny{(.009)} &       \tiny{(.022)} &         \tiny{(.002)} &         \tiny{(.019)} &              \tiny{(.008)} &    \tiny{(.009)} &       \tiny{(.030)} &         \tiny{(.008)} &         \tiny{(.006)} &              \tiny{(.012)} &    \tiny{(.007)} &       \tiny{(.038)} &         \tiny{(.009)} &         \tiny{(.010)} &              \tiny{(.011)} \\
CHOL &  \textbf{.635} &           .568 &           .569 &           .589 &              .576 &  \textbf{.680} &           .610 &           .582 &           .568 &              .562 &  \textbf{.664} &           .632 &           .594 &           .613 &              .548 \\
&    \tiny{(.015)} &       \tiny{(.086)} &         \tiny{(.009)} &         \tiny{(.015)} &              \tiny{(.032)} &    \tiny{(.036)} &       \tiny{(.051)} &         \tiny{(.011)} &         \tiny{(.018)} &              \tiny{(.013)} &    \tiny{(.029)} &       \tiny{(.059)} &         \tiny{(.011)} &         \tiny{(.008)} &              \tiny{(.011)} \\
ESCA &  \textbf{.630} &           .619 &           .576 &           .612 &              .540 &           .640 &  \textbf{.645} &           .589 &           .629 &              .550 &  \textbf{.667} &           .628 &           .574 &           .623 &              .549 \\
&    \tiny{(.009)} &       \tiny{(.037)} &         \tiny{(.006)} &         \tiny{(.015)} &              \tiny{(.004)} &    \tiny{(.008)} &       \tiny{(.028)} &         \tiny{(.008)} &         \tiny{(.012)} &              \tiny{(.004)} &    \tiny{(.016)} &       \tiny{(.047)} &         \tiny{(.017)} &         \tiny{(.015)} &              \tiny{(.005)} \\
HNSC &  \textbf{.658} &           .610 &           .579 &           .605 &              .575 &  \textbf{.675} &           .617 &           .577 &           .648 &              .577 &  \textbf{.691} &           .645 &           .583 &           .658 &              .591 \\
&    \tiny{(.006)} &       \tiny{(.019)} &         \tiny{(.006)} &         \tiny{(.006)} &              \tiny{(.005)} &    \tiny{(.006)} &       \tiny{(.009)} &         \tiny{(.005)} &         \tiny{(.004)} &              \tiny{(.002)} &    \tiny{(.005)} &       \tiny{(.006)} &         \tiny{(.005)} &         \tiny{(.011)} &              \tiny{(.006)} \\
KIRC &  \textbf{.676} &           .596 &           .579 &           .648 &              .565 &  \textbf{.696} &           .656 &           .563 &           .676 &              .549 &  \textbf{.721} &           .670 &           .592 &           .676 &              .589 \\
&    \tiny{(.011)} &       \tiny{(.009)} &         \tiny{(.017)} &         \tiny{(.013)} &              \tiny{(.012)} &    \tiny{(.008)} &       \tiny{(.044)} &         \tiny{(.015)} &         \tiny{(.012)} &              \tiny{(.010)} &    \tiny{(.007)} &       \tiny{(.032)} &         \tiny{(.013)} &         \tiny{(.007)} &              \tiny{(.014)} \\
KIRP &  \textbf{.819} &           .683 &           .697 &           .809 &              .703 &  \textbf{.835} &           .720 &           .720 &           .810 &              .745 &  \textbf{.854} &           .797 &           .769 &           .831 &              .701 \\
&    \tiny{(.012)} &       \tiny{(.091)} &         \tiny{(.020)} &         \tiny{(.010)} &              \tiny{(.015)} &    \tiny{(.010)} &       \tiny{(.038)} &         \tiny{(.015)} &         \tiny{(.008)} &              \tiny{(.017)} &    \tiny{(.007)} &       \tiny{(.103)} &         \tiny{(.009)} &         \tiny{(.010)} &              \tiny{(.009)} \\
LGG  &           .734 &           .754 &           .662 &  \textbf{.772} &              .557 &           .756 &           .739 &           .690 &  \textbf{.787} &              .542 &           .778 &           .804 &           .717 &  \textbf{.810} &              .578 \\
&    \tiny{(.011)} &       \tiny{(.022)} &         \tiny{(.028)} &         \tiny{(.010)} &              \tiny{(.020)} &    \tiny{(.008)} &       \tiny{(.033)} &         \tiny{(.011)} &         \tiny{(.005)} &              \tiny{(.009)} &    \tiny{(.007)} &       \tiny{(.029)} &         \tiny{(.011)} &         \tiny{(.003)} &              \tiny{(.011)} \\
LIHC &  \textbf{.672} &           .644 &           .635 &           .658 &              .598 &  \textbf{.678} &           .654 &           .625 &           .664 &              .601 &  \textbf{.695} &           .685 &           .625 &           .688 &              .601 \\
&    \tiny{(.007)} &       \tiny{(.024)} &         \tiny{(.008)} &         \tiny{(.011)} &              \tiny{(.006)} &    \tiny{(.004)} &       \tiny{(.012)} &         \tiny{(.009)} &         \tiny{(.008)} &              \tiny{(.004)} &    \tiny{(.004)} &       \tiny{(.037)} &         \tiny{(.012)} &         \tiny{(.003)} &              \tiny{(.003)} \\
LUAD &  \textbf{.665} &           .604 &           .563 &           .636 &              .586 &  \textbf{.674} &           .650 &           .578 &           .639 &              .586 &  \textbf{.695} &           .655 &           .582 &           .662 &              .603 \\
&    \tiny{(.004)} &       \tiny{(.035)} &         \tiny{(.012)} &         \tiny{(.012)} &              \tiny{(.009)} &    \tiny{(.006)} &       \tiny{(.027)} &         \tiny{(.006)} &         \tiny{(.016)} &              \tiny{(.005)} &    \tiny{(.006)} &       \tiny{(.012)} &         \tiny{(.004)} &         \tiny{(.009)} &              \tiny{(.006)} \\
LUSC &  \textbf{.633} &           .579 &           .527 &           .566 &              .512 &  \textbf{.656} &           .619 &           .521 &           .594 &              .521 &  \textbf{.672} &           .651 &           .511 &           .618 &              .523 \\
&    \tiny{(.009)} &       \tiny{(.032)} &         \tiny{(.009)} &         \tiny{(.007)} &              \tiny{(.006)} &    \tiny{(.009)} &       \tiny{(.009)} &         \tiny{(.005)} &         \tiny{(.004)} &              \tiny{(.004)} &    \tiny{(.012)} &       \tiny{(.029)} &         \tiny{(.003)} &         \tiny{(.011)} &              \tiny{(.007)} \\
MESO &           .634 &           .600 &  \textbf{.700} &           .636 &              .666 &           .640 &           .674 &  \textbf{.678} &           .671 &              .644 &           .634 &           .678 &  \textbf{.695} &           .673 &              .658 \\
&    \tiny{(.013)} &       \tiny{(.047)} &         \tiny{(.004)} &         \tiny{(.011)} &              \tiny{(.014)} &    \tiny{(.012)} &       \tiny{(.035)} &         \tiny{(.009)} &         \tiny{(.011)} &              \tiny{(.003)} &    \tiny{(.008)} &       \tiny{(.073)} &         \tiny{(.013)} &         \tiny{(.008)} &              \tiny{(.014)} \\
PAAD &  \textbf{.629} &           .562 &           .603 &           .580 &              .533 &  \textbf{.641} &           .632 &           .579 &           .612 &              .524 &  \textbf{.654} &           .640 &           .594 &           .619 &              .546 \\
&    \tiny{(.007)} &       \tiny{(.039)} &         \tiny{(.006)} &         \tiny{(.005)} &              \tiny{(.007)} &    \tiny{(.009)} &       \tiny{(.041)} &         \tiny{(.006)} &         \tiny{(.004)} &              \tiny{(.006)} &    \tiny{(.010)} &       \tiny{(.015)} &         \tiny{(.009)} &         \tiny{(.009)} &              \tiny{(.004)} \\
SARC &  \textbf{.663} &           .605 &           .607 &           .621 &              .606 &  \textbf{.690} &           .674 &           .593 &           .674 &              .590 &  \textbf{.704} &           .671 &           .616 &           .677 &              .612 \\
&    \tiny{(.004)} &       \tiny{(.027)} &         \tiny{(.016)} &         \tiny{(.007)} &              \tiny{(.008)} &    \tiny{(.007)} &       \tiny{(.031)} &         \tiny{(.008)} &         \tiny{(.013)} &              \tiny{(.008)} &    \tiny{(.007)} &       \tiny{(.017)} &         \tiny{(.016)} &         \tiny{(.016)} &              \tiny{(.009)} \\
SKCM &  \textbf{.727} &           .627 &           .548 &           .603 &              .593 &  \textbf{.752} &           .631 &           .571 &           .601 &              .617 &  \textbf{.747} &           .697 &           .545 &           .651 &              .627 \\
&    \tiny{(.015)} &       \tiny{(.064)} &         \tiny{(.005)} &         \tiny{(.037)} &              \tiny{(.023)} &    \tiny{(.011)} &       \tiny{(.063)} &         \tiny{(.017)} &         \tiny{(.027)} &              \tiny{(.010)} &    \tiny{(.017)} &       \tiny{(.050)} &         \tiny{(.023)} &         \tiny{(.039)} &              \tiny{(.012)} \\
STAD &  \textbf{.604} &           .576 &           .513 &           .592 &              .504 &  \textbf{.618} &           .616 &           .513 &           .598 &              .511 &           .640 &  \textbf{.654} &           .528 &           .616 &              .508 \\
&    \tiny{(.003)} &       \tiny{(.026)} &         \tiny{(.004)} &         \tiny{(.009)} &              \tiny{(.003)} &    \tiny{(.004)} &       \tiny{(.017)} &         \tiny{(.015)} &         \tiny{(.008)} &              \tiny{(.005)} &    \tiny{(.004)} &       \tiny{(.021)} &         \tiny{(.015)} &         \tiny{(.013)} &              \tiny{(.005)} \\
UCEC &  \textbf{.701} &           .624 &           .543 &           .690 &              .576 &  \textbf{.735} &           .636 &           .582 &           .723 &              .585 &  \textbf{.744} &           .693 &           .610 &           .709 &              .552 \\
&    \tiny{(.007)} &       \tiny{(.022)} &         \tiny{(.020)} &         \tiny{(.013)} &              \tiny{(.011)} &    \tiny{(.004)} &       \tiny{(.029)} &         \tiny{(.018)} &         \tiny{(.013)} &              \tiny{(.011)} &    \tiny{(.005)} &       \tiny{(.018)} &         \tiny{(.022)} &         \tiny{(.010)} &              \tiny{(.012)} \\
UCS  &  \textbf{.578} &           .496 &           .523 &           .568 &              .422 &           .587 &  \textbf{.634} &           .519 &           .591 &              .412 &           .614 &  \textbf{.627} &           .519 &           .576 &              .406 \\
&    \tiny{(.008)} &       \tiny{(.203)} &         \tiny{(.016)} &         \tiny{(.023)} &              \tiny{(.008)} &    \tiny{(.006)} &       \tiny{(.082)} &         \tiny{(.012)} &         \tiny{(.014)} &              \tiny{(.013)} &    \tiny{(.003)} &       \tiny{(.053)} &         \tiny{(.009)} &         \tiny{(.009)} &              \tiny{(.008)} \\
UVM  &  \textbf{.757} &           .733 &           .631 &           .660 &              .526 &           .744 &           .774 &           .652 &  \textbf{.782} &              .450 &           .752 &  \textbf{.821} &           .719 &           .784 &              .500 \\
&    \tiny{(.011)} &       \tiny{(.085)} &         \tiny{(.024)} &         \tiny{(.024)} &              \tiny{(.017)} &    \tiny{(.010)} &       \tiny{(.054)} &         \tiny{(.027)} &         \tiny{(.009)} &              \tiny{(.013)} &    \tiny{(.015)} &       \tiny{(.023)} &         \tiny{(.024)} &         \tiny{(.017)} &              \tiny{(.020)} \\
\bottomrule
P-value&
&\scriptsize{1e-4}&\scriptsize{4e-2}&\scriptsize{1e-5}&\scriptsize{2e-3}&
&\scriptsize{6e-5}&\scriptsize{7e-2}&\scriptsize{4e-6}&\scriptsize{2e-2}&
&\scriptsize{2e-4}&\scriptsize{7e-2}&\scriptsize{1e-6}&\scriptsize{1e-1}\\
\bottomrule
Rank& \textbf{1.33}& 3.19& 3.83& 2.29& 4.36& \textbf{1.48}& 2.48& 3.98& 2.50& 4.57& \textbf{1.52}& 2.29& 4.00& 2.43& 4.76\\ 
\bottomrule
\end{tabular}
\caption{The performance comparison on the miRNA data in terms of $\text{C-index}^\prime$. The following settings were used: 15\%, 20\%, and 25\% supervision. The numbers in brackets depict the standard error. The last row shows the rank in each supervision group.
The p-value row depicts the p-value for the upper-tailed Wilcoxon signed-ranks test between each method and MDSSA.}\label{tab:Supp-miRNA-C-index-2-RSF}
\end{table*}

\end{document}